%% file: Arxiv_submission.tex
\documentclass[11pt,twoside]{article}

\usepackage{fullpage}

\usepackage{epsf}
\usepackage{fancyhdr}
\usepackage{graphics}
\usepackage{graphicx}
\usepackage{psfrag}
\usepackage{microtype}
\usepackage{subfigure}
\usepackage{algorithmic}

\usepackage[linesnumbered,ruled]{algorithm2e}
\DontPrintSemicolon
\usepackage{color}

\usepackage{amsthm}
\usepackage{amsfonts}
\usepackage{amsmath}
\usepackage{mathrsfs}
\usepackage{ushort}
\usepackage{amssymb,bbm}
\usepackage{tikz}
\usepackage{accents}
\usepackage{stmaryrd}

\usepackage{natbib}

\usepackage{url}
\usepackage[colorlinks,linkcolor=magenta,citecolor=blue, pagebackref=true,backref=true]{hyperref}
\renewcommand*{\backrefalt}[4]{%
    \ifcase #1 \footnotesize{(Not cited.)}%
    \or        \footnotesize{(Cited on page~#2.)}%
    \else      \footnotesize{(Cited on pages~#2.)}%
    \fi}

\long\def\comment#1{}
\usepackage{nicefrac}
\usepackage[normalem]{ulem}
\usepackage{chngpage}

 \usepackage{tabularx}%

\usepackage{enumitem}
\usepackage{booktabs}
\usepackage{caption}

\usepackage{mathtools}

\usepackage{fullpage}

\newtheorem{theorem}{Theorem}[section]

\newtheorem{lemma}[theorem]{Lemma}
\newtheorem{proposition}[theorem]{Proposition}

\newtheorem{definition}{Definition}[section]
\newtheorem{example}{Example}[section]

\newtheorem{remark}[theorem]{Remark}
\newtheorem{assumption}[theorem]{Assumption}

\newcommand{\argmin}{\mathop{\rm argmin}}

\newcommand{\dist}{\textnormal{dist}}

\newcommand{\supp}{\textnormal{supp}}

\newcommand{\br}{\mathbb{R}}

\newcommand{\ba}{\begin{array}}
\newcommand{\ea}{\end{array}}

\newcommand{\FCal}{\mathcal{F}}
\newcommand{\RCal}{\mathcal{R}}
\newcommand{\PCal}{\mathcal{P}}

\newcommand{\EE}{{\mathbb{E}}}
\newcommand{\PP}{\mathbb{P}}

\newcommand{\one}{\textbf{1}}

\begin{document}


\begin{center}

{\bf{\LARGE{Online Nonsubmodular Minimization with Delayed \\ [.2cm] Costs: From Full Information to Bandit Feedback}}}

\vspace*{.2in}
{\large{
\begin{tabular}{c}
Tianyi Lin$^{\star, \diamond}$ \and Aldo Pacchiano$^{\star, \ddagger}$ \and Yaodong Yu$^{\star, \diamond}$ \and Michael I. Jordan$^{\diamond, \dagger}$ \\
\end{tabular}
}}

\vspace*{.2in}

\begin{tabular}{c}
Department of Electrical Engineering and Computer Sciences$^\diamond$ \\
Department of Statistics$^\dagger$ \\ 
University of California, Berkeley \\
Microsoft Research, NYC$^\ddagger$
\end{tabular}

\vspace*{.2in}

\today

\vspace*{.2in}

\begin{abstract} 
Motivated by applications to online learning in sparse estimation and Bayesian optimization, we consider the problem of online unconstrained nonsubmodular minimization with delayed costs in both full information and bandit feedback settings. In contrast to previous works on online unconstrained submodular minimization, we focus on a class of nonsubmodular functions with special structure, and prove regret guarantees for several variants of the online and approximate online bandit gradient descent algorithms in static and delayed scenarios. We derive bounds for the agent's regret in the full information and bandit feedback setting, even if the delay between choosing a decision and receiving the incurred cost is unbounded. Key to our approach is the notion of $(\alpha, \beta)$-regret and the extension of the generic convex relaxation model from~\citet{El-2020-Optimal}, the analysis of which is of independent interest. We conduct and showcase several simulation studies to demonstrate the efficacy of our algorithms.
\end{abstract}
\let\thefootnote\relax\footnotetext{$^\star$ Tianyi Lin, Aldo Pacchiano and Yaodong Yu contributed equally to this work.}
\end{center}

\input{sec/intro}

\input{sec/prelim}

\input{sec/result_standard}

\input{sec/result_delay}

\input{sec/exp}

\input{sec/conclu}

\section{Acknowledgments}
This work was supported in part by the Mathematical Data Science program of the Office of Naval Research under grant number N00014-18-1-2764 and by the Vannevar Bush Faculty Fellowship program under grant number N00014-21-1-2941. The work of Michael Jordan is also partially supported by NSF Grant IIS-1901252. 

\bibliographystyle{plainnat}
\bibliography{ref}

\appendix
\input{sec/app_structure}
\input{sec/app_OGD}
\input{sec/app_BGD}
\input{sec/app_DOGD}
\input{sec/app_DBGD}

\end{document}

%% file: sec/intro.tex
\section{Introduction}
With machine learning systems increasingly being deployed in real-world settings, there is an urgent need for online learning algorithms that can minimize cumulative costs over the long run, even in the face of complete uncertainty about future outcomes. There exist a myriad of works that deal with this setting, most prominently in the area of online learning and bandits~\citep{Cesa-2006-Prediction,Lattimore-2020-Bandit}. The majority of this literature deals with problems where the decisions are taken from either a small set (such as in the multi armed bandit framework~\citep{Auer-2002-Using}), a continuous decision space (as in linear bandits~\citep{Auer-2002-Using,Dani-2008-Stochastic}) or in the case the decision set is combinatorial in nature, the response is often assumed to maintain a simple functional relationship with the input (e.g., linear~\cite{Cesa-2012-Combinatorial}). 

In this paper, we depart from these assumptions and explore what we believe is a more realistic type of model for the setting where the actions can be encoded as selecting a subset of a universe of size $n$. We study a sequential interaction between an agent and the world that takes place in rounds. At the beginning of round $t$, the agent chooses a \emph{subset} $S^t \subseteq [n]$ (e.g., selecting the set of products in a factory~\citep{Mccormick-2005-Submodular}), after which the agent suffers cost $f_t(S^t)$ such that $f_t$ is an $\alpha-$weakly DR-submodular and $\beta-$weakly DL-supermodular function~\citep{Lehmann-2006-Combinatorial}. The agent then may receive extra information about $f_t$ as feedback, for example in the full information setting the agent observes the whole function $f_t$ and in the bandit feedback scenario the learner does not receive any extra information about $f_t$ beyond the value of $f_t(S^t)$. The standard metric to measure an online learning algorithm is \textit{regret}~\citep{Blum-2007-External}: the regret at time $T$ is the difference between $\sum_{t=1}^T f_t(S^t)$ that is the total cost achieved by the algorithm and $\min_{x\in A} \sum_{t=1}^T f_t(x)$ that is the total cost achieved by the best fixed action in hindsight. A \textit{no-regret} learning algorithm is one that achieves sublinear regret (as a function of $T$). Many no-regret learning algorithms have been developed based on online convex optimization toolbox~\citep{Zinkevich-2003-Online, Kalai-2005-Efficient, Shalev-2006-Convex, Hazan-2007-Logarithmic, Shalev-2011-Online, Arora-2012-Multiplicative, Hazan-2016-Introduction} many of them achieving minimax-optimal regret bounds for different cost functions even when these are produced by the world in an adversarial fashion. However, many online decision-making problems remain open, for example when the decision space is discrete and large (e.g., exponential in the number of problem parameters) and the cost functions are nonlinear~\citep{Hazan-2012-Online}. 

To the best of our knowledge, ~\citet{Hazan-2012-Online} were the first to investigate non-parametric online learning in combinatorial domains by considering the setting where the costs $f_t$ are all submodular functions. In this formulation the decision space is the set of all subsets of a set of $n$ elements; and the cost functions are \textit{submodular}. They provided no-regret algorithms for both the full information and bandit settings. Their chief innovation was to propose a computationally efficient algorithm for online submodular learning that resolved the exponential computational and statistical dependence on $n$ suffered by all previous approaches~\citep{Hazan-2012-Online}. These results served as a catalyst for a rich and expanding research area~\citep{Streeter-2008-Online, Jegelka-2011-Online, Buchbinder-2014-Online, Chen-2018-Online, Roughgarden-2018-Optimal, Chen-2018-Projection, Cardoso-2019-Differentially, Anari-2019-Structured, Harvey-2020-Improved, Thang-2021-Online, Matsuoka-2021-Tracking}. 

Even though submodularity can be used to model a few important typical cost functions that arise in machine learning problems~\citep{Boykov-2001-Fast, Boykov-2004-Experimental, Narasimhan-2005-Q, Bach-2010-Structured}, it is an insufficient assumption for many other applications where the cost functions do not satisfy submodularity, e.g., structured sparse learning~\citep{El-2015-Totally}, batch Bayesian optimization~\citep{Gonzalez-2016-Batch, Bogunovic-2016-Truncated}, Bayesian A-optimal experimental design~\citep{Bian-2017-Guarantees}, column subset selection~\citep{Sviridenko-2017-Optimal} and so on. In this work we aim to fill in this gap. In view of all this, we consider the following question: 
\begin{center}
\textbf{Can we design online learning algorithms when the cost functions are nonsubmodular?}
\end{center}
This paper provides an affirmative answer to this question by demonstrating that online/bandit approximate gradient descent algorithm can be directly extended from online submodular minimization~\citep{Hazan-2012-Online} to online nonsubmodular minimization when each cost functions $f_t$ satisfy the regularity condition in~\citet{El-2020-Optimal}. 

Moreover, in online decision-making there is often a significant delay between decision and feedback. This delay has an adverse effect on the characterization between marketing feedback and an agent's decision~\citep{Quanrud-2015-Online, Heliou-2020-Gradient}. For example, a click on an ad can be observed within seconds of the ad being displayed, but the corresponding sale can take hours or days to occur. We extend all of our algorithms to the delayed feedback setting by leveraging a pooling strategy recently introduced by~\citet{Heliou-2020-Gradient} into the framework of online/bandit approximate gradient descent. 

\paragraph{Contribution.} First, we introduce a new notion of $(\alpha, \beta)$-regret which allows for analyzing no-regret online learning algorithms when the loss functions are nonsubmodular.  We then propose two randomized algorithms for both the full-information and bandit feedback settings respectively with the regret bounds in expectation and high-probability sense. We then combine the aforementioned algorithms with the pooling strategy found in~\citep{Heliou-2020-Gradient} and prove that the resulting algorithms are no-regret even when the delays are unbounded (cf. Assumption~\ref{Assumption:delay}). Specifically, when the delay $d_t$ satisfies $d_t = o(t^{\gamma})$, we establish a $O(\sqrt{nT^{1+\gamma}})$ regret bound in full-information setting and a $O(nT^{\frac{2+\gamma}{3}})$ regret bound in bandit feedback setting. To our knowledge, this is the first theoretical guarantee for no-regret learning in online nonsubmodular minimization with delayed costs. Experimental results on sparse learning with synthetic data confirm our theoretical findings. 

It is worth comparing our results with that in the existing works~\citep{El-2020-Optimal, Hazan-2012-Online, Heliou-2020-Gradient}. First of all, the results concerning online nonsubmodular minimization are not a straightforward consequence of~\citet{El-2020-Optimal}. Indeed, it is natural yet nontrivial to identify the notion of $(\alpha,\beta)$-regret under which formal guarantees can be established for the nonsubmodular case. This notion does not appear before and appears to be a novel idea and an interesting conceptual contribution. Further, our results provide the first theoretical guarantee for no-regret learning in online and bandit nonsubmodular minimization and generalize the results in~\citet{Hazan-2012-Online}. Even though the online and bandit learning algorithms and regret analysis share the similar spirits with the context of~\citet{Hazan-2012-Online}, the proof techniques are different since we need to deal with the nonsubmodular case with $(\alpha,\beta)$-regret. Finally, we are not aware of any results on online and bandit combinatorial optimization with delayed costs.~\citet{Heliou-2020-Gradient} focused on the gradient-free game-theoretical learning with delayed costs where the action sets are \textit{continuous} and \textit{bounded}. Thus, their results can not imply ours. The only component that two works share is the pooling strategy which has been a common algorithmic component to handle the delays. Even though the pooling strategy is crucial to our delayed algorithms, we make much efforts to combine them properly and prove $(\alpha,\beta)$-regret bound of our new algorithms. 

\paragraph{Notation.} We let $[n]$ be the set $\{1, 2, \ldots, n\}$ and $\br_+^n$ be the set of all vectors in $\br^n$ with nonnegative components.  We denote $2^{[n]}$ as the set of all subsets of $[n]$. For a set $S \subseteq [n]$, we let $\chi_S \in \{0, 1\}^n$ be the characteristic vector satisfying that $\chi_S(i) = 1$ for each $i \in S$ and $\chi_S(i) = 0$ for each $i \notin S$. For a function $f: 2^{[n]} \mapsto \br$, we denote the marginal gain of adding an element $i$ to $S$ by $f(i \mid S) = f(S \cup \{i\}) - f(S)$. In addition, $f$ is normalized if $f(\emptyset) = 0$ and nondecreasing if $f(A) \leq f(B)$ for $A \subseteq B$. For a vector $x \in \br^n$, its Euclidean norm refers to $\|x\|$ and its $i$-th entry refers to $x_i$.  We denote the support set of $x$ by $\supp(x) = \{i \in [n]: x_i \neq 0\}$ and, by abuse of notation, we let $x$ define a set function $x(S) = \sum_{i \in S} x_i$.  We let $P_S$ be the projection onto a closed set $S$ and $\dist(x, S) = \inf_{y \in S} \|x - y\|$ denotes the distance between $x$ and $S$. A pair of parameters $(\alpha, \beta) \in \br_+ \times \br_+$ in the regret refer to approximation factors of the corresponding offline setting. Lastly, $a = O(b(\alpha, \beta, n, T))$ refers to an upper bound $a \leq C \cdot b(\alpha, \beta, n, T)$ where $C > 0$ is independent of $\alpha$, $\beta$, $n$ and $T$. 

\section{Related Work}
The offline nonsubmodular optimization with different notions of approximate submodularity has recently received a lot of attention. Most research focused on the maximization of nonsubmodular set functions, emerging as an important paradigm for studying real-world application problems~\citep{Das-2011-Submodular, Horel-2016-Maximization, Chen-2018-Weakly, Kuhnle-2018-Fast, Hassidim-2018-Optimization, Elenberg-2018-Restricted, Harshaw-2019-Submodular}. In contrast, we are aware of relatively few investigations into the minimization of nonsubmodular set functions. An interesting example is the ratio problem~\citep{Bai-2016-Algorithms} where the objective function to be minimized is the ratio of two set functions and is thus nonsubmodular in general. Note that the ratio problem does not admit a constant factor approximation even when two set functions are submodular~\citep{Svitkina-2011-Submodular}. However, if the objective function to be minimized is approximately modular with bounded curvature,  the optimal approximation algorithms exist even when the constrain sets are assumed~\citep{Iyer-2013-Curvature}. Another typical example is the minimization of the difference of two submodular functions, where some approximation algorithms were proposed in~\citet{Iyer-2012-Algorithms} and~\citet{Kawahara-2015-Approximate} but without any approximation guarantee.  Very recently,~\citet{El-2020-Optimal} provided a comprehensive treatment of optimal approximation guarantees for minimizing nonsubmodular set functions, characterized by how close the function is to submodular. Our work is close to theirs and our results can be interpreted as the extension of~\citet{El-2020-Optimal} to online learning with delayed feedback. 

Another line of relevant works comes from online learning literature and focuses on no-regret algorithms in different settings with delayed costs. In the context of online convex optimization,~\citet{Quanrud-2015-Online} proposed an extension of online gradient descent (OGD) where the agent performs a batched gradient update the moment gradients are received and proved that OGD achieved a regret bound of $O(\sqrt{T + D_T})$ where $D_T$ is the total delay over a horizon $T$.  However, their batch update approach can not be extended to bandit convex optimization since it does not work with stochastic estimates of the received gradient information (or when attempting to infer such information from realized costs). This issue was posted by~\citet{Zhou-2017-Countering} and recently resolved by~\citet{Heliou-2020-Gradient} who proposed a new pooling strategy based on a priority queue.  The effect of delay was also discussed in the multi-armed bandit (MAB) literature under different assumptions~\citep{Joulani-2013-Online, Joulani-2016-Delay, Vernade-2017-Stochastic, Pike-2018-Bandits, Thune-2019-Nonstochastic, Bistritz-2019-Online, Zhou-2019-Learning, Zimmert-2020-Optimal, Gyorgy-2021-Adapting}. In particular,~\citet{Thune-2019-Nonstochastic} proved the regret bound in adversarial MABs with the cumulative delay and~\citet{Gyorgy-2021-Adapting} studied the adaptive tuning to delays and data in this setting. Further,~\citet{Joulani-2016-Delay} and~\citet{Zimmert-2020-Optimal} also investigated adaptive tuning to the unknown sum of delays while~\citet{Bistritz-2019-Online} and~\citet{Zhou-2019-Learning} gave further results in adversarial and linear contextual bandits respectively. However, the algorithms developed in the aforementioned works have little to do with online nonsubmodular minimization with delayed costs.

%% file: sec/prelim.tex
\section{Preliminaries and Technical Background}\label{sec:prelim}
We present the basic setup for minimizing structured nonsubmodular functions, including motivating examples and convex relaxation based on Lov\'{a}sz extension. We extend the offline setting to online setting and $(\alpha, \beta)$-regret which is important to the subsequent analysis. 

\subsection{Structured nonsubmodular function}
Minimizing a set function $f: 2^{[n]} \mapsto \br$ is NP-hard in general but is solved exactly with \textit{submodular} structure in polynomial time~\citep{Iwata-2003-Faster, Grotschel-2012-Geometric, Lee-2015-Faster} and in strongly polynomial time~\citep{Schrijver-2000-Combinatorial, Iwata-2001-Combinatorial, Iwata-2009-Simple, Orlin-2009-Faster, Lee-2015-Faster}. More specifically, $f$ is submodular if it satisfies the diminishing returns (DR) property as follows, 
\begin{equation}\label{def:submodular}
f(i \mid A) \geq f(i \mid B), \quad \textnormal{for all } A \subseteq B, \ i \in [n] \setminus B. 
\end{equation} 
Further, $f$ is modular if the inequality in Eq.~\eqref{def:submodular} holds as an equality and is supermodular if 
\begin{equation*}
f(i \mid B) \geq f(i \mid A), \quad \textnormal{for all } A \subseteq B, \ i \in [n] \setminus B. 
\end{equation*}
Relaxing these inequalities will bring us the notions of weak DR-submodularity/supermodularity that were introduced by~\citet{Lehmann-2006-Combinatorial} and revisited in the machine learning literature~\citep{Bian-2017-Guarantees}.  Formally, we have
\begin{definition}
A set function $f: 2^{[n]} \mapsto \br$ is $\alpha$-weakly DR-submodular with $\alpha > 0$ if 
\begin{equation*}
f(i \mid A) \geq \alpha f(i \mid B), \quad \textnormal{for all } A \subseteq B, \ i \in [n] \setminus B. 
\end{equation*}
Similarly,  $f$ is $\beta$-weakly DR-supermodular with $\beta > 0$ if 
\begin{equation*}
f(i \mid B) \geq \beta f(i \mid A), \quad \textnormal{for all } A \subseteq B, \ i \in [n] \setminus B. 
\end{equation*}
We say that $f$ is $(\alpha, \beta)$-weakly DR-modular if both of the above two inequalities hold true. 
\end{definition}
The above notions of weak DR-submodularity (or weak DR-supermodularity) generalize the notions of submodularity (or supermodularity); indeed, we have $f$ is submodular (or supermodular) if and only if $\alpha = 1$ (or $\beta = 1$).  They are also special cases of more general notions of weak submodularity (or weak supermodularity)~\citep{Das-2011-Submodular} and we refer to~\citet[Proposition~1]{Bogunovic-2018-Robust} and~\citet[Proposition~8]{El-2018-Combinatorial} for the details. For an overview of the approximate submodularity,  we refer to~\citet[Section~6]{Bian-2017-Guarantees} and~\citet[Figure~1]{El-2020-Optimal}. In addition, the parameters $1 - \alpha$ and $1 - \beta$ are referred to as \textit{generalized inverse curvature} and \textit{generalized curvature} respectively~\citep{Bian-2017-Guarantees, Bogunovic-2018-Robust} and can be interpreted as the extension of inverse curvature and curvature~\citep{Conforti-1984-Submodular} for submodular and supermodular functions. Intuitively, these parameters quantify how far the function $f$ is from being a submodular (or supermodular) function. 

Recently,~\citet{El-2020-Optimal} have proposed and studied the problem of minimizing a class of structured nonsubmodular functions as follows, 
\begin{equation}\label{prob:offline}
\min_{S \subseteq [n]} \ f(S) := \bar{f}(S) - \ushort{f}(S), 
\end{equation}
where $\bar{f}$ and $\ushort{f}$ are both normalized (i.e., $\bar{f}(\emptyset) = \ushort{f}(\emptyset) = 0$)\footnote{In general, we can let $\bar{f}(S) \leftarrow \bar{f}(S) - \bar{f}(\emptyset)$ and $\ushort{f}(S) \leftarrow \ushort{f}(S) - \ushort{f}(\emptyset)$ which will not change the minimization problem.} and nondecreasing, $\bar{f}$ is $\alpha$-weakly DR-submodular and $\ushort{f}$ is $\beta$-weakly DR-supermodular. Note that the problem in Eq.~\eqref{prob:offline} is challenging; indeed, $f$ is neither weakly DR-submodular nor weakly DR-supermodular in general since the weak DR-submodularity (or weak DR-supermodularity) are only valid for monotone functions. 

It is worth mentioning that Eq.~\eqref{prob:offline} is not necessarily theoretically artificial but encompasses a wide range of applications. We present two typical examples which can be formulated in the form of Eq.~\eqref{prob:offline} and refer to~\citet[Section~4]{El-2020-Optimal} for more details.

\begin{example}[Structured Sparse Learning]\label{def:SSL}
We aim to estimate a sparse parameter vector whose support satisfies a particular structure and commonly formulate such problems as $\min_{x \in \br^n} \ell(x) + \lambda f(\supp(x))$, where $\ell: \br^n \mapsto \br$ is a loss function and $f: 2^{[n]} \mapsto \br$ is a set function favoring the desirable supports.  Existing approaches such as~\citep{Bach-2010-Structured} proposed to replace the discrete regularization function $f(\supp(x))$ by its closest convex relaxation and is computationally tractable only when $f$ is submodular. However, this problem is often better modeled by a nonsubmodular regularizer in practice~\citep{El-2015-Totally}.  An alternative formulation of structured sparse learning problems is 
\begin{equation}\label{prob:SSL}
\min_{S \subseteq [n]} \lambda f(S) - h(S), 
\end{equation}
where $h(S) = \ell(0) - \min_{\supp(x) \subseteq S} \ell(x)$. Note that Eq.~\eqref{prob:SSL} can be reformulated into the form of Eq.~\eqref{prob:offline} under certain conditions; indeed, $h$ is a normalized and nondecreasing function and~\citet[Proposition~5]{El-2020-Optimal} has shown that $h$ is weakly DR-modular if $\ell$ is smooth, strongly convex and is generated from random data.  Examples of weakly DR-submodular regularizers $f$ include the ones used in time-series and cancer diagnosis~\citep{Rapaport-2008-Classification} and healthcare~\citep{Sakaue-2019-Greedy}. 
\end{example}
\begin{example}[Batch Bayesian Optimization]
We aim to optimize an unknown expensive-to-evaluate noisy function $\ell$ with as few batches of function evaluations as possible. The evaluation points are chosen to maximize an acquisition function -- the variance reduction function~\citep{Gonzalez-2016-Batch} -- subject to a cardinality constraint. Maximizing the variance reduction may be phrased as a special instance of the problems in Eq.~\eqref{prob:offline} in the form of $\min_{S \subseteq [n]} \lambda |S| - G(S)$, where $G: 2^{[n]} \mapsto \br$ is the variance reduction function defined accordingly and~\citet[Proposition~6]{El-2020-Optimal} has shown that it is also non-decreasing and weakly DR-modular. This formulation allows to include nonlinear costs with (weak) decrease in marginal costs (economies of scale) with  some applications in the sensor placement. 
\end{example}

\subsection{Convex relaxation based on the Lov\'{a}sz extension}
The Lov\'{a}sz extension~\citep{Lovasz-1983-Submodular} is a toolbox commonly used for minimizing a submodular set function $f: 2^{[n]} \mapsto \br$.  It is a continuous interpolation of $f$ on the unit hypercube $[0, 1]^n$ and can be minimized efficiently since it is \textit{convex} if and only if $f$ is submodular. The minima of the Lov\'{a}sz extension also recover the minima of $f$.

Before the formal argument, we define a maximal chain of $[n]$; that is, $\{A_0, \ldots, A_n\}$ is a maximal chain if $\emptyset = A_0 \subseteq A_1 \subseteq \ldots \subseteq A_n = [n]$.  Formally, we have
\begin{definition}\label{def:Lovasz}
Given a submodular function $f$, the Lov\'{a}sz extension is the function $f_L: [0, 1]^n \mapsto \br$ given by $f_L(x) = \sum_{i=0}^n \lambda_i f(A_i)$ where $\{A_0, \ldots, A_n\}$ is a maximal chain\footnote{Both the chain and the set of $\lambda_i$ may depend on the input $x$.} of $[n]$ so that $\sum_{i=0}^n \lambda_i \chi_{A_i} = x$ and $\sum_{i=0}^n \lambda_i = 1$ where $\chi_{A_i}(j) = 1$ for $\forall j \in A_i$ and $\chi_{A_i}(j) = 0$ for $\forall j \notin S$. 
\end{definition}
Even though Definition~\ref{def:Lovasz} implies that $f_L(\chi_S) = f(S)$ for all $S \subseteq [n]$, it remains unclear how to find the chain or the coefficients.  The preceding discussion defines the Lov\'{a}sz extension in an equivalent way that is more amenable for computing the subgradient of $f_L$.  

Let $x = (x_1, x_2, \ldots, x_n) \in [0, 1]^n$ and we define that $\pi: [n] \mapsto [n]$ is the sorting permutation of $\{x_1, x_2, \ldots, x_n\}$ where $\pi(i) = j$ implies that $x_j$ is the $i$-th largest element. By definition, we have $1 \geq x_{\pi(1)} \geq \ldots \geq x_{\pi(n)} \geq 0$ and let $x_{\pi(0)} = 1$ and $x_{\pi(n+1)} = 0$ for simplicity. Then, we set $\lambda_i = x_{\pi(i)} - x_{\pi(i+1)}$ for all $0 \leq i \leq n$ and let $A_0 = \emptyset$ and $A_i = \{\pi(1), \ldots, \pi(i)\}$ for all $i \in [n]$.  We also have
\begin{eqnarray*}
\lefteqn{\sum_{i=0}^n \lambda_i \chi_{A_i} = \sum_{i=0}^n (x_{\pi(i)} - x_{\pi(i+1)})(\chi_{A_{i-1}} + e_{\pi(i)})} \\ 
& = & \sum_{i=1}^n e_{\pi(i)}\sum_{j=i}^n (x_{\pi(j)} - x_{\pi(j+1)}) = x. 
\end{eqnarray*}
As such, we obtain that $f_L(x) = \sum_{i=1}^n x_{\pi(i)} f(\pi(i) \mid A_{i-1})$ where $x_{\pi(1)} \geq x_{\pi(2)} \geq \ldots \geq x_{\pi(n)}$ are the sorted entries in decreasing order, $A_0 = \emptyset$ and $A_i = \{\pi(1), \ldots, \pi(i)\}$ for all $i \in [n]$. Then, the classical results~\citep{Edmonds-2003-Submodular, Fujishige-2005-Submodular} suggest that the subgradient $g$ of $f_L$ at any $x \in [0, 1]^n$ can be computed by simply sorting the entries in decreasing order and taking 
\begin{equation}
g_{\pi(i)} = f(A_i) - f(A_{i-1}), \textnormal{ for all } i \in [n]. 
\end{equation}
Since $f_L$ is convex if and only if $f$ is submodular, we can apply the convex optimization toolbox here. Recently, ~\citet{El-2020-Optimal} have shown that the similar idea can be extended to nonsubmodular optimization in Eq.~\eqref{prob:offline}. 

More specifically, we can define the convex closure $f_C$ for any nonsubmodular function $f$; indeed, $f_C: [0, 1]^n \mapsto \br$ is the point-wise largest convex function which always lower bounds $f$. By definition, $f_C$ is the \textit{tightest} convex extension of $f$ and $\min_{S \subseteq [n]} f(S) = \min_{x \in [0, 1]^n} f_C(x)$. In general, it is NP-hard to evaluate and optimize $f_C$~\citep{Vondrak-2007-Submodularity}. Fortunately,~\citet{El-2020-Optimal} demonstrated that the Lov\'{a}sz extension $f_L$ approximates $f_C$ such that the vector computed using the approach in~\citet{Edmonds-2003-Submodular} and~\citet{Fujishige-2005-Submodular} approximates the subgradient of $f_C$. We summarize their results in the following proposition and provide the proofs in Appendix~\ref{app:structure} for completeness. 
\begin{proposition}\label{Prop:structure}
Focusing on Eq.~\eqref{prob:offline}, we let $x \in [0, 1]^n$ with $x_{\pi(1)} \geq \ldots \geq x_{\pi(n)}$ and $g_{\pi(i)} = f(A_i) - f(A_{i-1})$ for all $i \in [n]$ where $A_0 = \emptyset$ and $A_i = \{\pi(1), \ldots, \pi(i)\}$ for all $i \in [n]$.  Then, we have
\begin{equation}\label{Prop:structure-first}
f_L(x) = g^\top x \geq f_C(x), 
\end{equation}
and 
\begin{equation}\label{Prop:structure-second}
g(A) = \sum_{i \in A} g_i \leq \tfrac{1}{\alpha}\bar{f}(A) - \beta\ushort{f}(A),  \textnormal{ for all } A \subseteq [n], 
\end{equation}
and 
\begin{equation}\label{Prop:structure-third}
g^\top z \leq \tfrac{1}{\alpha}\bar{f}_C(z) + \beta (-\ushort{f})_C(z), \textnormal{ for all } z \in [0, 1]^n.  
\end{equation}
\end{proposition}
Proposition~\ref{Prop:structure} highlights how $f_L$ approximates $f_C$; indeed, we see from Eq.~\eqref{Prop:structure-first} and Eq.~\eqref{Prop:structure-third} that $f_C(x) \leq f_L(x) \leq \tfrac{1}{\alpha}\bar{f}_C(x) + \beta (-\ushort{f})_C(x)$ for all $x \in [0, 1]^n$.  As such, it gives the key insight for analyzing the offline algorithms in~\citet{El-2020-Optimal} and will play an important role in the subsequent analysis of our paper.   

\subsection{Online nonsubmodular minimization}
We consider online nonsubmodular minimization which extends the offline problem in Eq.~\eqref{prob:offline} to the online setting.  In particular, an adversary first chooses structured nonsubmodular functions $f_1, f_2, \ldots, f_T: 2^{[n]} \mapsto \br$ given by 
\begin{equation}\label{prob:online}\small
f_t(S) := \bar{f}_t(S) - \ushort{f}_t(S),  \textnormal{ for all } S \subseteq [n], \ t \in [T], 
\end{equation}
where $\bar{f}_t$ and $\ushort{f}_t$ are normalized and non-decreasing, $\bar{f}_t$ is $\alpha$-weakly DR-submodular and $\ushort{f}_t$ is $\beta$-weakly DR-supermodular. In each round $t = 1, 2,\ldots, T$, the agent chooses $S^t$ and observes the incurred loss $f_t(S^t)$ after committing to her decision. Throughout the horizon $[0, T]$, one aims to minimize the regret -- the difference between $\sum_{t=1}^T f_t(S^t)$ and the loss at the best fixed solution in hindsight, i.e., $S_\star^T = \argmin_{S \subseteq [n]} \sum_{t=1}^T f_t(S)$ -- which is defined by\footnote{If the sets $S^t$ are chosen by a randomized algorithm, we consider the expected regret over the randomness.} 
\begin{equation}\label{eq:regret}\small
R(T) = \sum_{t=1}^T f_t(S^t) - \sum_{t=1}^T f_t(S_\star^T). 
\end{equation}
An algorithm is \textit{no-regret} if $R(T)/T \rightarrow 0$ as $T \rightarrow +\infty$ and \textit{efficient} if it computes each decision set $S^t$ in polynomial time. In the context, the regret is used when the minimization for a known cost, i.e., $\min_{S \subseteq [n]} f(S)$, can be solved exactly.  However, solving the optimization problem in Eq.~\eqref{prob:offline} with nonsubmodular costs is NP-hard regardless of any multiplicative constant factor~\citep{Iyer-2012-Algorithms, Trevisan-2014-Inapproximability}. Thus, it is necessary to consider a bicriteria-like approximation guarantee with the factors $\alpha, \beta > 0$ as~\citet{El-2020-Optimal} suggested. In particular, $(\alpha, \beta)$ are bounds on the quality of a solution $S$ returned by a given offline algorithm compared to the optimal solution $\bar{S}$; that is, $f(S) \leq \frac{1}{\alpha}\bar{f}(S_\star) - \beta \ushort{f}(S_\star)$.  Such bicriteria-like approximation is optimal:~\citet[Theorem~2]{El-2020-Optimal} has shown that no algorithm with subexponential number of value queries can improve on it in the oracle model. 

Our goal is to analyze \textit{online approximate gradient descent algorithm and its bandit variant} for online nonsubmodular minimization. Let $(\alpha, \beta)$ be the approximation factors attained by an offline algorithm that solves $\min_{S \subseteq [n]} f(S)$ for a known nonsubmodular function $f$ in Eq.~\eqref{prob:offline}. The \textit{$(\alpha, \beta)$-regret} compares to the best solution that can be expected in polynomial time and is defined by 
\begin{equation}\label{eq:alpha-beta-regret}\small
R_{\alpha, \beta}(T) = \sum_{t=1}^T f_t(S^t) - \sum_{t=1}^T (\tfrac{1}{\alpha}\bar{f}_t(S_\star^T) - \beta\ushort{f}_t(S_\star^T)), 
\end{equation}
where $S_\star^T = \argmin_{S \subseteq [n]} \sum_{t=1}^T f_t(S)$.  It is analogous to the $\alpha$-regret which is widely used in online constrained submodular minimization~\citep{Jegelka-2011-Online} and online submodular maximization~\citep{Streeter-2008-Online}.  

As mentioned before, we consider the algorithmic design in both \textit{full information} and \textit{bandit feedback} settings. In the former one, the agent is allowed to have unlimited access to the value oracles of $f_t(\cdot)$ after choosing $S^t$ in each round $t$. In the latter one, the agent only observes the incurred loss at the point that she has chosen in each round $t$, i.e., $f_t(S^t)$, and receives no other information. 

%% file: sec/result_standard.tex
\section{Online Approximation Algorithm}\label{sec:standard}
We analyze online approximate gradient descent algorithm and its bandit variant for regret minimization when the nonsubmodular cost functions are in the form of Eq~\eqref{prob:online}. Due to space limit, we defer the proofs to Appendix~\ref{app:OGD} and~\ref{app:BGD}. 
\begin{algorithm}[!t]
\caption{Online Approximate Gradient Descent}\label{alg:OGD}
\begin{algorithmic}[1]
\STATE \textbf{Initialization:} the point $x^1 \in [0, 1]^n$ and the stepsize $\eta > 0$;
\FOR{$t = 1, 2, \ldots$}
\STATE Let $x_{\pi(1)}^t \geq \ldots x_{\pi(n)}^t$ be the sorted entries in the decreasing order with $A_i^t = \{\pi(1), \ldots, \pi(i)\}$ for all $i \in [n]$ and $A_0^t = \emptyset$. Let $x_{\pi(0)}^t = 1$ and $x_{\pi(n+1)}^t = 0$. 
\STATE Let $\lambda_i^t = x_{\pi(i)}^t - x_{\pi(i+1)}^t$ for all $0 \leq i \leq n$. 
\STATE Sample $S^t$ from the distribution $\PP(S^t = A_i^t) = \lambda_i^t$ for all $0 \leq i \leq n$ and observe the new loss function $f_t$. 
\STATE Compute $g_{\pi(i)}^t = f_t(A_i^t) - f_t(A_{i-1}^t)$ for all $i \in [n]$.  
\STATE Compute $x^{t+1} = P_{[0, 1]^n}(x^t - \eta g^t)$. 
\ENDFOR
\end{algorithmic}
\end{algorithm}

\subsection{Full information setting}
Let $[0, 1]^n$ be the unit hypercube and the cost function on $[0, 1]^n$ corresponding to $f_t$ is the function $(f_t)_C$ that is the convex closure of $f_t$.  Equipped with Proposition~\ref{Prop:structure}, we can compute approximate subgradients of $(f_t)_C$ such that the online gradient descent~\citep{Zinkevich-2003-Online} is applicable. 

This leads to Algorithm~\ref{alg:OGD} which performs one-step projected gradient descent that yields $x^t$ and then samples $S^t$ from the distribution $\lambda$ over $\{A_i\}_{i= 0}^n$ encoded by $x^t$. It is worth mentioning that $\lambda_i^t = x_{\pi(i)}^t - x_{\pi(i+1)}^t$ for all $0 \leq i \leq n$ and $\lambda$ is thus completely independent of $f_t$. This guarantees that Algorithm~\ref{alg:OGD} is valid in online manner since $f_t$ is realized after the decision maker chooses $S^t$. One of the advantages of Algorithm~\ref{alg:OGD} is that it does not require the value of $\alpha$ and $\beta$ which can be hard to compute in practice. We summarize our results for Algorithm~\ref{alg:OGD} in the following theorem. 
\begin{theorem}\label{Theorem:OGD}
Suppose the adversary chooses nonsubmodular functions in Eq.~\eqref{prob:online} satisfying $\bar{f}_t([n]) + \ushort{f}_t([n]) \leq L$. Fixing $T \geq 1$ and letting $\eta = \frac{\sqrt{n}}{L\sqrt{T}}$ in Algorithm~\ref{alg:OGD}, we have $\EE[R_{\alpha, \beta}(T)] = O(\sqrt{nT})$ and $R_{\alpha, \beta}(T) = O(\sqrt{nT} + \sqrt{T\log(1/\delta)})$ with probability $1 - \delta$. 
\end{theorem}
\begin{remark}
Theorem~\ref{Theorem:OGD} demonstrates that Algorithm~\ref{alg:OGD} is regret-optimal for our setting; indeed, our setting includes online unconstrained submodular minimization as a special case where $(\alpha, \beta)$-regret becomes standard regret in Eq.~\eqref{eq:regret} and~\citet{Hazan-2012-Online} shows that Algorithm~\ref{alg:OGD} is optimal up to constants. Our theoretical result also extends the results in~\citet{Hazan-2012-Online} from submodular cost functions to nonsubmodular cost functions in Eq.~\eqref{prob:online} using the $(\alpha, \beta)$-regret instead of the standard regret in Eq.~\eqref{eq:regret}.  
\end{remark}
\begin{algorithm}[!t]
\caption{Bandit Approximate Gradient Descent}\label{alg:BGD}
\begin{algorithmic}[1]
\STATE \textbf{Initialization:} the point $x^1 \in [0, 1]^n$ and the stepsize $\eta > 0$; the exploration probability $\mu \in (0,1)$. 
\FOR{$t = 1, 2, \ldots, T$}
\STATE Let $x_{\pi(1)}^t \geq \ldots x_{\pi(n)}^t$ be the sorted entries in decreasing order with $A_i^t = \{\pi(1), \ldots, \pi(i)\}$ for all $i \in [n]$ and $A_0^t = \emptyset$. Let $x_{\pi(0)}^t = 1$ and $x_{\pi(n+1)}^t = 0$.  
\STATE Let $\lambda_i^t = x_{\pi(i)}^t - x_{\pi(i+1)}^t$ for all $0 \leq i \leq n$. 
\STATE Sample $S^t$ from the distribution $\PP(S^t = A_i^t) = (1 - \mu)\lambda_i^t + \frac{\mu}{n+1}$ for all $0 \leq i \leq n$ and observe the loss $f_t(S^t)$. 
\STATE Compute $\hat{f}_i^t = \frac{\mathbf{1}(S^t = A^t_i)}{(1-\mu)\lambda_i^t + \mu/(n+1)}f_t(S^t)$ for all $0 \leq i \leq n$.
\STATE Compute $\hat{g}^t_{\pi(i)} = \hat{f}_i^t - \hat{f}_{i-1}^t$ for all $i \in [n]$.
\STATE Compute $x^{t+1}$ $x^{t+1} = P_{[0, 1]^n}(x^t - \eta g^t)$. 
\ENDFOR
\end{algorithmic}
\end{algorithm}
\subsection{Bandit feedback setting}
In contrast with the full-information setting, the agent only observes the loss function $f_t$ at her action $S^t$, i.e., $f_t(S^t)$, in bandit feedback setting.  This is a more challenging setup since the agent does not have full access to the new loss function $f_t$ at each round $t$ yet. 

Despite the bandit feedback, we can compute an unbiased estimator of the gradient $g^t$ in Algorithm~\ref{alg:OGD} using the technique of importance weighting and try to implement a stochastic version of Algorithm~\ref{alg:OGD}. More specifically, we notice that $\hat{f}_i^t = \frac{\one(S^t = A_i^t)}{\lambda_i^t} f_t(S^t)$ is unbiased for estimating $f_t(A_i^t)$ for all $0 \leq i \leq n$. Thus, $\hat{g}_{\pi(i)}^t = \hat{f}_i^t - \hat{f}_{i-1}^t$ for all $i \in [n]$ gives us an unbiased estimator of the gradient $g^t$. However, the variance of the estimator $\hat{g}$ could be undesirably large since the values of $\lambda_i^t$ may be arbitrarily small.

To resolve this issue, we can sample $S^t$ from a mixture distribution that combines (with probability $1-\mu$) samples from $\lambda^t$ and (with probability $\mu$) samples from the uniform distribution over $\{A_i^t\}_{i= 0}^n$. This guarantees that the variance of $\hat{f}_i^t$ is upper bounded by $O(n^2/\mu)$. The similar idea has been employed in~\citet{Hazan-2012-Online} for online submodular minimization. Then, we conduct the careful analysis for the estimators $\hat{g}^t$ such that the scale of the variance is taken into account. Note that our analysis is different from the standard analysis in~\citet{Flaxman-2005-Online} which seems oversimplified for our setting and results in worse regret of $O(T^{3/4})$ compared to our result in the following theorem. 
\begin{theorem}\label{Theorem:BGD}
Suppose the adversary chooses nonsubmodular functions $f_t$ in Eq.~\eqref{prob:online} satisfying $\bar{f}_t([n]) + \ushort{f}_t([n]) \leq L$.  Fixing $T \geq 1$ and letting $(\eta, \mu) = (\frac{1}{LT^{2/3}}, \frac{n}{T^{1/3}})$ in Algorithm~\ref{alg:BGD}, we have $\EE[R_{\alpha, \beta}(T)] = O(nT^{\frac{2}{3}})$ and $R_{\alpha, \beta}(T) = O(nT^{\frac{2}{3}} + \sqrt{n\log(1/\delta)}T^{\frac{2}{3}})$ with probability $1-\delta$.
\end{theorem}
\begin{remark}
Theorem~\ref{Theorem:BGD} demonstrates that Algorithm~\ref{alg:BGD} is no-regret for our setting even when only the bandit feedback is available, further extending the results in~\citet{Hazan-2012-Online} from submodular cost functions to nonsubmodular cost functions in Eq.~\eqref{prob:online} using the $(\alpha, \beta)$-regret instead of the standard regret in Eq.~\eqref{eq:regret}.  
\end{remark}

%% file: sec/result_delay.tex
\section{Online Delayed Approximation Algorithm}\label{sec:delay}
We investigate Algorithm~\ref{alg:OGD} and~\ref{alg:BGD} for regret minimization even when the delay between choosing an action and receiving the incurred cost exists and can be unbounded. 

\subsection{The general framework}
The general online learning framework with large delay that we consider can be represented as follows. In each round $t = 1, \ldots, T$, the agent chooses the decision $S^t \subseteq [n]$ and this generates a loss $f_t(S^t)$. Simultaneously, $S^t$ triggers a \textit{delay} $d_t \geq 0$ which determines the round $t + d_t$ at which the information about $f_t$ will be received. Finally, the agent receives the information about $f_t$ from all previous rounds $\RCal_t = \{s: s + d_s = t\}$. 

The above model has been stated in an abstract way as the basis for the regret analysis. The information about $f_t$ is determined by whether the setting is full information or bandit feedback.  Our blanket assumptions for the stream of the delays encountered will be: 
\begin{assumption}\label{Assumption:delay}
The delays $d_t = o(t^\gamma)$ for some $\gamma < 1$. 
\end{assumption}
Assumption~\ref{Assumption:delay} is not theoretically artificial but uncovers that long delays are observed in practice~\citep{Chapelle-2014-Modeling}; indeed, the data statistics from real-time bidding company suggested that more than $10\%$ of the conversions were $ \geq 2$ weeks old. More specifically,~\citet{Chapelle-2014-Modeling} showed that the delays in online advertising have long-tail distributions when conditioning on context and feature variables available to the advertiser, thus justifying the existence of unbounded delays. Note that Assumption~\ref{Assumption:delay} is mild and the delays can even be adversarial as in~\citet{Quanrud-2015-Online}. 
\begin{algorithm}[!t]
\caption{Delay Online Approximate Gradient Descent}\label{alg:DOGD}
\begin{algorithmic}[1]
\STATE \textbf{Initialization:} the point $x^1 \in [0, 1]^n$ and the stepsize $\eta_t > 0$; $\PCal_0 \leftarrow \emptyset$ and $f_\infty = 0$. 
\FOR{$t = 1, 2, \ldots$}
\STATE Let $x_{\pi(1)}^t \geq \ldots x_{\pi(n)}^t$ be the sorted entries in the decreasing order with $A_i^t = \{\pi(1), \ldots, \pi(i)\}$ for all $i \in [n]$ and $A_0^t = \emptyset$. Let $x_{\pi(0)}^t = 1$ and $x_{\pi(n+1)}^t = 0$. 
\STATE Let $\lambda_i^t = x_{\pi(i)}^t - x_{\pi(i+1)}^t$ for all $0 \leq i \leq n$. 
\STATE Sample $S^t$ from the distribution $\PP(S^t = A_i^t) = \lambda_i^t$ for $0 \leq i \leq n$ and observe the new loss function $f_t$. 
\STATE Compute $g_{\pi(i)}^t = f_t(A_i^t) - f_t(A_{i-1}^t)$ for all $i \in [n]$ and then trigger a delay $d_t \geq 0$. 
\STATE Let $\RCal_t = \{s: s + d_s = t\}$ and $\PCal_t \leftarrow \PCal_{t-1} \cup \RCal_t$.  Take $q_t = \min \PCal_t$ and set $\PCal_t \leftarrow \PCal_t \setminus \{q_t\}$. 
\STATE Compute $x^{t+1}$ using Eq.~\eqref{def:DOGD_main}. 
\ENDFOR
\end{algorithmic}
\end{algorithm}

\subsection{Full information setting}
At the round $t$, the agent receives the loss function $f_s(\cdot)$ for $\RCal_t = \{s: s + d_s = t\}$ after committing her decision, i,e., gets to observe $f_s(A_i^t)$ for all $s \in \RCal_t$ and all $0 \leq i \leq n$. To let Algorithm~\ref{alg:OGD} handle these delays, the first thing to note is that the set $\RCal_t$ received at a given round might be empty, i.e., we could have $\RCal_t = \emptyset$ for some $t \geq 1$. Following up the pooling strategy in~\citet{Heliou-2020-Gradient}, we assume that, as information is received over time, the agent adds it to an information pool $\PCal_t$ and then uses the oldest information available in the pool (where ``oldest" stands for the time at which the information was generated). 

Since no information is available at $t = 0$, we have $\PCal_0 = \emptyset$ and update the agent's information pool recursively: $\PCal_t = \PCal_{t-1} \cup \RCal_t \setminus \{q_t\}$ where $q_t = \min(\PCal_{t-1} \cup \RCal_t)$ denotes the oldest round from which the agent has unused information at round $t$. As~\citet{Heliou-2020-Gradient} pointed out, this scheme can be seen as a priority queue where $\{f_s(\cdot), s \in \RCal_t\}$ arrives at time $t$ and is assigned in order; subsequently, the oldest information is utilized at first. An important issue that arises in the above computation is that, it may well happen that the agent's information pool $\PCal_t$ is empty at time $t$ (e.g., if we have $d_1 > 0$ at time $t=1$). Following the convention that $\inf \emptyset = +\infty$, we set $q_t = +\infty$ and $g^\infty = 0$ (since it is impossible to have information at time $t=+\infty$). Under this convection, the computation of a new iterate $x^{t+1}$ at time $t$ can be written more explicitly form as follows, 
\begin{equation}\label{def:DOGD_main}
x^{t+1} = \left\{ 
\begin{array}{ll}
x^t & \textnormal{if } \PCal_t = \emptyset, \\
P_{[0, 1]^n}(x^t - \eta_t g^{q_t}), & \textnormal{otherwise}.
\end{array}
\right.
\end{equation}
We present a delayed variant of Algorithm~\ref{alg:OGD} in Algorithm~\ref{alg:DOGD}. There is no information aggregation here but the updates of $x^{t+1}$ follows the pooling policy induced by a priority queue.  We summarize our results in the following theorem. 
\begin{theorem}\label{Theorem:DOGD}
Suppose the adversary chooses nonsubmodular functions in Eq.~\eqref{prob:online} satisfying $\bar{f}_t([n]) + \ushort{f}_t([n]) \leq L$ and let the delays satisfy Assumption~\ref{Assumption:delay}. Fixing $T \geq 1$ and letting $\eta_t = \frac{\sqrt{n}}{L\sqrt{t^{1+\gamma}}}$ in Algorithm~\ref{alg:DOGD}, we have $\EE[R_{\alpha, \beta}(T)] = O(\sqrt{nT^{1+\gamma}})$ and $R_{\alpha, \beta}(T) = O(\sqrt{nT^{1+\gamma}} + \sqrt{T\log(1/\delta)})$ with probability $1 - \delta$. 
\end{theorem}
\begin{remark}
Theorem~\ref{Theorem:DOGD} demonstrates that Algorithm~\ref{alg:DOGD} is no-regret if Assumption~\ref{Assumption:delay} hold. To our knowledge, this is the first theoretical guarantee for no-regret learning in online nonsubmodular minimization with delayed costs and also complement similar results for online convex optimization with delayed costs~\citep{Quanrud-2015-Online}. 
\end{remark}
\begin{algorithm}[!t]
\caption{Delay Bandit Approximate Gradient Descent}\label{alg:DBGD}
\begin{algorithmic}[1]
\STATE \textbf{Initialization:} the point $x^1 \in [0, 1]^n$ and the stepsize $\eta_t > 0$; $\PCal_0 \leftarrow \emptyset$ and $f_\infty = 0$; the exploration probability $\mu_t \in (0, 1)$. 
\FOR{$t = 1, 2, \ldots$}
\STATE Let $x_{\pi(1)}^t \geq \ldots x_{\pi(n)}^t$ be the sorted entries in the decreasing order with $A_i^t = \{\pi(1), \ldots, \pi(i)\}$ for all $i \in [n]$ and $A_0^t = \emptyset$. Let $x_{\pi(0)}^t = 1$ and $x_{\pi(n+1)}^t = 0$. 
\STATE Let $\lambda_i^t = x_{\pi(i)}^t - x_{\pi(i+1)}^t$ for all $0 \leq i \leq n$. 
\STATE Sample $S^t$ from the distribution $\PP(S^t = A_i^t) = (1-\mu_t)\lambda_i^t + \frac{\mu_t}{n+1}$ for $0 \leq i \leq n$ and observe the loss $f_t(S^t)$. 
\STATE Compute $\hat{f}_i^t = \frac{\one(S^t = A_i^t)}{(1-\mu_t)\lambda_i^t + \mu_t/(n+1)}f_t(S^t)$.
\STATE Compute $\hat{g}^t_{\pi(i)} = \hat{f}_i^t - \hat{f}_{i-1}^t$ for all $i \in [n]$ and then trigger a delay $d_t \geq 0$. 
\STATE Let $\RCal_t = \{s: s + d_s = t\}$ and $\PCal_t \leftarrow \PCal_{t-1} \cup \RCal_t$.  Take $q_t = \min \PCal_t$ and set $\PCal_t \leftarrow \PCal_t \setminus \{q_t\}$. 
\STATE Compute $x^{t+1}$ using Eq.~\eqref{def:DBGD_main}. 
\ENDFOR
\end{algorithmic}
\end{algorithm}
\subsection{Bandit feedback setting}
As we have done in the previous section, we will make use of an unbiased estimator $\hat{g}$ of the gradient for the bandit feedback setting.  However, we only receive the old estimator $\hat{g}^{q_t}$ at the round $t$ due to the delay $d_t$. Following the same reasoning as in the full information setting, the computation of a new iterate $x^{t+1}$ at time $t$ can be written more explicitly form as follows, 
\begin{equation}\label{def:DBGD_main}
x^{t+1} = \left\{ 
\begin{array}{ll}
x^t & \textnormal{if } \PCal_t = \emptyset, \\
P_{[0, 1]^n}(x^t - \eta_t \hat{g}^{q_t}), & \textnormal{otherwise}.
\end{array}
\right.
\end{equation}
Algorithm~\ref{alg:DBGD} follows the same template as Algorithm~\ref{alg:DOGD} but substituting the exact gradients with the gradient estimator. We summarize our results in the following theorem. 
\begin{theorem}\label{Theorem:DBGD}
Suppose the adversary chooses nonsubmodular functions in Eq.~\eqref{prob:online} satisfying $\bar{f}_t([n]) + \ushort{f}_t([n]) \leq L$ and let the delays satisfy Assumption~\ref{Assumption:delay}. Fixing $T \geq 1$ and letting $(\eta_t, \mu_t) = (\frac{1}{Lt^{(2+\gamma)/3}}, \frac{n}{t^{(1-\gamma)/3}})$ in Algorithm~\ref{alg:DBGD}, we have $\EE[R_{\alpha, \beta}(T)] = O(n T^{\frac{2+\gamma}{3}})$ and $R_{\alpha, \beta}(T) = O(n T^{\frac{2+\gamma}{3}} + \sqrt{n\log(1/\delta)}T^{\frac{4-\gamma}{6}})$ with probability $1 - \delta$. 
\end{theorem}
\begin{remark} 
Theorem~\ref{Theorem:DBGD} demonstrates that Algorithm~\ref{alg:DBGD} attains the regret of $nT^{\frac{2+\gamma}{3}}$ which is worse that that of $\sqrt{nT^{1+\gamma}}$ for Algorithm~\ref{alg:DOGD} and reduces to that of $nT^{\frac{2}{3}}$ for Algorithm~\ref{alg:BGD}. Since $\gamma < 1$ is assumed, Algorithm~\ref{alg:DBGD} is the first no-regret bandit learning algorithm for online nonsubmodular minimization with delayed costs to our knowledge. 
\end{remark}

%% file: sec/exp.tex
\section{Experiments}\label{sec:exp}
We conduct the numerical experiments on structured sparse learning problems and include Algorithm~\ref{alg:OGD}-\ref{alg:DBGD}, which we refer to as OAGD, BAGD, DOAGD, and DBAGD. All the experiments are implemented in Python 3.7 with a 2.6 GHz Intel Core i7 and 16GB of memory. For all our experiments, we set total number of rounds $T=10,000$, dimension $d=10$, number of samples (for round $t$) $n=100$, and sparse parameter $k=2$. For OAGD and DOAGD, we set the default step size $\eta_{\text{o}}=\sqrt{n}/(L\sqrt{T})$ (as described in Theorem~\ref{Theorem:OGD}). For BAGD and DBAGD, we set the default step size $\eta_{\text{b}}=1/(LT^{{2}/{3}})$ (as described in Theorem~\ref{Theorem:BGD}).

Our goal is to estimate the sparse vector $x^\star \in \br^d$ using the structured nonsubmodular model (see Example~\ref{def:SSL}). Following the setup in~\citet{El-2020-Optimal}, we let the function $f^r$ be the regularization in Eq.~\eqref{prob:SSL} such that $f(S) = f^r(S) = \max(S)-\min(S)+1$ for all $S\neq \emptyset$ and $f^r(\emptyset)=0$. We generate true solution $x^\star \in \br^d$ with $k$ consecutive 1's and other $n-k$ elements are zeros.  We define the function $h_t(S)$ for the round $t$ as follows: let $y_t = A_t x^\star + \epsilon_t$ where each row of $A_t \in \br^{n \times d}$ is an i.i.d. Gaussian vector and each entry of $\epsilon_t \in \br^n$ is sampled from a normal distribution with standard deviation equals to 0.01.  Then, we define the square loss $\ell_t(x) = \|A_t x-y_t\|^2_2$ and let $h_t(S) = \ell_t(0) - \min_{\supp(x) \subseteq S} \ell_t(x)$. We consider the constant delays in our experiments, i.e., the delay $\max_{t} d_t \leq d$ for all $t \geq 1$ where $d > 0$ is a constant. 

Figure~\ref{fig:main} summarizes some of experimental results. Indeed, we see from Figure~\ref{fig:main-a} that the bigger delays lead to worse regret for the full-information setting which confirms Theorem~\ref{Theorem:OGD} and~\ref{Theorem:DOGD}. The result in Figure~\ref{fig:main-b} demonstrates the similar phenomenon for the bandit feedback setting which confirms Theorem~\ref{Theorem:BGD} and~\ref{Theorem:DBGD}. Further, Figure~\ref{fig:main-c} demonstrates the effect of bandit feedback and delay simultaneously; indeed,  OAGD and DOAGD perform better than BAGD and DBAGD since the regret will increase if only the bandit feedback is available. We implement all the algorithms with varying step sizes and summarize the results in Figure~\ref{fig:appendix-1} and~\ref{fig:appendix-2}. In the former one, we use step sizes  $\eta_{\text{o}}/2=\sqrt{n}/(2L\sqrt{T})$ for OAGD and DOAGD and $\eta_{\text{b}}/2=1/(2LT^{{2}/{3}})$ for BAGD and DBAGD. In the latter one, we use step sizes  $\eta_{\text{o}}/5=\sqrt{n}/(5L\sqrt{T})$ for OAGD and DOAGD, and $\eta_{\text{b}}/5=1/(5LT^{{2}/{3}})$ for BAGD and DBAGD. Figure~\ref{fig:main}-\ref{fig:appendix-2} demonstrate that our proposed algorithms are not sensitive to the step size choice. 
\begin{figure*}[!t]
\centering
\subfigure[Effect of delay under the full-information setting.\label{fig:main-a}]{\includegraphics[width=.32\textwidth]{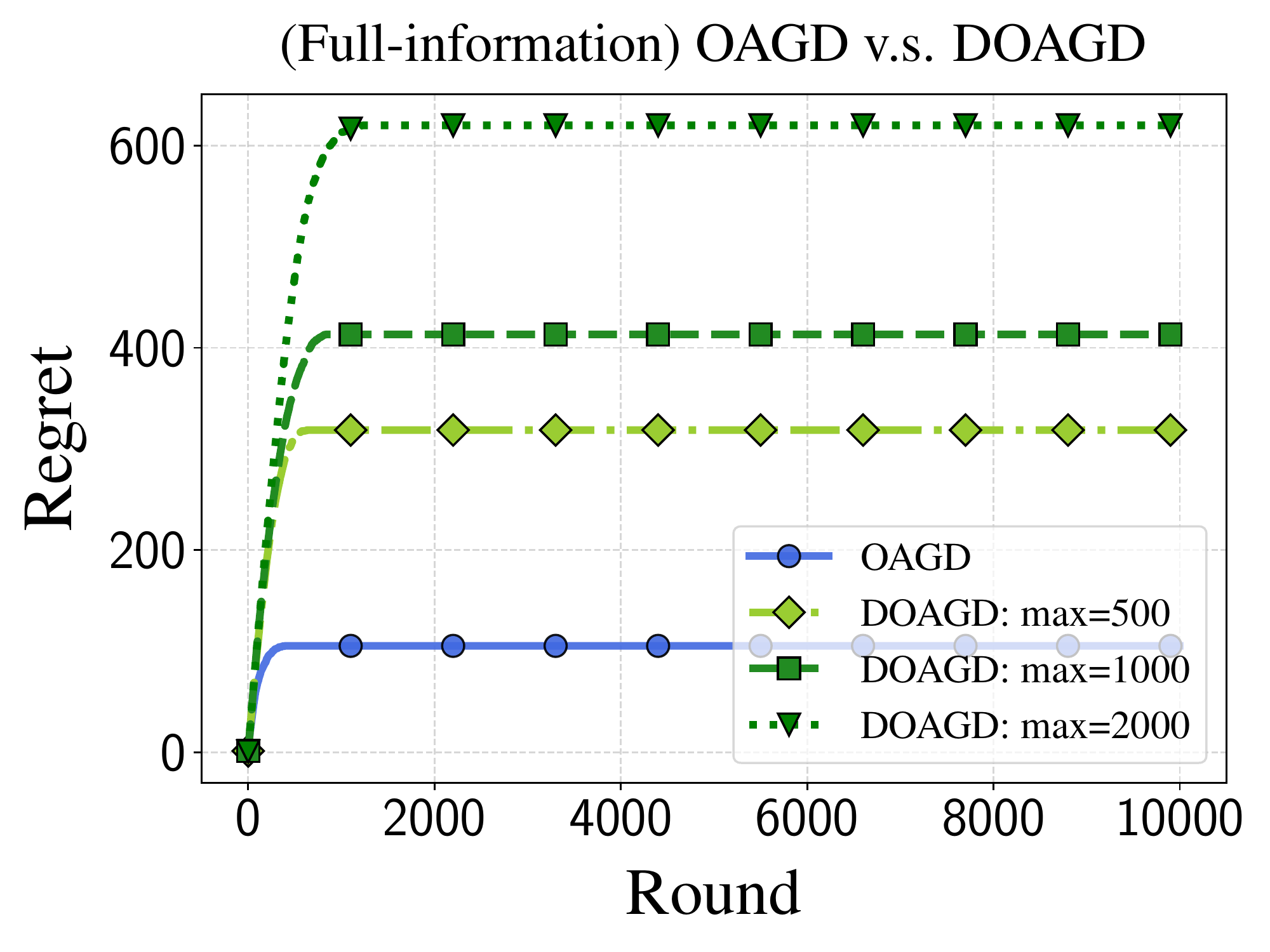}}
\subfigure[Effect of delay under the bandit feedback setting. \label{fig:main-b}]{\includegraphics[width=.32\textwidth]{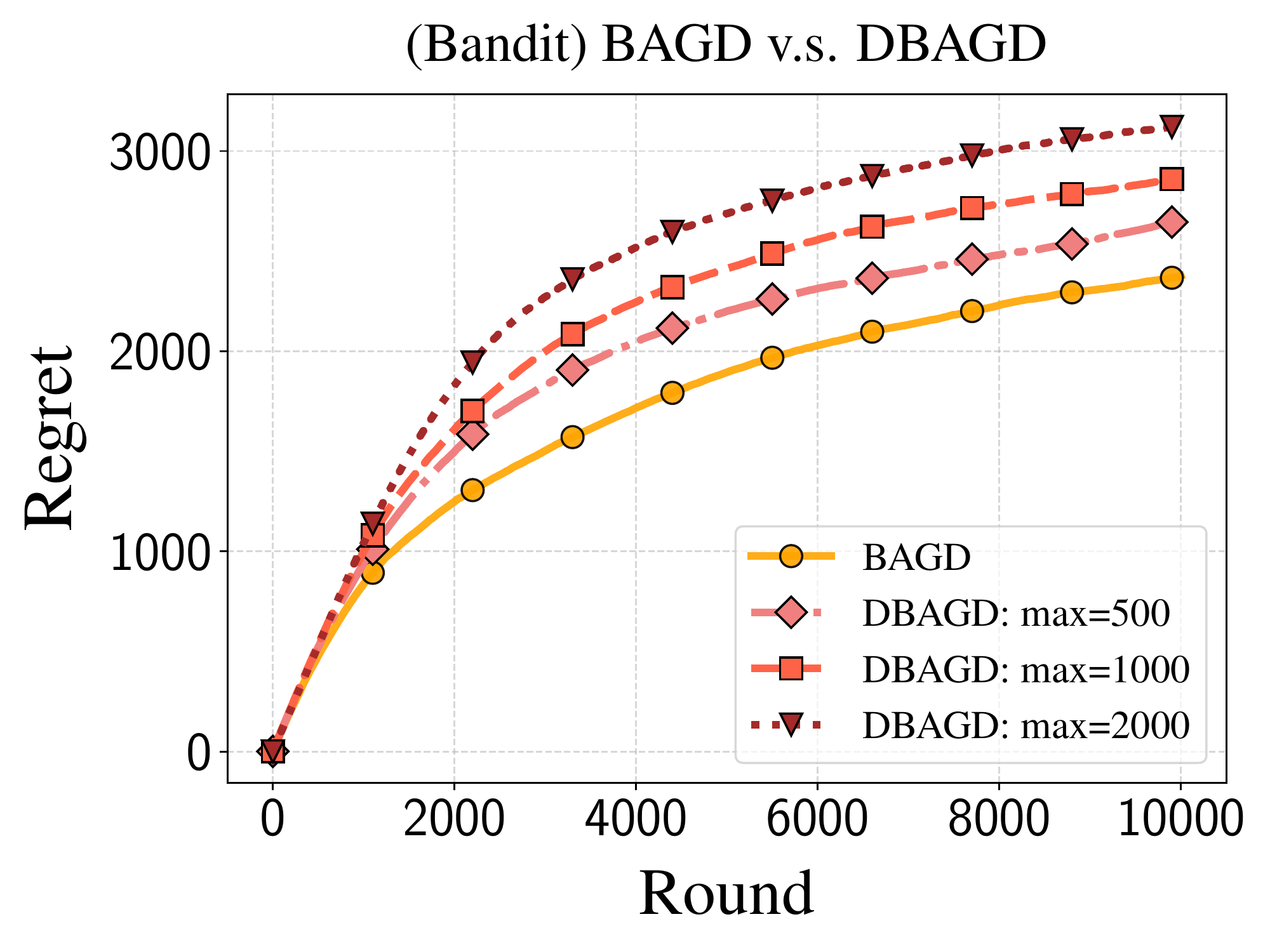}}
\subfigure[Effect of bandit feedback and delay.\label{fig:main-c}]{\includegraphics[width=.32\textwidth]{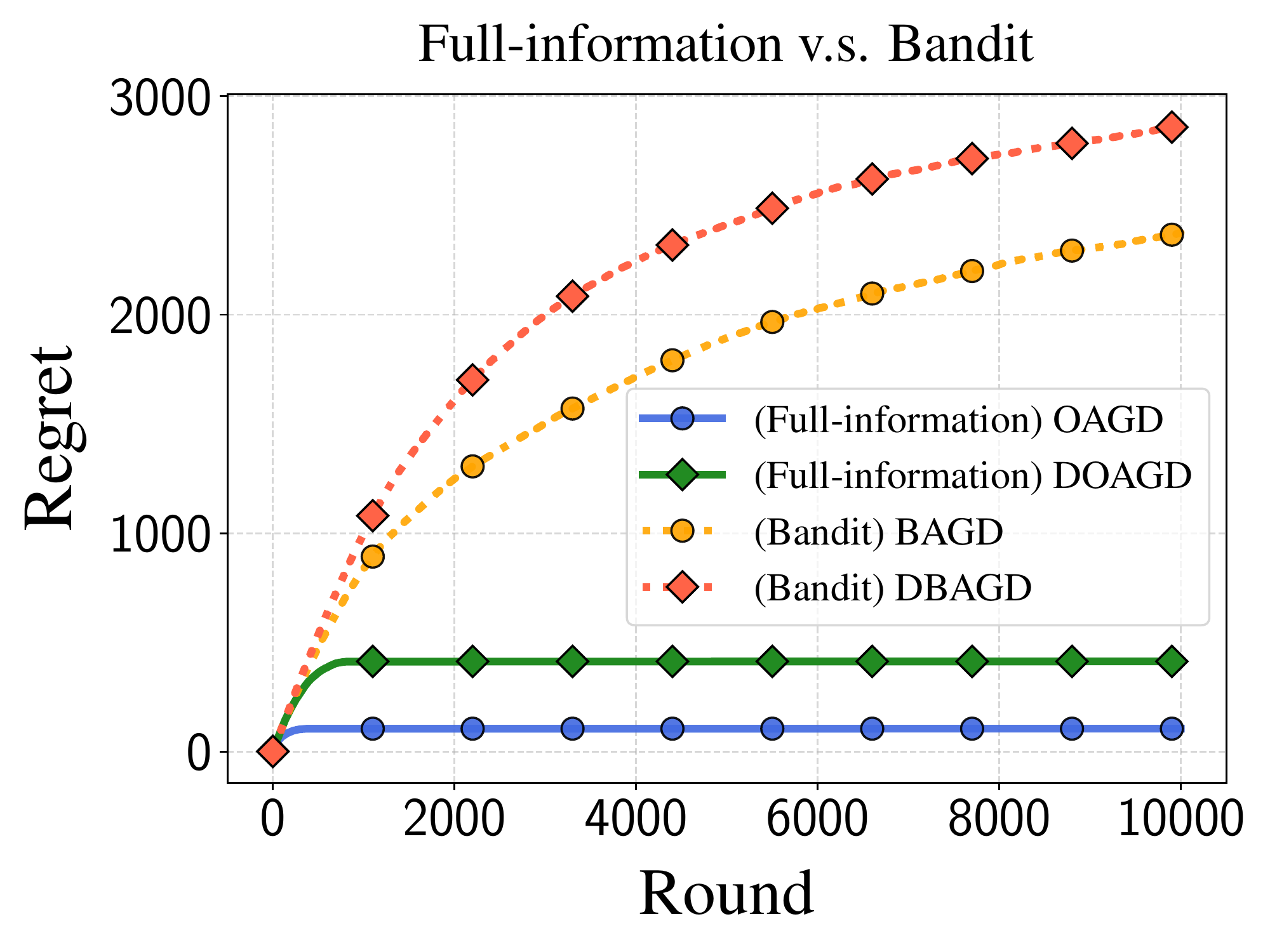}}
\vspace{-0.1in}
\caption{Comparison of our algorithms on sparse learning with delayed costs. In (a) and (b), we examine the effect of delay in the full-information and bandit settings respectively where the maximum delay $d \in \{500, 1000, 2000\}$. In (c), we examine the effect of bandit feedback by comparing the online algorithm with its bandit version where the maximum delay $d = 500$.}\label{fig:main}
\vspace{-0.15in}
\end{figure*}
\begin{figure*}[!t]
\centering
\subfigure[Effect of delay under the full-information setting.\label{fig:appendix-1-a}]{\includegraphics[width=.32\textwidth]{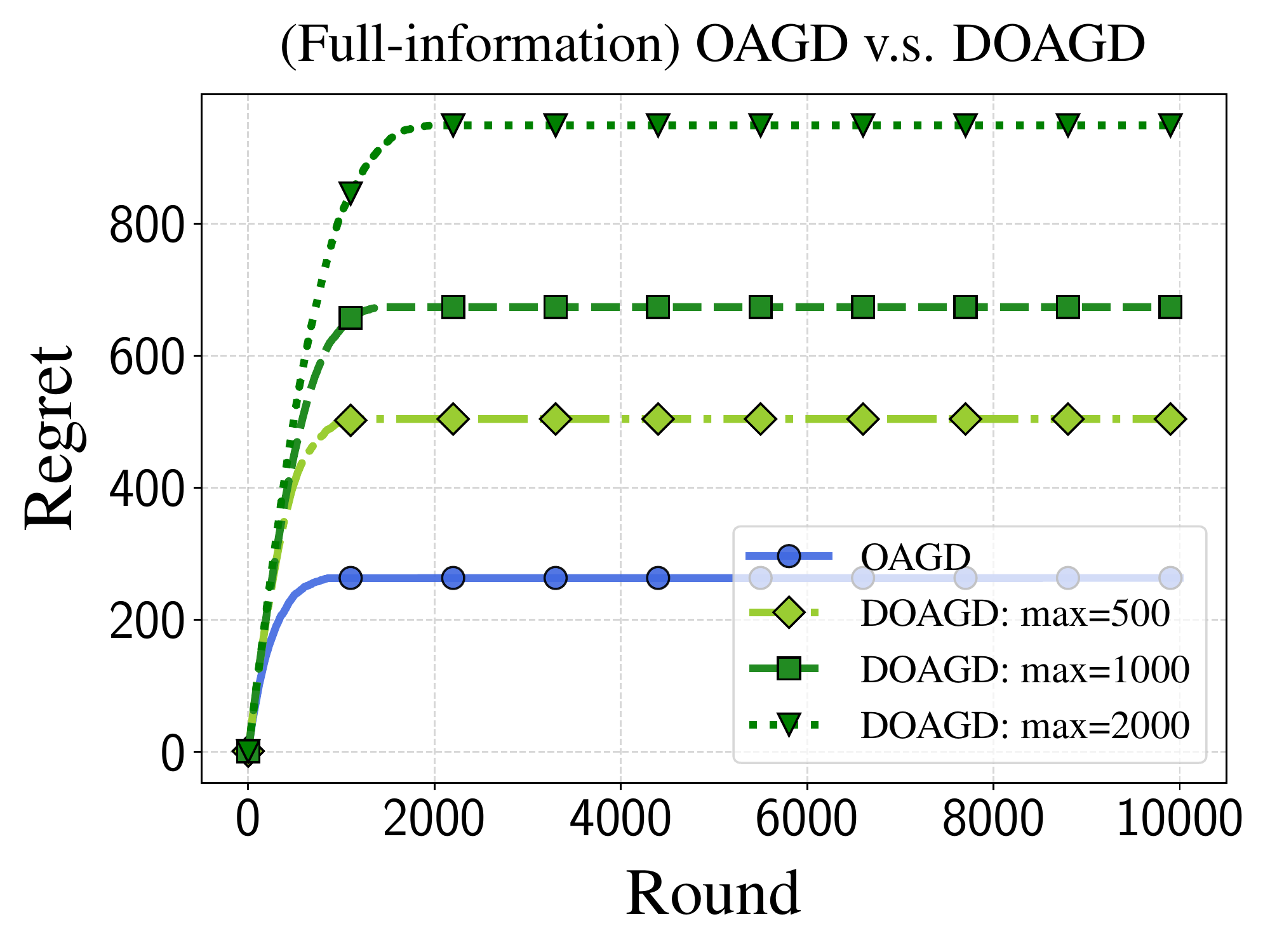}}
\subfigure[Effect of delay under the bandit feedback setting. \label{fig:appendix-1-b}]{\includegraphics[width=.32\textwidth]{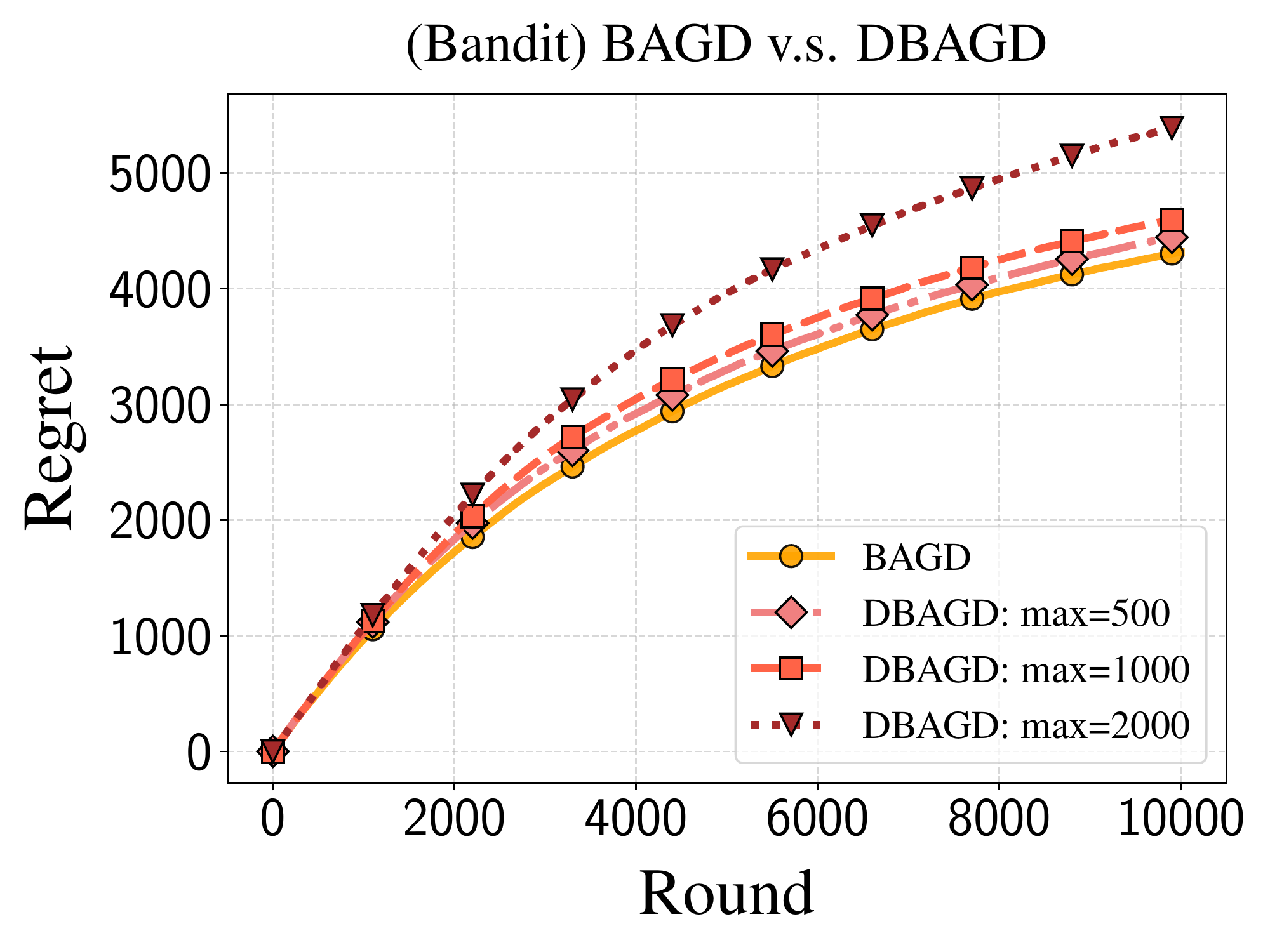}}
\subfigure[Effect of bandit feedback and delay.\label{fig:appendix-1-c}]{\includegraphics[width=.32\textwidth]{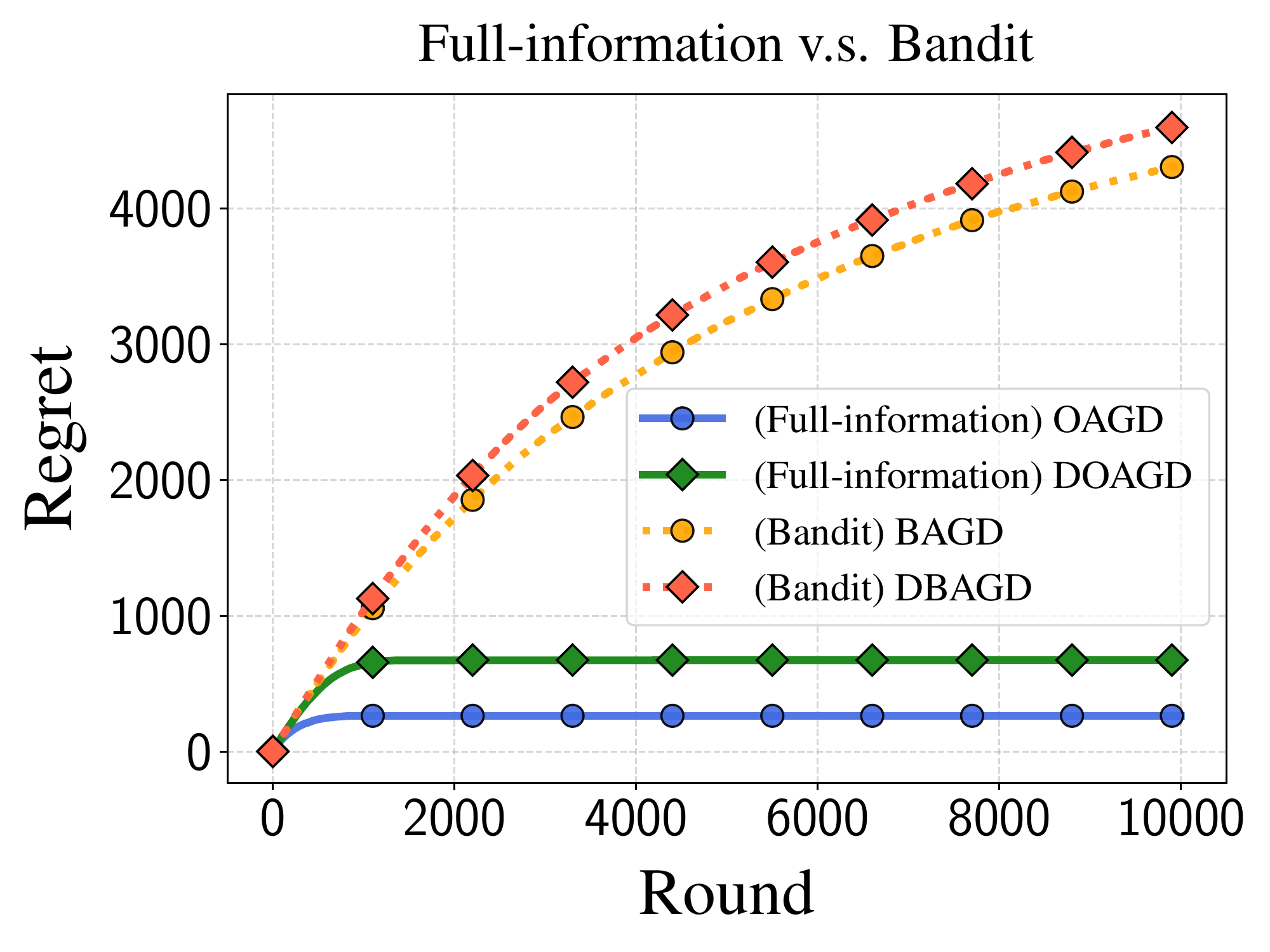}}
\vspace{-0.1in}
\caption{Comparison of our algorithms on sparse learning with delayed costs and step size $\eta=\sqrt{n}/(2L\sqrt{T})$ for OAGD and DOAGD, and $\eta=1/(2LT^{{2}/{3}})$ for BAGD and DBAGD. Note that (a), (b) and (c) follow the same setup as Figure~\ref{fig:main}. }
\label{fig:appendix-1}
\vspace{-0.15in}
\end{figure*}
\begin{figure*}[!t]
\centering
\subfigure[Effect of delay under the full-information setting.\label{fig:appendix-2-a}]{\includegraphics[width=.32\textwidth]{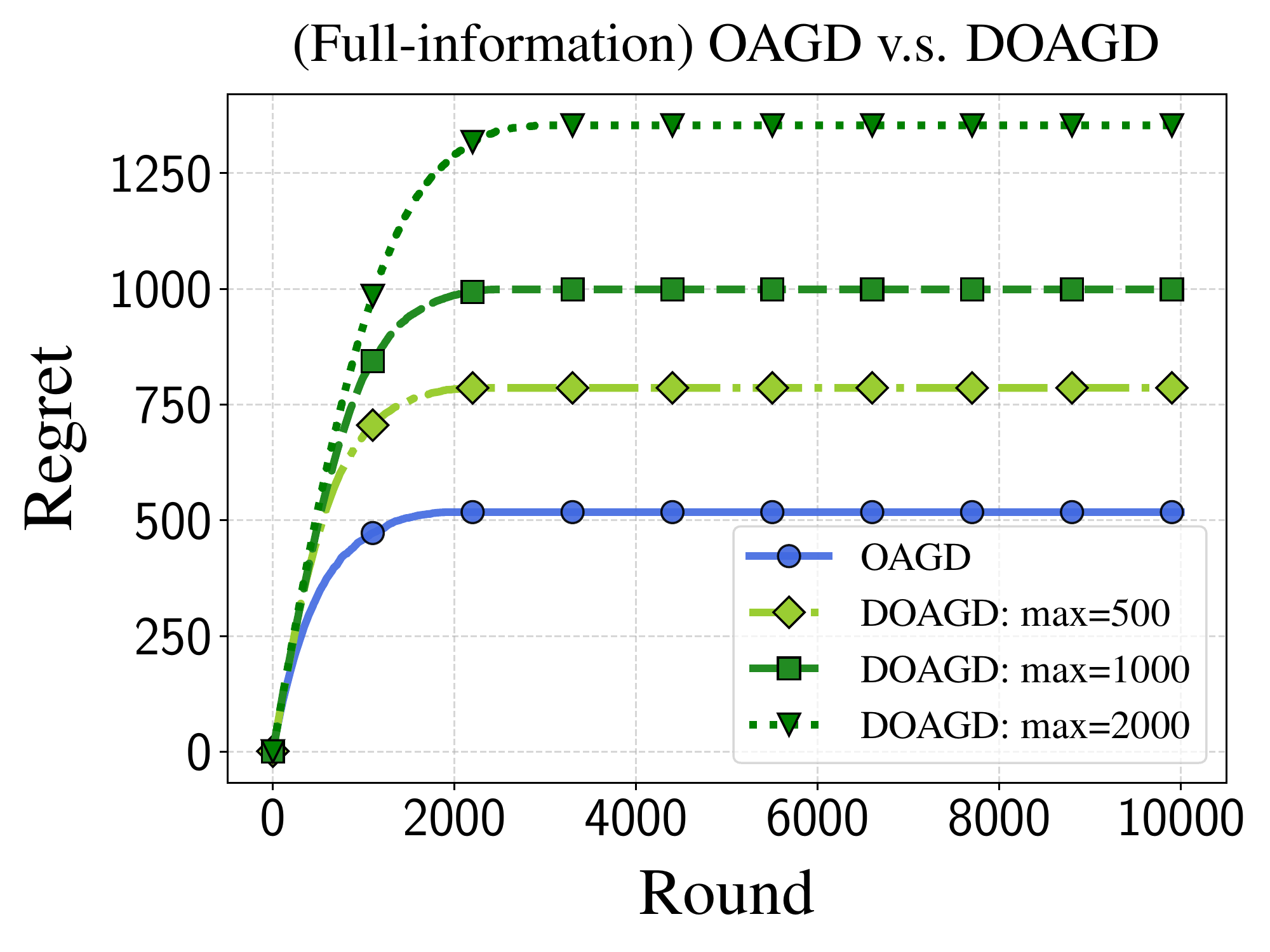}}
\subfigure[Effect of delay under the bandit feedback setting. \label{fig:appendix-2-b}]{\includegraphics[width=.32\textwidth]{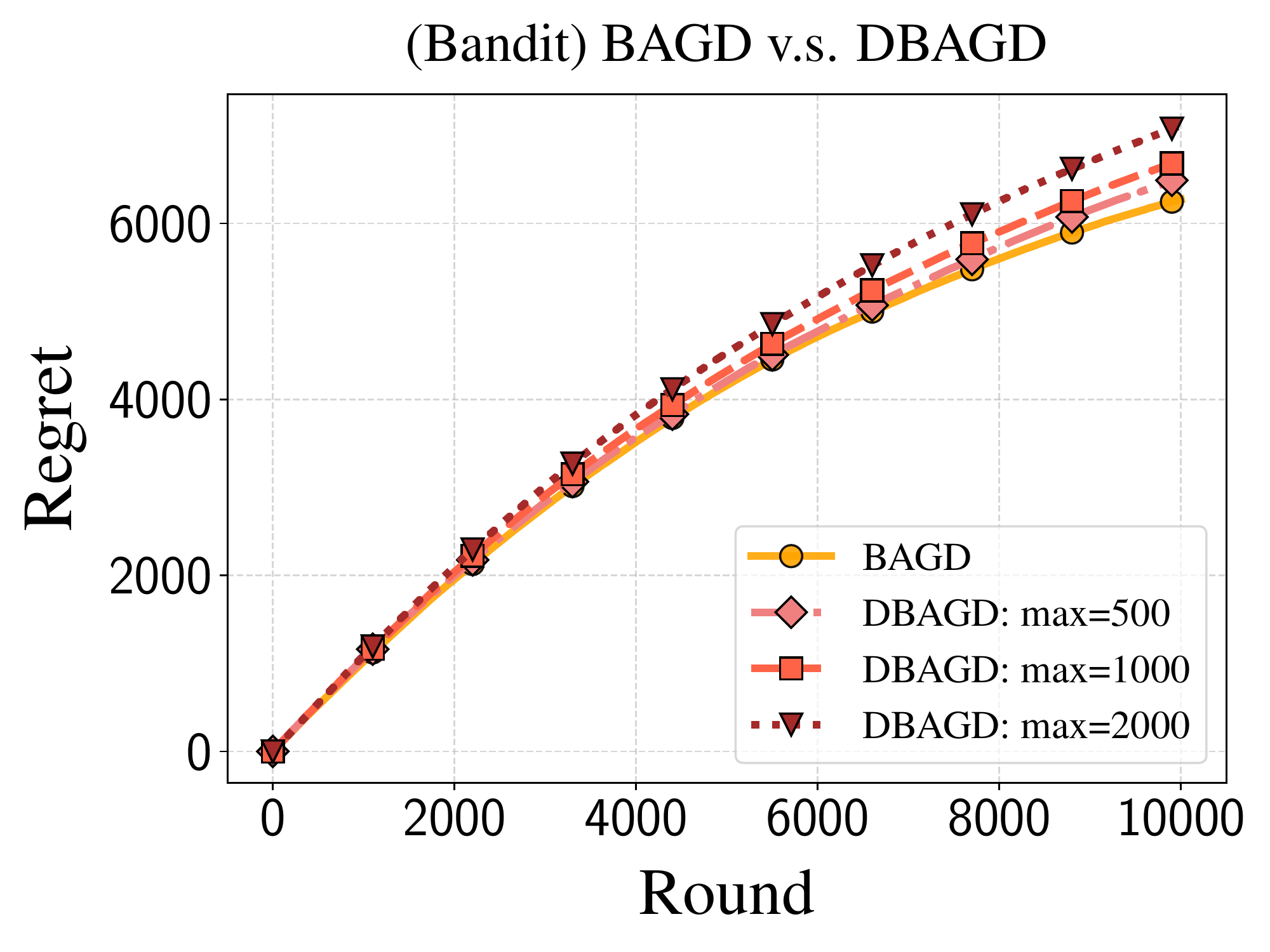}}
\subfigure[Effect of bandit feedback and delay.\label{fig:appendix-2-c}]{\includegraphics[width=.32\textwidth]{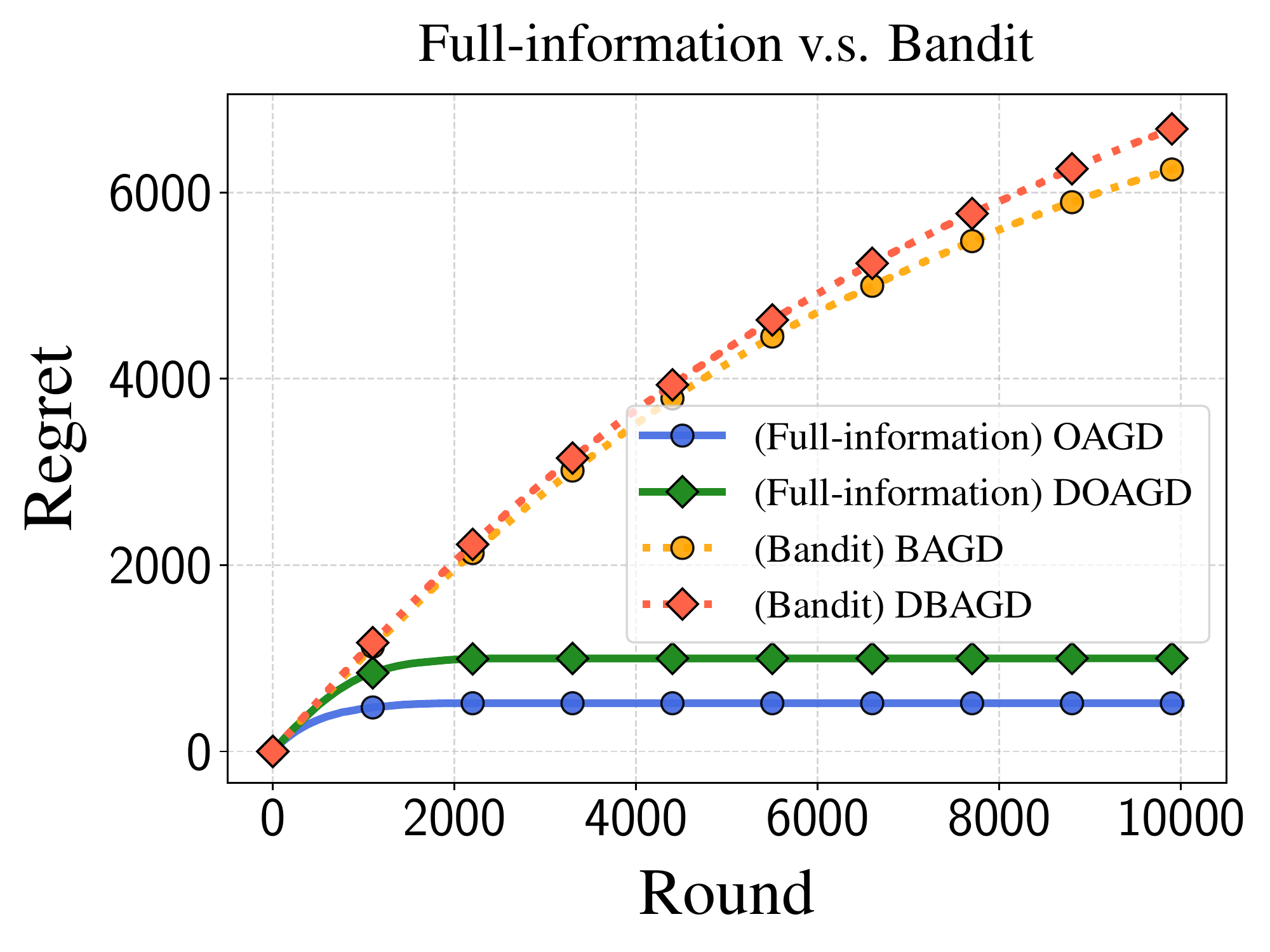}}
\vspace{-0.1in}
\caption{Comparison of our algorithms on sparse learning with delayed costs and step size $\eta=\sqrt{n}/(5L\sqrt{T})$ for OAGD and DOAGD, and $\eta=1/(5LT^{{2}/{3}})$ for BAGD and DBAGD. Note that (a), (b) and (c) follow the same setup as Figure~\ref{fig:main}. }
\label{fig:appendix-2}
\vspace{-0.15in}
\end{figure*}

%% file: sec/conclu.tex
\section{Concluding Remarks}\label{sec:conclu}
This paper studied online nonsubmodular minimization with special structure through the lens of $(\alpha, \beta)$-regret and the extension of generic convex relaxation model. We proved that online approximate gradient descent algorithm and its bandit variant adapted for the convex relaxation model could achieve the bounds of $O(\sqrt{nT})$ and $O(nT^{\frac{2}{3}})$ in terms of $(\alpha, \beta)$-regret respectively. We also investigated the delayed variants of two algorithms and proved new regret bounds where the delays can even be unbounded. More specifically, if delays satisfy $d_t=o(t^{\gamma})$ with $\gamma < 1$, we showed that our proposed algorithms achieve the regret bound of $O(\sqrt{nT^{1+\gamma}})$ and $O(nT^{\frac{2+\gamma}{3}})$ for full-information setting and bandit setting respectively. Simulation studies validate our theoretical findings in practice.

%% file: sec/app_structure.tex
\section{Proof of Proposition~\ref{Prop:structure}}\label{app:structure}
We have
\begin{equation}\label{inequality:structure-main}
f_C(x) = \max_{g \in \br^d, \rho \in \br} \left\{g^\top x + \rho: g(A) + \rho \leq f(A), \forall A \subseteq [n]\right\}. 
\end{equation}
First, we prove Eq.~\eqref{Prop:structure-first} using the definition of $f_L$ and Eq.~\eqref{inequality:structure-main}. Indeed, we have $f_L(x) = g^\top x$ where we let $x \in [0, 1]^n$ with $x_{\pi(1)} \geq \ldots \geq x_{\pi(n)}$ and $g_{\pi(i)} = f(\pi(i) \mid A_{i-1})$ for all $i \in [n]$. Then, it suffices to show that $g^\top x \geq \tilde{g}^\top x + \tilde{\rho}$ in which $\tilde{g}(A) + \tilde{\rho} \leq f(A)$ for all $A \subseteq [n]$. We have
\begin{eqnarray*}
\lefteqn{g^\top x - (\tilde{g}^\top x + \tilde{\rho}) = \sum_{i=1}^n x_{\pi(i)}\left(f(\pi(i) \mid A_{i-1}) - \tilde{g}_{\pi(i)}\right) - \tilde{\rho}} \\
& = & \sum_{i=1}^{n-1} \left(x_{\pi(i)} - x_{\pi(i+1)}\right)\left(f(A_i) - \tilde{g}(A_i)\right) + x_{\pi(n)}\left(f([n]) - \tilde{g}([n])\right) - \tilde{\rho}. 
\end{eqnarray*}
Since $\tilde{g}(A) + \tilde{\rho} \leq f(A)$ for all $A \subseteq [n]$, we have
\begin{equation*}
f([n]) - \tilde{g}([n]) \geq \tilde{\rho},  \quad f(A_i) - \tilde{g}(A_i) \geq \tilde{\rho},  \quad \textnormal{for all } i \in [n]. 
\end{equation*}
Putting these pieces together with $x_{\pi(1)} \geq \ldots \geq x_{\pi(n)}$ yields that 
\begin{equation*}
g^\top x - (\tilde{g}^\top x + \tilde{\rho}) \geq \sum_{i=1}^{n-1} \left(x_{\pi(i)} - x_{\pi(i+1)}\right)\tilde{\rho} + x_{\pi(n)}\tilde{\rho} - \tilde{\rho} = (x_{\pi(1)} - 1)\tilde{\rho}. 
\end{equation*}
Since $x \in [0,1]^n$, we have $x_{\pi(1)} \leq 1$.  Since $\tilde{g}(A) + \tilde{\rho} \leq f(A)$ for all $A \subseteq [n]$ and $f(\emptyset) = 0$, we derive by letting $A = \emptyset$ that $\tilde{\rho} \leq f(\emptyset) - \tilde{g}(\emptyset) \leq 0$.  This implies the desired result. 

Further, we prove Eq.~\eqref{Prop:structure-second} using the definition of weak DR-submodularity.  Indeed, we have $g(A) = \sum_{i \in A} g_i$.  Since $g_{\pi(i)} = f(\pi(i) \mid A_{i-1})$ for all $i \in [n]$, we have
\begin{equation*}
g(A) = \sum_{\pi(i) \in A} f(\pi(i) \mid A_{i-1}) = \sum_{\pi(i) \in A} \left(\bar{f}(\pi(i) \mid A_{i-1}) - \ushort{f}(\pi(i) \mid A_{i-1})\right). 
\end{equation*}
Since $\bar{f}$ is $\alpha$-weakly DR-submodular, $\ushort{f}$ is $\beta$-weakly DR-supermodular and $A \cap A_{i-1} \subseteq A_{i-1}$, we have
\begin{equation}\label{inequality:structure-DR}
\bar{f}(\pi(i) \mid A \cap A_{i-1}) \geq \alpha \bar{f}(\pi(i) \mid A_{i-1}), \quad \ushort{f}(\pi(i) \mid A_{i-1}) \geq \beta\ushort{f}(\pi(i) \mid A \cap A_{i-1}). 
\end{equation}
Putting these pieces together yields that 
\begin{equation*}
g(A) \leq \sum_{\pi(i) \in A} \left(\tfrac{1}{\alpha}\bar{f}(\pi(i) \mid A \cap A_{i-1}) - \beta\ushort{f}(\pi(i) \mid A \cap A_{i-1})\right). 
\end{equation*}
Then, we have
\begin{eqnarray*}
g(A) & \leq & \tfrac{1}{\alpha} \left(\sum_{i=1}^n \left(\bar{f}(A \cap A_i) - \bar{f}(A \cap A_{i-1})\right)\right) - \beta \left(\sum_{i=1}^n \left(\ushort{f}(A \cap A_i) - \ushort{f}(A \cap A_{i-1})\right)\right) \\ 
& = & \tfrac{1}{\alpha}\bar{f}(A) - \beta\ushort{f}(A), \quad \textnormal{for all } A \subseteq [n]. 
\end{eqnarray*}
This implies the desired result. 

Finally, we prove Eq.~\eqref{Prop:structure-third} using Eq.~\eqref{inequality:structure-main}.  Indeed, we have $g = \bar{g} - \ushort{g}$ where $\bar{g}_{\pi(i)} = \bar{f}(\pi(i) \mid A_{i-1})$ and $\ushort{g}_{\pi(i)} = \ushort{f}(\pi(i) \mid A_{i-1})$ for all $i \in [n]$. For any $A \subseteq [n]$, we obtain by using Eq.~\eqref{inequality:structure-DR} that 
\begin{align*}
\bar{g}(A) & \leq \tfrac{1}{\alpha}\left(\sum_{\pi(i) \in A} \bar{f}(\pi(i) \mid A \cap A_{i-1}) \right) = \tfrac{1}{\alpha}\left(\sum_{i=1}^n \left(\bar{f}(A \cap A_i) - \bar{f}(A \cap A_{i-1})\right)\right) = \tfrac{1}{\alpha}\bar{f}(A), \\
-\ushort{g}(A) & \leq -\beta\left(\sum_{\pi(i) \in A} \ushort{f}(\pi(i) \mid A \cap A_{i-1}) \right) = -\beta\left(\sum_{i=1}^n \left(\ushort{f}(A \cap A_i) - \ushort{f}(A \cap A_{i-1})\right)\right) = -\beta\ushort{f}(A). 
\end{align*}
Equivalently, we have $\alpha \bar{g}(A) + 0 \leq \bar{f}(A)$ and $\tfrac{1}{\beta}(-\ushort{g}(A)) + 0 \leq -\ushort{f}(A)$ for any $A \subseteq [n]$. Using Eq.~\eqref{inequality:structure-main}, we have
\begin{eqnarray*}
\alpha \bar{g}^\top z + 0 \leq \bar{f}_C(z), \quad \tfrac{1}{\beta}(-\ushort{g})^\top z + 0 \leq (-\ushort{f})_C(z), \quad \textnormal{for all } z \in [0, 1]^n.  
\end{eqnarray*}
Since $g = \bar{g} - \ushort{g}$ and $\alpha, \beta > 0$, we have $g^\top z \leq \tfrac{1}{\alpha}\bar{f}_C(z) + \beta(-\ushort{f})_C(z)$. This implies the desired result. 

%% file: sec/app_OGD.tex
\section{Regret Analysis for Algorithm~\ref{alg:OGD}}\label{app:OGD}
In this section, we present several technical lemmas for analyzing the regret minimization property of Algorithm~\ref{alg:OGD}.  We also give the missing proof of Theorem~\ref{Theorem:OGD}. 

\subsection{Technical lemmas}
We provide two technical lemmas for Algorithm~\ref{alg:OGD}.  The first lemma gives a bound on the vector $g^t$ and the difference between $x^t$ and any fixed $x \in [0, 1]^n$. 
\begin{lemma}\label{Lemma:OGD-first}
Suppose that the iterates $\{x^t\}_{t \geq 1}$ and the vectors $\{g^t\}_{t \geq 1}$ be generated by Algorithm~\ref{alg:OGD} and $x \in [0, 1]^n$ and let $f_t = \bar{f}_t - \ushort{f}_t$ satisfy that $\bar{f}_t([n]) + \ushort{f}_t([n]) \leq L$ for all $t \geq 1$ and both $\bar{f}_t$ and $\ushort{f}_t$ are nondecreasing. Then, we have $\|x^t - x\| \leq \sqrt{n}$ and $\|g^t\| \leq L$ for all $t \geq 1$. 
\end{lemma}
\begin{proof}
Since $x^t \in [0, 1]^n$ and $x \in [0, 1]^n$ is fixed, we have $\|x^t - x\| \leq \sqrt{\sum_{i=1}^n 1} = \sqrt{n}$ for all $t \geq 1$. By the definition of $g^t$, we have $g_{\pi(i)}^t = f_t(A_i^t) - f_t(A_{i-1}^t)$ for all $i \in [n]$ where $A_i^t = \{\pi(1), \ldots, \pi(i)\}$ for all $i \in [n]$.  Then, we have
\begin{equation*}
\|g^t\| \leq \sum_{i=1}^n |f_t(A_i^t) - f_t(A_{i-1}^t)| \leq \sum_{i=1}^n |\bar{f}_t(A_i^t) - \bar{f}_t(A_{i-1}^t)| + \sum_{i=1}^n |\ushort{f}_t(A_i^t) - \ushort{f}_t(A_{i-1}^t)|. 
\end{equation*}
Since $\bar{f}_t$ and $\ushort{f}_t$ are both normalized and non-decreasing, we have
\begin{equation*}
\sum_{i=1}^n |\bar{f}_t(A_i^t) - \bar{f}_t(A_{i-1}^t)| + \sum_{i=1}^n |\ushort{f}_t(A_i^t) - \ushort{f}_t(A_{i-1}^t)| = \bar{f}_t([n]) + \ushort{f}_t([n]) \leq L. 
\end{equation*}
Putting these pieces together yields that $\|g^t\| \leq L$ for all $t = 1, 2, \ldots, T$. 
\end{proof}
The second lemma gives a key inequality for analyzing Algorithm~\ref{alg:OGD}. 
\begin{lemma}\label{Lemma:OGD-second}
Suppose that the iterates $\{x^t\}_{t \geq 1}$ are generated by Algorithm~\ref{alg:OGD} and $x \in [0, 1]^n$ and let $f_t = \bar{f}_t - \ushort{f}_t$ satisfy that $\bar{f}_t([n]) + \ushort{f}_t([n]) \leq L$ for all $t \geq 1$. Then, we have
\begin{equation*}
\sum_{t=1}^T \EE[(f_t)_L(x^t)] \leq \left(\sum_{t=1}^T \tfrac{1}{\alpha}(\bar{f}_t)_C(x) + \beta(-\ushort{f}_t)_C(x)\right) + \tfrac{n}{2\eta} + \tfrac{\eta L^2 T}{2}. 
\end{equation*}
\end{lemma}
\begin{proof}
Since $x^{t+1} = P_{[0, 1]^n}(x^t - \eta g^t)$, we have
\begin{equation*}
(x - x^{t+1})^\top(x^{t+1} - x^t + \eta g^t) \geq 0, \quad \textnormal{for all } x \in [0, 1]^n. 
\end{equation*}
Rearranging the above inequality and using the fact that $\eta > 0$, we have
\begin{equation}\label{inequality:OGD-first}
(x^{t+1} - x)^\top g^t \leq \frac{1}{\eta}(x - x^{t+1})^\top(x^{t+1} - x^t) = \tfrac{1}{2\eta}\left(\|x - x^t\|^2 - \|x - x^{t+1}\|^2 - \|x^t - x^{t+1}\|^2\right). 
\end{equation}
Using Young's inequality, we have
\begin{equation}\label{inequality:OGD-second}
(x^t - x^{t+1})^\top g^t \leq \tfrac{1}{2\eta}\|x^t - x^{t+1}\|^2 + \tfrac{\eta}{2}\|g^t\|^2. 
\end{equation}
Combining Eq.~\eqref{inequality:OGD-first} and Eq.~\eqref{inequality:OGD-second} yields that 
\begin{equation*}
(x^t - x)^\top g^t \leq \tfrac{1}{2\eta}\left(\|x - x^t\|^2 - \|x - x^{t+1}\|^2\right) + \tfrac{\eta}{2}\|g^t\|^2. 
\end{equation*}
Since $f_t = \bar{f}_t - \ushort{f}_t$ where $\bar{f}_t$ and $\ushort{f}_t$ are both non-decreasing, $\bar{f}_t$ is $\alpha$-weakly DR-submodular and $\ushort{f}_t$ is $\beta$-weakly DR-supermodular, Proposition~\ref{Prop:structure} implies that 
\begin{equation*}
(x^t - x)^\top g^t \geq (f_t)_L(x^t) - \left(\tfrac{1}{\alpha}(\bar{f}_t)_C(x) + \beta(-\ushort{f}_t)_C(x)\right). 
\end{equation*}
By Lemma~\ref{Lemma:OGD-first}, we have $\|g^t\| \leq L$ for all $t = 1, 2, \ldots, T$.  Then, we have
\begin{equation*}
(f_t)_L(x^t) \leq \tfrac{1}{\alpha}(\bar{f}_t)_C(x) + \beta(-\ushort{f}_t)_C(x) + \tfrac{1}{2\eta}\left(\|x - x^t\|^2 - \|x - x^{t+1}\|^2\right) + \tfrac{\eta L^2}{2}. 
\end{equation*}
Summing up the above inequality over $t = 1, 2, \ldots, T$ and using $\|x^1 - x\| \leq \sqrt{n}$ (cf. Lemma~\ref{Lemma:OGD-first}), we have 
\begin{equation*}
\sum_{t=1}^T (f_t)_L(x^t) \leq \left(\sum_{t=1}^T \tfrac{1}{\alpha}(\bar{f}_t)_C(x) + \beta(-\ushort{f}_t)_C(x)\right) + \tfrac{n}{2\eta} + \tfrac{\eta L^2 T}{2}. 
\end{equation*}
Taking the expectation of both sides yields the desired inequality. 
\end{proof}

\subsection{Proof of Theorem~\ref{Theorem:OGD}}
By the definition of the Lov\'{a}sz extension, we have
\begin{equation*}
(f_t)_L(x^t) = \sum_{i = 1}^{n-1} (x_{\pi(i)}^t - x_{\pi(i+1)}^t)f_t(A_i^t) + (1 - x_{\pi(1)}^t)f_t(A_0^t) + x_{\pi(n)}^t f_t(A_n^t) = \sum_{i=0}^n \lambda_i^t f_t(A_i^t). 
\end{equation*}
By the update formula, we have $\EE[f_t(S^t) \mid x^t] = (f_t)_L(x^t)$ which implies that $\EE[f_t(S^t)] = \EE[(f_t)_L(x^t)]$. By the definition of the convex closure,  we obtain that the convex closure of a set function $f$ agrees with $f$ on all the integer points~\citep[Page~4, Proposition~3.3]{Dughmi-2009-Submodular}.  Letting $S_\star^T = \argmin_{S \subseteq [n]} \sum_{t=1}^T f_t(S)$, we have $S_\star^T$ is an integer point and 
\begin{equation*}
(\bar{f}_t)_C(\chi_{S_\star^T}) = \bar{f}_t(S_\star^T), \quad (-\ushort{f}_t)_C(\chi_{S_\star^T}) = - \beta\ushort{f}_t(S_\star^T), 
\end{equation*}
which implies that 
\begin{equation*}
\tfrac{1}{\alpha}(\bar{f}_t)_C(\chi_{S_\star^T}) + \beta(-\ushort{f}_t)_C(\chi_{S_\star^T}) = \tfrac{1}{\alpha}\bar{f}_t(S_\star^T) - \beta\ushort{f}_t(S_\star^T). 
\end{equation*}
Putting these pieces together and letting $x = \chi_{S_\star^T}$ in the inequality of Lemma~\ref{Lemma:OGD-second} yields that 
\begin{equation*}
\sum_{t=1}^T \EE[f_t(S^t)] \leq \left(\sum_{t=1}^T \tfrac{1}{\alpha} \bar{f}_t(S_\star^T) - \beta\ushort{f}_t(S_\star^T)\right) + \tfrac{n}{2\eta} + \tfrac{\eta L^2 T}{2}. 
\end{equation*}
Plugging the choice of $\eta = \frac{\sqrt{n}}{L\sqrt{T}}$ into the above inequality yields that $\EE[R_{\alpha, \beta}(T)] = O(\sqrt{nT})$ as desired. 

We proceed to derive a high probability bound using the concentration inequality.  In particular, we review the Hoeffding inequality~\citep{Hoeffding-1963-Probability} and refer to~\citet[Appendix A]{Cesa-2006-Prediction} for a proof. The following proposition is a restatement of~\citet[Corollary~A.1]{Cesa-2006-Prediction}. 
\begin{proposition}\label{Prop:Hoeffding}
Let $X_1, \ldots, X_n$ be independent real-valued random variables such that for each $i = 1, \ldots, n$, there exist some $a_i \leq b_i$ such that $\PP(a_i \leq X_i \leq b_i) = 1$. Then for every $\epsilon > 0$, we have
\begin{eqnarray*}
\PP\left(\sum_{i=1}^n X_i - \EE\left[\sum_{i=1}^n X_i\right] > + \epsilon\right) & \leq & \exp\left(-\frac{2\epsilon^2}{\sum_{i=1}^n (b_i - a_i)^2}\right), \\
\PP\left(\sum_{i=1}^n X_i - \EE\left[\sum_{i=1}^n X_i\right] < -\epsilon\right) & \leq & \exp\left(-\frac{2\epsilon^2}{\sum_{i=1}^n (b_i - a_i)^2}\right). 
\end{eqnarray*}
\end{proposition}
Since the sequence of points $x^1, x^2, \ldots, x^T$ is obtained by several deterministic gradient descent steps, we have this sequence is purely deterministic. Since each of $S^t$ is obtained by independent randomized rounding on the point $x^t$, we have the sequence of random variables $X_t = f_t(S^t)$ is independent.  By definition of $f_t$, we have
\begin{equation*}
|X_t| = |\bar{f}_t(S^t) - \ushort{f}_t(S^t)| \leq \bar{f}_t(S^t) + \ushort{f}_t(S^t).  
\end{equation*}
Since $\bar{f}_t$ and $\ushort{f}_t$ are non-decreasing and $\bar{f}_t([n]) + \ushort{f}_t([n]) \leq L$ for all $t \geq 1$, we have $\PP(-L \leq X_i \leq L) = 1$ for all $t \geq 1$. Then, by Proposition~\ref{Prop:Hoeffding}, we have
\begin{equation*}
\PP\left(\sum_{i=1}^n f_t(S^t) - \EE\left[\sum_{i=1}^n f_t(S^t)\right] > \epsilon\right) \leq \exp\left(-\frac{\epsilon^2}{2nL^2}\right). 
\end{equation*}
Equivalently, we have $\sum_{i=1}^n f_t(S^t) - \EE[\sum_{i=1}^n f_t(S^t)] \leq L\sqrt{2T\log(1/\delta)}$ with probability at least $1 - \delta$.  This together with $\EE[R_{\alpha, \beta}(T)] = O(\sqrt{nT})$ yields that $R_{\alpha, \beta}(T) = O(\sqrt{nT} + \sqrt{T\log(1/\delta)})$ with probability at least $1 - \delta$ as desired. 

%% file: sec/app_BGD.tex
\section{Regret Analysis for Algorithm~\ref{alg:BGD}}\label{app:BGD}
In this section, we present several technical lemmas for analyzing the regret minimization property of Algorithm~\ref{alg:BGD}.  We also give the missing proofs of Theorem~\ref{Theorem:BGD}. 

\subsection{Technical lemmas}
We provide several technical lemmas for Algorithm~\ref{alg:BGD}.  The first lemma is analogous to Lemma~\ref{Lemma:OGD-first} and gives a bound on the vector $\hat{g}^t$ (in expectation) and the difference between $x^t$ and any fixed $x \in [0, 1]^n$. 
\begin{lemma}\label{Lemma:BGD-first}
Suppose that the iterates $\{x^t\}_{t \geq 1}$ and the vectors $\{\hat{g}^t\}_{t \geq 1}$ be generated by Algorithm~\ref{alg:BGD} and $x \in [0, 1]^n$ and let $f_t = \bar{f}_t - \ushort{f}_t$ satisfy that $\bar{f}_t([n]) + \ushort{f}_t([n]) \leq L$ for all $t \geq 1$ and both $\bar{f}_t$ and $\ushort{f}_t$ are nondecreasing. Then, we have $\|x^t - x\| \leq \sqrt{n}$ for all $t \geq 1$ and 
\begin{equation*}
\EE[\hat{g}^t \mid x^t] = g^t, \quad \EE[\|\hat{g}^t\|^2 \mid x^t] \leq \tfrac{8n^2L^2}{\mu}, \quad \|\hat{g}^t\|^2 \leq \tfrac{2(n+1)^2L^2}{\mu^2}. 
\end{equation*}
where we have $g_{\pi(i)}^t = f_t(A_i^t) - f_t(A_{i-1}^t)$ for all $i \in [n]$.  
\end{lemma}
\begin{proof}
Using the same argument as in Lemma~\ref{Lemma:OGD-first}, we have $\|x^t - x\| \leq \sqrt{n}$ for all $t \geq 1$. By the definition of $\hat{g}^t$, we have 
\begin{equation*}
\hat{g}^t_{\pi(i)} = \left(\tfrac{\one(S^t = A^t_i)}{(1-\mu)\lambda_i^t + \frac{\mu}{n+1}} - \tfrac{\one(S^t = A^t_{i-1})}{(1-\mu)\lambda_{i-1}^t + \frac{\mu}{n+1}}\right)f_t(S^t), \quad \textnormal{for all } i \in [n]. 
\end{equation*}
This together with the sampling scheme for $S^t$ implies that 
\begin{equation*}
\EE[\hat{g}^t_{\pi(i)} \mid x^t] = f_t(A_i^t) - f_t(A_{i-1}^t), \quad \textnormal{for all } i \in [n], 
\end{equation*}
Since $g_{\pi(i)}^t = f_t(A_i^t) - f_t(A_{i-1}^t)$ for all $i \in [n]$, we have $\EE[\hat{g}^t \mid x^t] = g^t$. Since $f_t = \bar{f}_t - \ushort{f}_t$ satisfy that $\bar{f}_t([n]) + \ushort{f}_t([n]) \leq L$ for all $t \geq 1$ and $\bar{f}_t$ and $\ushort{f}_t$ are both normalized and non-decreasing, we have
\begin{equation*}
\EE[\|\hat{g}^t\|^2 \mid x^t] \leq \sum_{i=0}^n \tfrac{2(f_t(A_i^t))^2}{(1-\mu)\lambda_i^t + \frac{\mu}{n+1}} \leq \tfrac{2(n+1)^2L^2}{\mu} \leq \tfrac{8n^2L^2}{\mu}. 
\end{equation*}
Further,  let $S^t = A_{i_t}$ in the round $t$, we can apply the same argument and obtain that 
\begin{equation*}
\|\hat{g}^t\|^2 \leq 2\left(\tfrac{f_t(A_{i_t}^t)}{(1-\mu)\lambda_{i_t}^t + \frac{\mu}{n+1}}\right)^2 \leq \tfrac{2(n+1)^2L^2}{\mu^2}. 
\end{equation*}
This completes the proof. 
\end{proof}
The second lemma is analogous to Lemma~\ref{Lemma:OGD-second} and gives a key inequality for analyzing Algorithm~\ref{alg:BGD}. 
\begin{lemma}\label{Lemma:BGD-second}
Suppose that the iterates $\{x^t\}_{t \geq 1}$ are generated by Algorithm~\ref{alg:BGD} and $x \in [0, 1]^n$ and let $f_t = \bar{f}_t - \ushort{f}_t$ satisfy that $\bar{f}_t([n]) + \ushort{f}_t([n]) \leq L$ for all $t \geq 1$. Then, we have
\begin{equation*}
\sum_{t=1}^T \EE[(f_t)_L(x^t)] \leq \left(\sum_{t=1}^T \tfrac{1}{\alpha}(\bar{f}_t)_C(x) + \beta(-\ushort{f}_t)_C(x)\right) + \tfrac{n}{2\eta} + \tfrac{4n^2L^2\eta T}{\mu}. 
\end{equation*}
\end{lemma}
\begin{proof}
Using the same argument as in Lemma~\ref{Lemma:OGD-second}, we have
\begin{equation*}
(x^t - x)^\top \hat{g}^t \leq \tfrac{1}{2\eta}\left(\|x - x^t\|^2 - \|x - x^{t+1}\|^2\right) + \tfrac{\eta}{2}\|\hat{g}^t\|^2. 
\end{equation*}
By Lemma~\ref{Lemma:BGD-first}, we have $\EE[\hat{g}^t \mid x^t] = g^t$ and $\EE[\|\hat{g}^t\|^2 \mid x^t] \leq \tfrac{8n^2L^2}{\mu}$ for all $t \geq 1$. This implies that 
\begin{equation*}
(x^t - x)^\top g^t \leq \tfrac{1}{2\eta}\left(\|x - x^t\|^2 - \EE[\|x - x^{t+1}\|^2 \mid x^t]\right) + \tfrac{4n^2 L^2\eta}{\mu}. 
\end{equation*}
Since $f_t = \bar{f}_t - \ushort{f}_t$ where $\bar{f}_t$ and $\ushort{f}_t$ are both non-decreasing, $\bar{f}_t$ is $\alpha$-weakly DR-submodular and $\ushort{f}_t$ is $\beta$-weakly DR-supermodular, Proposition~\ref{Prop:structure} implies that 
\begin{equation*}
(x^t - x)^\top g^t \geq (f_t)_L(x^t) - \left(\tfrac{1}{\alpha}(\bar{f}_t)_C(x) + \beta(-\ushort{f}_t)_C(x)\right). 
\end{equation*}
By Lemma~\ref{Lemma:OGD-first}, we have $\|g^t\| \leq L$ for all $t = 1, 2, \ldots, T$.  Then, we have
\begin{equation*}
(f_t)_L(x^t) \leq \tfrac{1}{\alpha}(\bar{f}_t)_C(x) + \beta(-\ushort{f}_t)_C(x) + \tfrac{1}{2\eta}\left(\|x - x^t\|^2 - \EE[\|x - x^{t+1}\|^2 \mid x^t]\right) + \tfrac{4n^2L^2\eta}{\mu}. 
\end{equation*}
Taking the expectation of both sides and summing up the resulting inequality over $t = 1, 2, \ldots, T$, we have
\begin{equation*}
\sum_{t=1}^T \EE[(f_t)_L(x^t)] \leq \left(\sum_{t=1}^T \tfrac{1}{\alpha}(\bar{f}_t)_C(x) + \beta(-\ushort{f}_t)_C(x)\right) + \tfrac{1}{2\eta}\|x - x^1\|^2 + \tfrac{4n^2L^2\eta T}{\mu}. 
\end{equation*}
Using $\|x^1 - x\| \leq \sqrt{n}$ (cf. Lemma~\ref{Lemma:BGD-first}) yields the desired inequality. 
\end{proof}
To prove the high probability bound, we require the following concentration inequality. In particular, we review the Bernstein inequality for martingales~\citep{Freedman-1975-Tail} and refer to~\citet[Appendix A]{Cesa-2006-Prediction} for a proof. The following proposition is a consequence of~\citet[Lemma~A.8]{Cesa-2006-Prediction}. 
\begin{proposition}\label{Prop:Bernstein}
Let $X_1, \ldots, X_n$ be a bounded martingale difference sequence with respect to the filtration $\FCal = (\FCal_i)_{1 \leq i \leq n}$ such that $|X_i| \leq K$ for each $i = 1, \ldots, n$. We also assume that $\EE[\|X_{i+1}\|^2 \mid \FCal_i] \leq V$ for each $i = 1, \ldots, n-1$. Then for every $\delta > 0$, we have
\begin{equation*}
\PP\left(\left|\sum_{i=1}^n X_i - \EE[X_i \mid \FCal_{i-1}]\right| > \sqrt{2TV\log(1/\delta)} + \tfrac{\sqrt{2}}{3}K\log(1/\delta)\right) \leq \delta. 
\end{equation*}
\end{proposition}
Then we provide our last lemma which significantly generalizes Lemma~\ref{Lemma:BGD-second} for deriving the high-probability bounds. 
\begin{lemma}\label{Lemma:BGD-third}
Suppose that the iterates $\{x^t\}_{t \geq 1}$ are generated by Algorithm~\ref{alg:BGD} with $\mu = \frac{n}{T^{1/3}}$ and $x \in [0, 1]^n$ and let $f_t = \bar{f}_t - \ushort{f}_t$ satisfy that $\bar{f}_t([n]) + \ushort{f}_t([n]) \leq L$ for all $t \geq 1$.  Fixing a sufficiently small $\delta \in (0, 1)$ and letting $T > \log^{\frac{3}{2}}(1/\delta)$. Then, we have
\begin{equation*}
\sum_{t=1}^T (f_t)_L(x^t) \leq \left(\sum_{t=1}^T \tfrac{1}{\alpha}\bar{f}_t(S) - \beta\ushort{f}_t(S)\right) + \tfrac{n}{2\eta} + \tfrac{4n^2L^2\eta T}{\mu} + 12LT^{\frac{2}{3}}\sqrt{n^2 + n\log(1/\delta)} + 6\eta L^2T\sqrt{n\log(1/\delta)}, 
\end{equation*}
with probability at least $1 - 3\delta$. 
\end{lemma}
\begin{proof}
Using the same argument as in Lemma~\ref{Lemma:BGD-second}, we have
\begin{equation*}
(x^t - x)^\top \hat{g}^t \leq \tfrac{1}{2\eta}\left(\|x - x^t\|^2 - \|x - x^{t+1}\|^2\right) + \tfrac{\eta}{2}\|\hat{g}^t\|^2, 
\end{equation*}
and 
\begin{equation*}
(x^t - x)^\top g^t \geq (f_t)_L(x^t) - \left(\tfrac{1}{\alpha}(\bar{f}_t)_C(x) + \beta(-\ushort{f}_t)_C(x)\right). 
\end{equation*}
For simplicity, we define $e_t = \hat{g}^t - g^t$. Then, we have
\begin{equation}\label{inequality:BGD-first}
(f_t)_L(x^t) - \left(\tfrac{1}{\alpha}(\bar{f}_t)_C(x) + \beta(-\ushort{f}_t)_C(x)\right) \leq (x - x^t)^\top e_t + \tfrac{1}{2\eta}\left(\|x - x^t\|^2 - \|x - x^{t+1}\|^2\right) + \tfrac{\eta}{2}\|\hat{g}^t\|^2
\end{equation}
Summing up Eq.~\eqref{inequality:BGD-first} over $t = 1, 2, \ldots, T$ and using $\|x^1 - x\| \leq \sqrt{n}$ and $\EE[\|\hat{g}^t\|^2 \mid x^t] \leq \tfrac{8n^2L^2}{\mu}$ for all $t \geq 1$ (cf. Lemma~\ref{Lemma:BGD-first}), we have
\begin{eqnarray*}
\sum_{t=1}^T (f_t)_L(x^t) & \leq & \left(\sum_{t=1}^T \tfrac{1}{\alpha}(\bar{f}_t)_C(x) + \beta(-\ushort{f}_t)_C(x) + (x - x^t)^\top e_t + \tfrac{\eta}{2}(\|\hat{g}^t \|^2 - \EE[\|\hat{g}^t \|^2 \mid x^t])\right) \\ 
& & + \tfrac{n}{2\eta} + \tfrac{4n^2L^2\eta T}{\mu}. 
\end{eqnarray*}
By the definition of the convex closure,  we obtain that the convex closure of a set function $f$ agrees with $f$ on all the integer points~\citep[Page~4, Proposition~3.3]{Dughmi-2009-Submodular}.  Letting $S \subseteq [n]$, we have $(\bar{f}_t)_C(\chi_S) = f_t(S)$ and $(-\ushort{f}_t)_C(\chi_S) = - \beta\ushort{f}_t(S)$ which implies that 
\begin{equation*}
\tfrac{1}{\alpha}(\bar{f}_t)_C(\chi_S) + \beta(-\ushort{f}_t)_C(\chi_S) = \tfrac{1}{\alpha}f_t(S) - \beta\ushort{f}_t(S). 
\end{equation*}
Letting $x = \chi_S$, we have
\begin{equation}\label{inequality:BGD-second}
\sum_{t=1}^T (f_t)_L(x^t) \leq \left(\sum_{t=1}^T \tfrac{1}{\alpha}\bar{f}_t(S) - \beta\ushort{f}_t(S)\right) + \tfrac{n}{2\eta} + \tfrac{4n^2L^2\eta T}{\mu} + \underbrace{\sum_{t=1}^T (\chi_S - x^t)^\top e_t}_{\textbf{I}} + \tfrac{\eta}{2} \underbrace{\left(\sum_{t=1}^T \|\hat{g}^t \|^2 - \EE[\|\hat{g}^t \|^2 \mid x^t]\right)}_{\textbf{II}}. 
\end{equation}
In what follows,  we prove the high probability bounds for the terms \textbf{I} and \textbf{II} in the above inequality. 
\paragraph{Bounding \textbf{I}.} Consider the random variables $X_t = (x^t)^\top\hat{g}^t$ for all $1 \leq t \leq T$ that are adapted to the natural filtration generated by the iterates $\{x_t\}_{t \geq 1}$. By Lemma~\ref{Lemma:BGD-first} and the H\"{o}lder's inequality, we have 
\begin{equation*}
|X_t| \leq \|\hat{g}^t\|_1\|x^t\|_\infty \leq 2\left|\tfrac{f_t(A_{i_t}^t)}{(1-\mu)\lambda_{i_t}^t + \frac{\mu}{n+1}}\right| \leq \tfrac{2(n+1)L}{\mu}
\end{equation*}
Since $\mu = \frac{n}{T^{1/3}}$, we have $|X_t| \leq 4LT^{\frac{1}{3}}$. Further, we have
\begin{equation*}
\EE[X_t^2 \mid x_t] \leq \EE[\|\hat{g}^t\|_1^2\|x^t\|_\infty^2 \mid x_t] = \sum_{i=0}^n \tfrac{4(f_t(A_i^t))^2}{(1-\mu)\lambda_i^t + \frac{\mu}{n+1}} \leq \tfrac{2(n+1)^2L^2}{\mu} \leq 8nL^2 T^{\frac{1}{3}}.  
\end{equation*} 
Since $\EE[\hat{g}^t \mid x^t] = g^t$ and $e_t = \hat{g}^t - g^t$, Proposition~\ref{Prop:Bernstein} implies that
\begin{equation*}
\PP\left(\left|\sum_{t=1}^T (x^t)^\top e_t\right| > 4LT^{\frac{2}{3}}\sqrt{n\log(1/\delta)} + 2LT^{\frac{1}{3}}\log(1/\delta)\right) \leq \delta. 
\end{equation*}
Since $T > \log^{\frac{3}{2}}(1/\delta)$, we have $T^{\frac{2}{3}}\sqrt{\log(1/\delta)} \geq T^{\frac{1}{3}}\log(1/\delta)$. This implies that 
\begin{equation*}
\PP\left(\left|\sum_{t=1}^T (x^t)^\top e_t\right| > 6LT^{\frac{2}{3}}\sqrt{n\log(1/\delta)} \right) \leq \delta. 
\end{equation*}
Similarly, we fix a set $S \subseteq [n]$ and consider the random variable $X_t = (\chi_S)^\top\hat{g}^t$ for all $1 \leq t \leq T$ that are adapted to the natural filtration generated by the iterates $\{x_t\}_{t \geq 1}$. By repeating the above argument with $\frac{\delta}{2^n}$, we have
\begin{equation*}
\PP\left(\left|\sum_{t=1}^T (\chi_S)^\top e_t\right| > 6LT^{\frac{2}{3}}\sqrt{n\log(2^n/\delta)} \right) \leq \tfrac{\delta}{2^n}. 
\end{equation*}
By taking a union bound over the $2^n$ choices of $S$, we obtain that 
\begin{equation*}
\PP\left(\left|\sum_{t=1}^T (\chi_S)^\top e_t\right| > 6LT^{\frac{2}{3}}\sqrt{n\log(2^n/\delta)} \right) \leq \delta, \quad \textnormal{for any } S \subseteq [n]. 
\end{equation*}
Since $\sqrt{n\log(2^n/\delta)} \leq \sqrt{n^2 + n\log(1/\delta)}$, we have $\textbf{I} \leq 12LT^{\frac{2}{3}}\sqrt{n^2 + n\log(1/\delta)}$ with probability at least $1 - 2\delta$. 

\paragraph{Bounding \textbf{II}.} Consider the random variables $X_t = \|\hat{g}^t\|^2$ for all $1 \leq t \leq T$ that are adapted to the natural filtration generated by the iterates $\{x^t\}_{t \geq 1}$. By Lemma~\ref{Lemma:BGD-first}, we have $|X_t| \leq \tfrac{2(n+1)^2L^2}{\mu^2}$. Since $\mu = \frac{n}{T^{1/3}}$, we have $|X_t| \leq 8L^2 T^{2/3}$. Further, we have
\begin{equation*}
\EE[X_t^2 \mid x_t] \leq \sum_{i=0}^n \tfrac{2(f_t(A_i^t))^4}{((1-\mu)\lambda_i^t + \frac{\mu}{n+1})^3} \leq \tfrac{2(n+1)^4L^4}{\mu^3} \leq 32nL^4 T. 
\end{equation*} 
Applying Proposition~\ref{Prop:Bernstein}, we have
\begin{equation*}
\PP\left(\left|\sum_{t=1}^T \|\hat{g}^t\|^2 - \EE[\|\hat{g}^t\|^2 \mid x^t]\right| > 8L^2T\sqrt{n\log(1/\delta)} + 4L^2T^{\frac{2}{3}}\log(1/\delta)\right) \leq \delta. 
\end{equation*}
Since $T > \log^{\frac{3}{2}}(1/\delta)$, we have $T\sqrt{\log(1/\delta)} \geq T^{\frac{2}{3}}\log(1/\delta)$. This implies that 
\begin{equation*}
\PP\left(\left|\sum_{t=1}^T \|\hat{g}^t\|^2 - \EE[\|\hat{g}^t\|^2 \mid x^t]\right| > 12L^2T\sqrt{n\log(1/\delta)} \right) \leq \delta. 
\end{equation*}
Therefore, we conclude that $\textbf{II} \leq 12L^2T\sqrt{n\log(1/\delta)}$ with probability at least $1 - \delta$. 

Putting these pieces together with Eq.~\eqref{inequality:BGD-second} yields that 
\begin{equation*}
\sum_{t=1}^T (f_t)_L(x^t) \leq \left(\sum_{t=1}^T \tfrac{1}{\alpha}\bar{f}_t(S) - \beta\ushort{f}_t(S)\right) + \tfrac{n}{2\eta} + \tfrac{4n^2L^2\eta T}{\mu} + 12LT^{\frac{2}{3}}\sqrt{n^2 + n\log(1/\delta)} + 6\eta L^2T\sqrt{n\log(1/\delta)}, 
\end{equation*}
with probability at least $1 - 3\delta$. 
\end{proof}

\subsection{Proof of Theorem~\ref{Theorem:BGD}}
By the definition of the Lov\'{a}sz extension and $\lambda^t$, we have
\begin{equation*}
(f_t)_L(x^t) = \sum_{i = 1}^{n-1} (x_{\pi(i)}^t - x_{\pi(i+1)}^t)f_t(A_i^t) + (1 - x_{\pi(1)}^t)f_t(A_0^t) + x_{\pi(n)}^t f_t(A_n^t) = \sum_{i=0}^n \lambda_i^t f_t(A_i^t).  
\end{equation*}
By the update formula of $S^t$, we have 
\begin{equation*}
\EE[f_t(S^t) \mid x^t] - (f_t)_L(x^t) = \mu\sum_{i=0}^n \left(\frac{1}{n+1} - \lambda_i^t\right)f_t(A_i^t) \leq \mu\sum_{i=0}^n \left(\frac{1}{n+1} + \lambda_i^t\right)|f_t(A_i^t)|. 
\end{equation*}
Since $f_t = \bar{f}_t - \ushort{f}_t$ satisfy that $\bar{f}_t([n]) + \ushort{f}_t([n]) \leq L$ for all $t \geq 1$ and $\bar{f}_t$ and $\ushort{f}_t$ are both normalized and non-decreasing, we have
\begin{equation}\label{inequality:BGD-third}
\EE[f_t(S^t) \mid x^t] - (f_t)_L(x^t) \leq L\mu\sum_{i=0}^n \left(\frac{1}{n+1} + \lambda_i^t\right) = 2L\mu. 
\end{equation}
which implies that 
\begin{equation*}
\EE[f_t(S^t)] - \EE[(f_t)_L(x^t)] \leq 2L\mu. 
\end{equation*}
Using the same argument as in Theorem~\ref{Theorem:OGD}, we have
\begin{equation*}
\tfrac{1}{\alpha}(\bar{f}_t)_C(\chi_{S_\star^T}) + \beta(-\ushort{f}_t)_C(\chi_{S_\star^T}) = \tfrac{1}{\alpha}f_t(S_\star^T) - \beta\ushort{f}_t(S_\star^T), \quad \textnormal{where } S_\star^T = \argmin_{S \subseteq [n]} \sum_{t=1}^T f_t(S). 
\end{equation*}
Putting these pieces together and letting $x = \chi_{S_\star^T}$ in the inequality of Lemma~\ref{Lemma:BGD-second} yields that 
\begin{equation*}
\sum_{t=1}^T \EE[f_t(S^t)] \leq \left(\sum_{t=1}^T \tfrac{1}{\alpha} \bar{f}_t(S_\star^T) - \beta\ushort{f}_t(S_\star^T)\right) + \tfrac{n}{2\eta} + \tfrac{4n^2L^2\eta T}{\mu} + 2LT\mu. 
\end{equation*}
Plugging the choice of $\eta = \frac{1}{LT^{2/3}}$ and $\mu = \frac{n}{T^{1/3}}$ into the above inequality yields that $\EE[R_{\alpha, \beta}(T)] = O(nT^{\frac{2}{3}})$ as desired. 

We proceed to derive a high probability bound using Lemma~\ref{Lemma:BGD-third}. Indeed, we first consider the case of $T < 2\log^{\frac{3}{2}}(1/\delta)$.  Since $f_t = \bar{f}_t - \ushort{f}_t$ satisfy that $\bar{f}_t([n]) + \ushort{f}_t([n]) \leq L$ for all $t \geq 1$, we have
\begin{equation*}
R_{\alpha, \beta}(T) \leq \sum_{t=1}^T f_t(S^t) - \sum_{t=1}^T (\tfrac{1}{\alpha}\bar{f}_t(S_\star^T) - \beta\ushort{f}_t(S_\star^T)) \leq \left(1 + \tfrac{1}{\alpha} + \beta\right)LT = O(T^{\frac{2}{3}}\sqrt{\log(1/\delta)}). 
\end{equation*}
For the case of $T \geq 2\log^{\frac{3}{2}}(1/\delta)$, we obtain by combining Lemma~\ref{Lemma:BGD-third} with Eq.~\eqref{inequality:BGD-third} that 
\begin{eqnarray*}
\lefteqn{\sum_{t=1}^T \EE[f_t(S^t) \mid x^t] \leq \left(\sum_{t=1}^T \tfrac{1}{\alpha}\bar{f}_t(S) - \beta\ushort{f}_t(S)\right)} \\
& & + \tfrac{n}{2\eta} + 2nLT^{\frac{2}{3}} + 4nL^2\eta T^{\frac{4}{3}} + 12LT^{\frac{2}{3}}\sqrt{n^2 + n\log(1/\delta)} + 6\eta L^2T\sqrt{n\log(1/\delta)}, 
\end{eqnarray*}
with probability at least $1 - 3\delta$.  Then, it suffices to bound the term $\sum_{t=1}^T f_t(S^t) - \sum_{t=1}^T \EE[f_t(S^t) \mid x^t]$ using Proposition~\ref{Prop:Bernstein}.  Consider the random variables $X_t = f_t(S^t)$ for all $1 \leq t \leq T$ that are adapted to the natural filtration generated by the iterates $\{x^t\}_{t \geq 1}$. Since $f_t = \bar{f}_t - \ushort{f}_t$ satisfy that $\bar{f}_t([n]) + \ushort{f}_t([n]) \leq L$ for all $t \geq 1$, we have $|X_t| \leq L$. Further, we have $\EE[X_t^2 \mid x_t] \leq L^2$. Applying Proposition~\ref{Prop:Bernstein}, we have
\begin{equation*}
\PP\left(\left|\sum_{t=1}^T f_t(S^t) - \EE[f_t(S^t) \mid x^t]\right| > L\sqrt{2T\log(1/\delta)} + \tfrac{L}{2}\log(1/\delta)\right) \leq \delta. 
\end{equation*}
Since $T > \log^{\frac{3}{2}}(1/\delta)$, we have $\sqrt{2T\log(1/\delta)} \geq \frac{1}{2}\log(1/\delta)$. This implies that 
\begin{equation*}
\PP\left(\left|\sum_{t=1}^T f_t(S^t) - \EE[f_t(S^t) \mid x^t]\right| > 3L\sqrt{T\log(1/\delta)} \right) \leq \delta. 
\end{equation*}
Therefore, we conclude that $\sum_{t=1}^T f_t(S^t) - \sum_{t=1}^T \EE[f_t(S^t) \mid x^t] \leq 3L\sqrt{T\log(1/\delta)}$ with probability at least $1 - \delta$. Putting these pieces together yields that 
\begin{eqnarray*}
\lefteqn{\sum_{t=1}^T f_t(S^t) \leq \left(\sum_{t=1}^T \tfrac{1}{\alpha}\bar{f}_t(S) - \beta\ushort{f}_t(S)\right) + \tfrac{n}{2\eta} + 3L\sqrt{T\log(1/\delta)} + 2nLT^{\frac{2}{3}}} \\ 
& & + 4nL^2\eta T^{\frac{4}{3}} + 12nLT^{\frac{2}{3}} + 12LT^{\frac{2}{3}}\sqrt{n\log(1/\delta)} + 6\eta L^2T\sqrt{n\log(1/\delta)}, 
\end{eqnarray*}
with probability at least $1 - 4\delta$.  Plugging the choice of $\eta = \frac{1}{LT^{2/3}}$ yields that 
\begin{equation*}
\sum_{t=1}^T f_t(S^t) \leq \left(\sum_{t=1}^T \tfrac{1}{\alpha}\bar{f}_t(S) - \beta\ushort{f}_t(S)\right) + \tfrac{37}{2}nLT^{\frac{2}{3}} + 21LT^{\frac{2}{3}}\sqrt{n\log(1/\delta)},  
\end{equation*}
with probability at least $1 - 4\delta$. Letting $S = S_\star^T = \argmin_{S \subseteq [n]} \sum_{t=1}^T f_t(S)$ and changing $\delta$ to $\frac{\delta}{4}$ yields that $R_{\alpha, \beta}(T) = O(nT^{\frac{2}{3}} + \sqrt{n\log(1/\delta)}T^{\frac{2}{3}})$ with probability at least $1 - \delta$ as desired. 

%% file: sec/app_DOGD.tex
\section{Regret Analysis for Algorithm~\ref{alg:DOGD}}\label{app:DOGD}
In this section, we present several technical lemmas for analyzing the regret minimization property of Algorithm~\ref{alg:DOGD}.  We also give the missing proofs of Theorem~\ref{Theorem:DOGD}. 

\subsection{Technical lemmas}
We provide one technical lemma for Algorithm~\ref{alg:DOGD} which is analogues to Lemma~\ref{Lemma:OGD-second}. It gives a key inequality for analyzing the regret minimization property of Algorithm~\ref{alg:DOGD}.  Note that the results in Lemma~\ref{Lemma:OGD-first} still hold true for the iterates $\{x^t\}_{t \geq 1}$ and $\{g^t\}_{t \geq 1}$ generated by Algorithm~\ref{alg:DOGD}.  
\begin{lemma}\label{Lemma:DOGD-first}
Suppose that the iterates $\{x^t\}_{t \geq 1}$ are generated by Algorithm~\ref{alg:DOGD} and $x \in [0, 1]^n$ and let $f_t = \bar{f}_t - \ushort{f}_t$ satisfy that $\bar{f}_t([n]) + \ushort{f}_t([n]) \leq L$ for all $t \geq 1$. Then, we have
\begin{equation*}
\sum_{t=1}^T \EE[(f_t)_L(x^t)] \leq \left(\sum_{t=1}^T \tfrac{1}{\alpha}(\bar{f}_t)_C(x) + \beta(-\ushort{f}_t)_C(x)\right) + \tfrac{n}{2\eta_{\bar{T}}} + \tfrac{L^2}{2}\left(\sum_{t=1}^{\bar{T}} \eta_t\right) + L^2\left(\sum_{t=1}^{\bar{T}} \sum_{s=q_t}^{t-1} \eta_s\right),  
\end{equation*}
where $\bar{T} > 0$ in the above inequality satisfies that $q_{\bar{T}} = T$. 
\end{lemma}
\begin{proof}
Using the same argument as in Lemma~\ref{Lemma:OGD-second}, we have
\begin{equation*}
(x^t - x)^\top g^{q_t} \leq \tfrac{1}{2\eta_t}\left(\|x - x^t\|^2 - \|x - x^{t+1}\|^2\right) + \tfrac{\eta_t}{2}\|g^{q_t}\|^2. 
\end{equation*}
Since $f_t = \bar{f}_t - \ushort{f}_t$ where $\bar{f}_t$ and $\ushort{f}_t$ are both normalized and non-decreasing, $\bar{f}_t$ is $\alpha$-weakly DR-submodular and $\ushort{f}_t$ is $\beta$-weakly DR-supermodular, Proposition~\ref{Prop:structure} implies that 
\begin{equation*}
(x^{q_t} - x)^\top g^{q_t} \geq (f_{q_t})_L(x^{q_t}) - \left(\tfrac{1}{\alpha}(\bar{f}_{q_t})_C(x) + \beta(-\ushort{f}_{q_t})_C(x)\right). 
\end{equation*}
By Lemma~\ref{Lemma:OGD-first}, we have $\|g^t\| \leq L$ for all $t \geq 1$. Then, we have
\begin{equation}\label{inequality:DOGD-first}
(f_{q_t})_L(x^{q_t}) \leq \tfrac{1}{\alpha}(\bar{f}_{q_t})_C(x) + \beta(-\ushort{f}_{q_t})_C(x) + \tfrac{1}{2\eta_t}\left(\|x - x^t\|^2 - \|x - x^{t+1}\|^2\right) + L\|x^{q_t} - x^t\| + \tfrac{\eta_t L^2}{2}. 
\end{equation}
Further, we have
\begin{equation}\label{inequality:DOGD-second}
\|x^{q_t} - x^t\| \leq \sum_{s=q_t}^{t-1} \eta_s \|g^s\| \leq L\left(\sum_{s=q_t}^{t-1} \eta_s\right). 
\end{equation}
Plugging Eq.~\eqref{inequality:DOGD-second} into Eq.~\eqref{inequality:DOGD-first} yields that 
\begin{equation}\label{inequality:DOGD-third}
(f_{q_t})_L(x^{q_t}) \leq \tfrac{1}{\alpha}(\bar{f}_{q_t})_C(x) + \beta(-\ushort{f}_{q_t})_C(x) + \tfrac{1}{2\eta_t}\left(\|x - x^t\|^2 - \|x - x^{t+1}\|^2\right) + \tfrac{\eta_t L^2}{2} + L^2\left(\sum_{s=q_t}^{t-1} \eta_s\right). 
\end{equation}
For a fixed horizon $T \geq 1$, we have $q_{\bar{T}} = T$ for some $\bar{T} \geq T$.  Then, by summing up Eq.~\eqref{inequality:DOGD-third} over $t = 1, 2, \ldots, \bar{T}$ and using $\|x^t - x\| \leq \sqrt{n}$ for all $t \geq 1$ (cf. Lemma~\ref{Lemma:OGD-first}) and that $\{\eta_t\}_{t \geq 1}$ is nonincreasing, we have 
\begin{equation*}
\sum_{t=1}^{\bar{T}} (f_{q_t})_L(x^{q_t}) \leq \left(\sum_{t=1}^{\bar{T}} \tfrac{1}{\alpha}(\bar{f}_{q_t})_C(x) + \beta(-\ushort{f}_{q_t})_C(x)\right) + \tfrac{n}{2\eta_{\bar{T}}} + \tfrac{L^2}{2}\left(\sum_{t=1}^{\bar{T}} \eta_t\right) + L^2\left(\sum_{t=1}^{\bar{T}} \sum_{s=q_t}^{t-1} \eta_s\right). 
\end{equation*}
Since $q_{\bar{T}} = T$ and our pooling policy is induced by a priority queue (note that $f_{q_t} = \bar{f}_{q_t} = \ushort{f}_{q_t} = 0$ if $\PCal_t = \emptyset$), we have 
\begin{eqnarray*}
\sum_{t=1}^{\bar{T}} (f_{q_t})_L(x^{q_t}) & = & \sum_{t=1}^T (f_t)_L(x^t), \\
\sum_{t=1}^{\bar{T}} \tfrac{1}{\alpha}(\bar{f}_{q_t})_C(x) + \beta(-\ushort{f}_{q_t})_C(x) & = & \sum_{t=1}^T \tfrac{1}{\alpha}(\bar{f}_t)_C(x) + \beta(-\ushort{f}_t)_C(x). 
\end{eqnarray*}
Therefore, we conclude that 
\begin{equation*}
\sum_{t=1}^T (f_t)_L(x^t) \leq \left(\sum_{t=1}^T \tfrac{1}{\alpha}(\bar{f}_t)_C(x) + \beta(-\ushort{f}_t)_C(x)\right) + \tfrac{n}{2\eta_T} + \tfrac{L^2}{2}\left(\sum_{t=1}^{\bar{T}} \eta_t\right) + L^2\left(\sum_{t=1}^{\bar{T}} \sum_{s=q_t}^{t-1} \eta_s\right). 
\end{equation*}
Taking the expectation of both sides yields the desired inequality. 
\end{proof}

\subsection{Proof of Theorem~\ref{Theorem:DOGD}}
By~\citet[Corollary~1]{Heliou-2020-Gradient}, we have $t - q_t = o(t^\gamma)$ under Assumption~\ref{Assumption:delay}; in particular, we have $t - q_t = o(t)$ and $q_t = \Theta(t)$.  Since $q_{\bar{T}} = T$, we have $T = \Theta(\bar{T})$ which implies that $\bar{T} = \Theta(T)$.  Recall that $\eta_t = \frac{\sqrt{n}}{L\sqrt{t^{1+\gamma}}}$, we have
\begin{eqnarray*}
\tfrac{n}{2\eta_{\bar{T}}} & = & \tfrac{L\sqrt{n\bar{T}^{1+\gamma}}}{2} \ = \ O(\sqrt{nT^{1+\gamma}}), \\
\tfrac{L^2}{2}\left(\sum_{t=1}^{\bar{T}} \eta_t\right) & = & \tfrac{\sqrt{n}L}{2}\left(\sum_{t=1}^{\bar{T}} \tfrac{1}{\sqrt{t^{1+\gamma}}}\right) \ \leq \ \tfrac{\sqrt{n}L}{1-\gamma}\sqrt{\bar{T}^{1-\gamma}} \ = \ O(\sqrt{nT^{1-\gamma}}), \\
L^2\left(\sum_{t=1}^{\bar{T}} \sum_{s=q_t}^{t-1} \eta_s\right) & \leq & L^2\left(\sum_{t=1}^{\bar{T}} (t - q_t)\eta_{q_t}\right) \ = \ O\left(\sqrt{n}L\sum_{t=1}^{\bar{T}} \tfrac{1}{\sqrt{t^{1-\gamma}}}\right) \ = \ O(L\sqrt{n\bar{T}^{1+\gamma}}) \ = \ O(\sqrt{nT^{1+\gamma}}),
\end{eqnarray*}
Putting these pieces together with Lemma~\ref{Lemma:DOGD-first} yields that 
\begin{equation}\label{inequality:DOGD-fourth}
\sum_{t=1}^T \EE[(f_t)_L(x^t)] - \left(\sum_{t=1}^T \tfrac{1}{\alpha}(\bar{f}_t)_C(x) + \beta(-\ushort{f}_t)_C(x)\right) = O(\sqrt{nT^{1+\gamma}}).  
\end{equation}
By the definition of the Lov\'{a}sz extension, we have
\begin{equation*}
(f_t)_L(x^t) = \sum_{i = 1}^{n-1} (x_{\pi(i)}^t - x_{\pi(i+1)}^t)f_t(A_i^t) + (1 - x_{\pi(1)}^t)f_t(A_0^t) + x_{\pi(n)}^t f_t(A_n^t). 
\end{equation*}
By the update formula, we have $\EE[f_t(S^t) \mid x^t] = (f_t)_L(x^t)$ which implies that $\EE[f_t(S^t)] = \EE[(f_t)_L(x^t)]$. Further, by using the same argument as in Theorem~\ref{Theorem:OGD}, we have
\begin{equation*}
\tfrac{1}{\alpha}(\bar{f}_t)_C(\chi_{S_\star^T}) + \beta(-\ushort{f}_t)_C(\chi_{S_\star^T}) = \tfrac{1}{\alpha}\bar{f}_t(S_\star^T) - \beta\ushort{f}_t(S_\star^T). 
\end{equation*}
Putting these pieces together and letting $x = \chi_{S_\star^T}$ in Eq.~\eqref{inequality:DOGD-fourth} yields that 
\begin{equation*}
\sum_{t=1}^T \EE[f_t(S^t)] - \left(\sum_{t=1}^T \tfrac{1}{\alpha} \bar{f}_t(S_\star^T) - \beta\ushort{f}_t(S_\star^T)\right) = O(\sqrt{nT^{1+\gamma}}). 
\end{equation*}
which implies that $\EE[R_{\alpha, \beta}(T)] = O(\sqrt{nT^{1+\gamma}})$ as desired. 

We proceed to derive a high probability bound using the concentration inequality in Proposition~\ref{Prop:Hoeffding}. Indeed, we have
\begin{equation*}
\PP\left(\sum_{i=1}^n f_t(S^t) - \EE\left[\sum_{i=1}^n f_t(S^t)\right] > \epsilon\right) \leq \exp\left(-\frac{\epsilon^2}{2nL^2}\right). 
\end{equation*}
Equivalently, we have $\sum_{i=1}^n f_t(S^t) - \EE[\sum_{i=1}^n f_t(S^t)] \leq L\sqrt{2T\log(1/\delta)}$ with probability at least $1 - \delta$.  This together with $\EE[R_{\alpha, \beta}(T)] = O(\sqrt{nT^{1+\gamma}})$ yields that $R_{\alpha, \beta}(T) = O(\sqrt{nT^{1+\gamma}} + \sqrt{T\log(1/\delta)})$ with probability at least $1 - \delta$. 

%% file: sec/app_DBGD.tex
\section{Regret Analysis for Algorithm~\ref{alg:DBGD}}\label{app:DBGD}
In this section, we present several technical lemmas for analyzing the regret minimization property of Algorithm~\ref{alg:DBGD}.  We also give the missing proofs of Theorem~\ref{Theorem:DBGD}. 

\subsection{Technical lemmas}
We provide two technical lemmas for Algorithm~\ref{alg:DBGD} which are analogues to Lemma~\ref{Lemma:BGD-second} and~\ref{Lemma:BGD-third}. It gives a key inequality for analyzing the regret minimization property of Algorithm~\ref{alg:DOGD}.  Note that the results in Lemma~\ref{Lemma:BGD-first} still hold true for the iterates $\{x^t\}_{t \geq 1}$ and $\{\hat{g}^t\}_{t \geq 1}$ generated by Algorithm~\ref{alg:DBGD}.  
\begin{lemma}\label{Lemma:DBGD-first}
Suppose that the iterates $\{x^t\}_{t \geq 1}$ are generated by Algorithm~\ref{alg:DBGD} and $x \in [0, 1]^n$ and let $f_t = \bar{f}_t - \ushort{f}_t$ satisfy that $\bar{f}_t([n]) + \ushort{f}_t([n]) \leq L$ for all $t \geq 1$. Then, we have
\begin{equation*}
\sum_{t=1}^T \EE[(f_t)_L(x^t)] \leq \left(\sum_{t=1}^T \tfrac{1}{\alpha}(\bar{f}_t)_C(x) + \beta(-\ushort{f}_t)_C(x)\right) + \tfrac{n}{2\eta_T} + 4n^2L^2\left(\sum_{t=1}^{\bar{T}} \tfrac{\eta_t}{\mu_{q_t}}\right) + 4nL^2\left(\sum_{t=1}^{\bar{T}} \sum_{s=q_t}^{t-1} \tfrac{\eta_s}{\mu_s}\right), 
\end{equation*}
where $\bar{T} > 0$ in the above inequality satisfies that $q_{\bar{T}} = T$. 
\end{lemma}
\begin{proof}
Using the same argument as in Lemma~\ref{Lemma:OGD-second}, we have
\begin{equation*}
(x^t - x)^\top \hat{g}^{q_t} \leq \tfrac{1}{2\eta_t}\left(\|x - x^t\|^2 - \|x - x^{t+1}\|^2\right) + \tfrac{\eta_t}{2}\|\hat{g}^{q_t}\|^2. 
\end{equation*}
Since our pooling policy is induced by a priority queue, $\hat{g}^{q_t}$ has never been used before updating $x^{t+1}$.  Thus, we have $\EE[\hat{g}^{q_t} \mid x^t] = \EE[\hat{g}^{q_t} \mid x^{q_t}]$ and $\EE[\|\hat{g}^{q_t}\|^2 \mid x^t] = \EE[\|\hat{g}^{q_t}\|^2 \mid x^{q_t}]$. By Lemma~\ref{Lemma:BGD-first}, we have $\EE[\hat{g}^{q_t} \mid x^{q_t}] = g^{q_t}$ and $\EE[\|\hat{g}^{q_t}\|^2 \mid x^{q_t}] \leq \tfrac{8n^2L^2}{\mu_{q_t}}$ for all $t \geq 1$.  Putting these pieces together yields that 
\begin{equation*}
(x^t - x)^\top g^{q_t} \leq \tfrac{1}{2\eta_t}\left(\|x - x^t\|^2 - \EE[\|x - x^{t+1}\|^2 \mid x^t]\right) + \tfrac{4n^2 L^2\eta_t}{\mu_{q_t}}. 
\end{equation*}
Since $f_t = \bar{f}_t - \ushort{f}_t$ where $\bar{f}_t$ and $\ushort{f}_t$ are both normalized and non-decreasing, $\bar{f}_t$ is $\alpha$-weakly DR-submodular and $\ushort{f}_t$ is $\beta$-weakly DR-supermodular, Proposition~\ref{Prop:structure} implies that 
\begin{equation*}
(x^{q_t} - x)^\top g^{q_t} \geq (f_{q_t})_L(x^{q_t}) - \left(\tfrac{1}{\alpha}(\bar{f}_{q_t})_C(x) + \beta(-\ushort{f}_{q_t})_C(x)\right). 
\end{equation*}
By Lemma~\ref{Lemma:OGD-first}, we have $\|g^t\| \leq L$ for all $t \geq 1$. Then, we have
\begin{equation}\label{inequality:DBGD-first}
(f_{q_t})_L(x^{q_t}) \leq \tfrac{1}{\alpha}(\bar{f}_{q_t})_C(x) + \beta(-\ushort{f}_{q_t})_C(x) + \tfrac{1}{2\eta_t}\left(\|x - x^t\|^2 - \EE[\|x - x^{t+1}\|^2 \mid x^t]\right) + L\|x^{q_t} - x^t\| + \tfrac{4n^2 L^2\eta_t}{\mu_{q_t}}. 
\end{equation}
Further, by Lemma~\ref{Lemma:BGD-first}, we have
\begin{equation}\label{inequality:DBGD-second}
\|x^{q_t} - x^t\| \leq \sum_{s=q_t}^{t-1} \eta_s \|\hat{g}^s\| \leq 2(n+1)L\left(\sum_{s=q_t}^{t-1} \tfrac{\eta_s}{\mu_s}\right). 
\end{equation}
Plugging Eq.~\eqref{inequality:DBGD-second} into Eq.~\eqref{inequality:DBGD-first} yields that 
\begin{eqnarray*}
\lefteqn{(f_{q_t})_L(x^{q_t}) \leq \tfrac{1}{\alpha}(\bar{f}_{q_t})_C(x) + \beta(-\ushort{f}_{q_t})_C(x)} \\
& & + \tfrac{1}{2\eta_t}\left(\|x - x^t\|^2 - \EE[\|x - x^{t+1}\|^2 \mid x^t]\right) + \tfrac{4n^2 L^2\eta_t}{\mu_{q_t}} + 4nL^2\left(\sum_{s=q_t}^{t-1} \tfrac{\eta_s}{\mu_s}\right). 
\end{eqnarray*}
By using the same argument as in Lemma~\ref{Lemma:DOGD-first}, we have
\begin{eqnarray*}
\sum_{t=1}^T (f_t)_L(x^t) & \leq & \left(\sum_{t=1}^T \tfrac{1}{\alpha}(\bar{f}_t)_C(x) + \beta(-\ushort{f}_t)_C(x)\right) + \sum_{t=1}^{\bar{T}} \tfrac{1}{2\eta_t}\left(\|x - x^t\|^2 - \EE[\|x - x^{t+1}\|^2 \mid x^t]\right) \\ 
& & + 4n^2L^2\left(\sum_{t=1}^{\bar{T}} \tfrac{\eta_t}{\mu_{q_t}}\right) + 4nL^2\left(\sum_{t=1}^{\bar{T}} \sum_{s=q_t}^{t-1} \tfrac{\eta_s}{\mu_s}\right). 
\end{eqnarray*}
Taking the expectation of both sides of the above inequality and using $\|x^t - x\| \leq \sqrt{n}$ for all $t \geq 1$ (cf. Lemma~\ref{Lemma:BGD-first}) and that $\{\eta_t\}_{t \geq 1}$ is nonincreasing yields the desired inequality. 
\end{proof}
Then, we provide our second lemma which significantly generalizes Lemma~\ref{Lemma:DBGD-first} for deriving the high-probability bounds. 
\begin{lemma}\label{Lemma:DBGD-second}
Suppose that the iterates $\{x^t\}_{t \geq 1}$ are generated by Algorithm~\ref{alg:DBGD} with $\eta_t = \frac{1}{Lt^{(2+\gamma)/3}}$, $\mu_t = \frac{n}{t^{(1-\gamma)/3}}$ and $x \in [0, 1]^n$ and let $f_t = \bar{f}_t - \ushort{f}_t$ satisfy that $\bar{f}_t([n]) + \ushort{f}_t([n]) \leq L$ for all $t \geq 1$.  Fixing a sufficiently small $\delta \in (0, 1)$ and letting $T > \log^{\frac{3}{2 + \gamma}}(1/\delta)$. Then, we have
\begin{eqnarray*}
\sum_{t=1}^T (f_t)_L(x^t) & \leq & \left(\sum_{t=1}^T \tfrac{1}{\alpha}\bar{f}_t(S) - \beta\ushort{f}_t(S)\right) + \tfrac{n}{2\eta_{\bar{T}}} + 4n^2L^2\left(\sum_{t=1}^{\bar{T}} \tfrac{\eta_t}{\mu_{q_t}}\right) + 4nL^2\left(\sum_{t=1}^{\bar{T}} \sum_{s=q_t}^{t-1} \tfrac{\eta_s}{\mu_s}\right) \\ 
& & + 12L\bar{T}^{\frac{4 - \gamma}{6}}\sqrt{n^2 + n\log(1/\delta)} + 6L\sqrt{n\bar{T}\log(1/\delta)}, 
\end{eqnarray*}
with probability at least $1 - 3\delta$ where $\bar{T} > 0$ in the above inequality satisfies that $q_{\bar{T}} = T$. 
\end{lemma}
\begin{proof}
Using the same argument as in Lemma~\ref{Lemma:DBGD-first}, we have
\begin{equation*}
(x^t - x)^\top \hat{g}^{q_t} \leq \tfrac{1}{2\eta_t}\left(\|x - x^t\|^2 - \|x - x^{t+1}\|^2\right) + \tfrac{\eta_t}{2}\|\hat{g}^{q_t}\|^2, 
\end{equation*}
and 
\begin{equation*}
(x^{q_t} - x)^\top g^{q_t} \geq (f_{q_t})_L(x^{q_t}) - \left(\tfrac{1}{\alpha}(\bar{f}_{q_t})_C(x) + \beta(-\ushort{f}_{q_t})_C(x)\right). 
\end{equation*}
For simplicity, we define $e_t = \hat{g}^t - g^t$. By Lemma~\ref{Lemma:OGD-first}, we have $\|g^t\| \leq L$ for all $t \geq 1$. Then, we have
\begin{eqnarray}\label{inequality:DBGD-third}
\lefteqn{(f_{q_t})_L(x^{q_t}) - \left(\tfrac{1}{\alpha}(\bar{f}_{q_t})_C(x) + \beta(-\ushort{f}_{q_t})_C(x)\right)} \\ 
& \leq & (x - x^t)^\top e_{q_t} + \tfrac{1}{2\eta_t}\left(\|x - x^t\|^2 - \|x - x^{t+1}\|^2\right) + L\|x^{q_t} - x^t\| + \tfrac{\eta_t}{2}\|\hat{g}^{q_t}\|^2. \nonumber
\end{eqnarray}
Plugging Eq.~\eqref{inequality:DBGD-second} into Eq.~\eqref{inequality:DBGD-third} yields that 
\begin{eqnarray*}
\lefteqn{(f_{q_t})_L(x^{q_t}) - \left(\tfrac{1}{\alpha}(\bar{f}_{q_t})_C(x) + \beta(-\ushort{f}_{q_t})_C(x)\right)} \\ 
& \leq & (x - x^t)^\top e_{q_t} + \tfrac{1}{2\eta_t}\left(\|x - x^t\|^2 - \|x - x^{t+1}\|^2\right) + \tfrac{\eta_t}{2}(\|\hat{g}^t \|^2 - \EE[\|\hat{g}^t \|^2 \mid x^t]) + \tfrac{4n^2 L^2\eta_t}{\mu_{q_t}} + 4nL^2\left(\sum_{s=q_t}^{t-1} \tfrac{\eta_s}{\mu_s}\right). 
\end{eqnarray*}
By using the same argument as in Lemma~\ref{Lemma:DOGD-first}, we have
\begin{eqnarray*}
\sum_{t=1}^T (f_t)_L(x^t) & \leq & \left(\sum_{t=1}^T \tfrac{1}{\alpha}(\bar{f}_t)_C(x) + \beta(-\ushort{f}_t)_C(x)\right) + \sum_{t=1}^{\bar{T}} (x - x^t)^\top e_{q_t} + \sum_{t=1}^{\bar{T}} \tfrac{\eta_t}{2}(\|\hat{g}^t \|^2 - \EE[\|\hat{g}^t \|^2 \mid x^t]) \\ 
& & + \tfrac{n}{2\eta_{\bar{T}}} + 4n^2L^2\left(\sum_{t=1}^{\bar{T}} \tfrac{\eta_t}{\mu_{q_t}}\right) + 4nL^2\left(\sum_{t=1}^{\bar{T}} \sum_{s=q_t}^{t-1} \tfrac{\eta_s}{\mu_s}\right). 
\end{eqnarray*}
By the definition of the convex closure,  we obtain that the convex closure of a set function $f$ agrees with $f$ on all the integer points~\citep[Page~4, Proposition~3.3]{Dughmi-2009-Submodular}.  Letting $S \subseteq [n]$, we have $(\bar{f}_t)_C(\chi_S) = f_t(S)$ and $(-\ushort{f}_t)_C(\chi_S) = - \beta\ushort{f}_t(S)$ which implies that 
\begin{equation*}
\tfrac{1}{\alpha}(\bar{f}_t)_C(\chi_S) + \beta(-\ushort{f}_t)_C(\chi_S) = \tfrac{1}{\alpha}f_t(S) - \beta\ushort{f}_t(S). 
\end{equation*}
Letting $x = \chi_S$, we have
\begin{eqnarray}\label{inequality:DBGD-fourth}
\sum_{t=1}^T (f_t)_L(x^t) & \leq & \left(\sum_{t=1}^T \tfrac{1}{\alpha}\bar{f}_t(S) - \beta\ushort{f}_t(S)\right) + \underbrace{\sum_{t=1}^{\bar{T}} (\chi_S - x^t)^\top e_{q_t}}_{\textbf{I}} + \underbrace{\left(\sum_{t=1}^{\bar{T}} \tfrac{\eta_t}{2}(\|\hat{g}^t \|^2 - \EE[\|\hat{g}^t \|^2 \mid x^t])\right)}_{\textbf{II}} \nonumber \\
& & + \tfrac{n}{2\eta_{\bar{T}}} + 4n^2L^2\left(\sum_{t=1}^{\bar{T}} \tfrac{\eta_t}{\mu_{q_t}}\right) + 4nL^2\left(\sum_{t=1}^{\bar{T}} \sum_{s=q_t}^{t-1} \tfrac{\eta_s}{\mu_s}\right).  
\end{eqnarray}
In what follows,  we prove the high probability bounds for the terms \textbf{I} and \textbf{II} in the above inequality. 
\paragraph{Bounding \textbf{I}.} Consider the random variables $X_t = (x^t)^\top\hat{g}^{q_t}$ for all $1 \leq t \leq \bar{T}$ that are adapted to the natural filtration generated by the iterates $\{x_t\}_{t \geq 1}$. By Lemma~\ref{Lemma:BGD-first} and the H\"{o}lder's inequality, we have 
\begin{equation*}
|X_t| \leq \|\hat{g}^{q_t}\|_1\|x^t\|_\infty \leq \tfrac{2(n+1)L}{\mu_t}.
\end{equation*}
Since $\mu = \frac{n}{t^{(1-\gamma)/3}}$, we have $|X_t| \leq 4L\bar{T}^{\frac{1-\gamma}{3}}$ for all $1 \leq t \leq \bar{T}$. Further, we have
\begin{equation*}
\EE[X_t^2 \mid x_t] \leq \EE[\|\hat{g}^t\|_1^2\|x^t\|_\infty^2 \mid x_t] \leq \tfrac{2(n+1)^2L^2}{\mu_t} \leq 8nL^2 \bar{T}^{\frac{1-\gamma}{3}}.  
\end{equation*} 
Since $\EE[\hat{g}^{q_t} \mid x^t] = g^{q_t}$ and $e_{q_t} = \hat{g}^{q_t} - g^{q_t}$, Proposition~\ref{Prop:Bernstein} implies that
\begin{equation*}
\PP\left(\left|\sum_{t=1}^{\bar{T}} (x^t)^\top e_{q_t}\right| > 4L\bar{T}^{\frac{4 - \gamma}{6}}\sqrt{n\log(1/\delta)} + 2L\bar{T}^{\frac{1-\gamma}{3}}\log(1/\delta)\right) \leq \delta. 
\end{equation*}
Since $\bar{T} \geq T > \log^{\frac{3}{2+\gamma}}(1/\delta)$, we have $\bar{T}^{\frac{4-\gamma}{6}}\sqrt{\log(1/\delta)} \geq \bar{T}^{\frac{1-\gamma}{3}}\log(1/\delta)$. This implies that 
\begin{equation*}
\PP\left(\left|\sum_{t=1}^{\bar{T}} (x^t)^\top e_{q_t}\right| > 6L\bar{T}^{\frac{4 - \gamma}{6}}\sqrt{n\log(1/\delta)} \right) \leq \delta. 
\end{equation*}
Similarly, we fix a set $S \subseteq [n]$ and consider the random variable $X_t = (\chi_S)^\top\hat{g}^t$ for all $1 \leq t \leq \bar{T}$ that are adapted to the natural filtration generated by the iterates $\{x_t\}_{t \geq 1}$. By repeating the above argument with $\frac{\delta}{2^n}$, we have
\begin{equation*}
\PP\left(\left|\sum_{t=1}^{\bar{T}} (\chi_S)^\top e_{q_t}\right| > 6L\bar{T}^{\frac{4 - \gamma}{6}}\sqrt{n\log(2^n/\delta)} \right) \leq \tfrac{\delta}{2^n}. 
\end{equation*}
By taking a union bound over the $2^n$ choices of $S$, we obtain that 
\begin{equation*}
\PP\left(\left|\sum_{t=1}^{\bar{T}} (\chi_S)^\top e_{q_t}\right| > 6L\bar{T}^{\frac{4 - \gamma}{6}}\sqrt{n\log(2^n/\delta)} \right) \leq \delta, \quad \textnormal{for any } S \subseteq [n]. 
\end{equation*}
Since $\sqrt{n\log(2^n/\delta)} \leq \sqrt{n^2 + n\log(1/\delta)}$, we have $\textbf{I} \leq 12L\bar{T}^{\frac{4 - \gamma}{6}}\sqrt{n^2 + n\log(1/\delta)}$ with probability at least $1 - 2\delta$. 

\paragraph{Bounding \textbf{II}.} Consider the random variables $X_t = \frac{\eta_t}{2}\|\hat{g}^{q_t}\|^2$ for all $1 \leq t \leq \bar{T}$ that are adapted to the natural filtration generated by the iterates $\{x^t\}_{t \geq 1}$. By Lemma~\ref{Lemma:BGD-first}, we have $|X_t| \leq \tfrac{(n+1)^2L^2\eta_t}{\mu_t^2}$. Since $\eta_t = \frac{1}{Lt^{(2+\gamma)/3}}$ and $\mu_t = \frac{n}{t^{(1-\gamma)/3}}$, we have $|X_t| \leq 4L$. Further, we have
\begin{equation*}
\EE[X_t^2 \mid x_t] \leq \tfrac{(n+1)^4L^4\eta_t^2}{2\mu_t^3} \leq 8nL^2. 
\end{equation*} 
Applying Proposition~\ref{Prop:Bernstein}, we have
\begin{equation*}
\PP\left(\left|\sum_{t=1}^{\bar{T}} \tfrac{\eta_t}{2}(\|\hat{g}^{q_t}\|^2 - \EE[\|\hat{g}^{q_t}\|^2 \mid x^t])\right| > 4L\sqrt{n\bar{T}\log(1/\delta)} + 2L\log(1/\delta)\right) \leq \delta. 
\end{equation*}
Since $\bar{T} \geq T > \log^{\frac{3}{2+\gamma}}(1/\delta)$, we have $\sqrt{\bar{T}\log(1/\delta)} \geq \log(1/\delta)$. This implies that 
\begin{equation*}
\PP\left(\left|\sum_{t=1}^{\bar{T}} \tfrac{\eta_t}{2}(\|\hat{g}^{q_t}\|^2 - \EE[\|\hat{g}^{q_t}\|^2 \mid x^t])\right| > 6L\sqrt{n\bar{T}\log(1/\delta)} \right) \leq \delta. 
\end{equation*}
Therefore, we conclude that $\textbf{II} \leq 6L\sqrt{n\bar{T}\log(1/\delta)}$ with probability at least $1 - \delta$. Putting these pieces together with Eq.~\eqref{inequality:DBGD-fourth} yields that 
\begin{eqnarray*}
\sum_{t=1}^T (f_t)_L(x^t) & \leq & \left(\sum_{t=1}^T \tfrac{1}{\alpha}\bar{f}_t(S) - \beta\ushort{f}_t(S)\right) + \tfrac{n}{2\eta_{\bar{T}}} + 4n^2L^2\left(\sum_{t=1}^{\bar{T}} \tfrac{\eta_t}{\mu_{q_t}}\right) + 4nL^2\left(\sum_{t=1}^{\bar{T}} \sum_{s=q_t}^{t-1} \tfrac{\eta_s}{\mu_s}\right) \\ 
& & + 12L\bar{T}^{\frac{4 - \gamma}{6}}\sqrt{n^2 + n\log(1/\delta)} + 6L\sqrt{n\bar{T}\log(1/\delta)}, 
\end{eqnarray*}
with probability at least $1 - 3\delta$. 
\end{proof}

\subsection{Proof of Theorem~\ref{Theorem:DBGD}}
By~\citet[Corollary~1]{Heliou-2020-Gradient}, we have $t - q_t = o(t^\gamma)$ under Assumption~\ref{Assumption:delay}; in particular, we have $t - q_t = o(t)$ and $q_t = \Theta(t)$.  Since $q_{\bar{T}} = T$, we have $T = \Theta(\bar{T})$ which implies that $\bar{T} = \Theta(T)$.  Recall that $\eta_t = \frac{1}{Lt^{(2+\gamma)/3}}$ and $\mu_t = \frac{n}{t^{(1-\gamma)/3}}$, we have
\begin{eqnarray*}
\tfrac{n}{2\eta_{\bar{T}}} & = & \tfrac{nL\bar{T}^{\frac{2+\gamma}{3}}}{2} \ = \ O(nT^{\frac{2+\gamma}{3}}), \\
4n^2L^2\left(\sum_{t=1}^{\bar{T}} \tfrac{\eta_t}{\mu_{q_t}}\right) & = & 4nL\left(\sum_{t=1}^{\bar{T}} \tfrac{(q_t)^{\frac{1-\gamma}{3}}}{t^{\frac{2+\gamma}{3}}}\right) \ = \ O\left(nL\sum_{t=1}^{\bar{T}} \tfrac{1}{t^{\frac{1+2\gamma}{3}}}\right) \ = \ O(nL\bar{T}^{\frac{2-2\gamma}{3}}) \ = \ O(nT^{\frac{2-2\gamma}{3}}), \\
4nL^2\left(\sum_{t=1}^{\bar{T}} \sum_{s=q_t}^{t-1} \tfrac{\eta_s}{\mu_s}\right) & \leq & 4L\left(\sum_{t=1}^{\bar{T}} (t - q_t)\tfrac{\eta_{q_t}}{\mu_t}\right) \ = \ O\left(L\sum_{t=1}^{\bar{T}} \tfrac{1}{t^{\frac{1-\gamma}{3}}}\right) \ = \ O(L\bar{T}^{\frac{2+\gamma}{3}}) \ = \ O(T^{\frac{2+\gamma}{3}}),
\end{eqnarray*}
Putting these pieces together with Lemma~\ref{Lemma:DBGD-first} yields that 
\begin{equation}\label{inequality:DBGD-fifth}
\sum_{t=1}^T \EE[(f_t)_L(x^t)] - \left(\sum_{t=1}^T \tfrac{1}{\alpha}(\bar{f}_t)_C(x) + \beta(-\ushort{f}_t)_C(x)\right) = O(nT^{\frac{2+\gamma}{3}}).  
\end{equation}
By using the similar argument as in Theorem~\ref{Theorem:BGD}, we have
\begin{equation}\label{inequality:DBGD-sixth}
\EE[f_t(S^t) \mid x^t] - (f_t)_L(x^t) \leq L\mu_t\sum_{i=0}^n \left(\tfrac{1}{n+1} + \lambda_i^t\right) = 2L\mu_t. 
\end{equation}
which implies that 
\begin{equation*}
\sum_{t=1}^T \EE[f_t(S^t)] - \EE[(f_t)_L(x^t)] \ \leq \ 2L \sum_{t=1}^T\mu_t = O(nT^{\frac{2+\gamma}{3}}). 
\end{equation*}
Using the same argument as in Theorem~\ref{Theorem:OGD}, we have
\begin{equation*}
\tfrac{1}{\alpha}(\bar{f}_t)_C(\chi_{S_\star^T}) + \beta(-\ushort{f}_t)_C(\chi_{S_\star^T}) = \tfrac{1}{\alpha}f_t(S_\star^T) - \beta\ushort{f}_t(S_\star^T), \quad \textnormal{where } S_\star^T = \argmin_{S \subseteq [n]} \sum_{t=1}^T f_t(S). 
\end{equation*}
Putting these pieces together and letting $x = \chi_{S_\star^T}$ in Eq.~\eqref{inequality:DBGD-fifth} yields that 
\begin{equation*}
\sum_{t=1}^T \EE[f_t(S^t)] - \left(\sum_{t=1}^T \tfrac{1}{\alpha} \bar{f}_t(S_\star^T) - \beta\ushort{f}_t(S_\star^T)\right) = O(nT^{\frac{2+\gamma}{3}}). 
\end{equation*}
which implies that $\EE[R_{\alpha, \beta}(T)] = O(nT^{\frac{2+\gamma}{3}})$ as desired. 

We proceed to derive a high probability bound using Lemma~\ref{Lemma:DBGD-second}. Indeed, we first consider the case of $T < 2\log^{\frac{3}{2+\gamma}}(1/\delta)$.  Since $f_t = \bar{f}_t - \ushort{f}_t$ satisfy that $\bar{f}_t([n]) + \ushort{f}_t([n]) \leq L$ for all $t \geq 1$, we have
\begin{equation*}
R_{\alpha, \beta}(T) \leq \sum_{t=1}^T f_t(S^t) - \sum_{t=1}^T (\tfrac{1}{\alpha}\bar{f}_t(S_\star^T) - \beta\ushort{f}_t(S_\star^T)) \leq \left(1 + \tfrac{1}{\alpha} + \beta\right)LT = O(T^{\frac{4-\gamma}{6}}\sqrt{\log(1/\delta)}). 
\end{equation*}
For the case of $T \geq 2\log^{\frac{3}{2+\gamma}}(1/\delta)$, we obtain by combining Lemma~\ref{Lemma:DBGD-second} with Eq.~\eqref{inequality:DBGD-sixth} that 
\begin{eqnarray*}
\lefteqn{\sum_{t=1}^T \EE[f_t(S^t) \mid x^t] \leq \left(\sum_{t=1}^T \tfrac{1}{\alpha}\bar{f}_t(S) - \beta\ushort{f}_t(S)\right) + \tfrac{n}{2\eta_{\bar{T}}} + 2L\left(\sum_{t=1}^T\mu_t\right) + 4n^2L^2\left(\sum_{t=1}^{\bar{T}} \tfrac{\eta_t}{\mu_{q_t}}\right)} \\ 
& & + 4nL^2\left(\sum_{t=1}^{\bar{T}} \sum_{s=q_t}^{t-1} \tfrac{\eta_s}{\mu_s}\right) + 12L\bar{T}^{\frac{4 - \gamma}{6}}\sqrt{n^2 + n\log(1/\delta)} + 6L\sqrt{n\bar{T}\log(1/\delta)}, 
\end{eqnarray*}
with probability at least $1 - 3\delta$.  Then, it suffices to bound the term $\sum_{t=1}^T f_t(S^t) - \sum_{t=1}^T \EE[f_t(S^t) \mid x^t]$ using Proposition~\ref{Prop:Bernstein}.  By using the same argument as in Theorem~\ref{Theorem:BGD}, we have
\begin{equation*}
\PP\left(\left|\sum_{t=1}^T f_t(S^t) - \EE[f_t(S^t) \mid x^t]\right| > 3L\sqrt{T\log(1/\delta)} \right) \leq \delta, 
\end{equation*}
which implies that $\sum_{t=1}^T f_t(S^t) - \sum_{t=1}^T \EE[f_t(S^t) \mid x^t] \leq 3L\sqrt{T\log(1/\delta)}$ with probability at least $1 - \delta$. Putting these pieces together yields that 
\begin{eqnarray*}
\lefteqn{\sum_{t=1}^T f_t(S^t) \leq \left(\sum_{t=1}^T \tfrac{1}{\alpha}\bar{f}_t(S) - \beta\ushort{f}_t(S)\right) + \tfrac{n}{2\eta_{\bar{T}}} + 2L\left(\sum_{t=1}^T\mu_t\right) + 4n^2L^2\left(\sum_{t=1}^{\bar{T}} \tfrac{\eta_t}{\mu_{q_t}}\right)} \\ 
& & + 4nL^2\left(\sum_{t=1}^{\bar{T}} \sum_{s=q_t}^{t-1} \tfrac{\eta_s}{\mu_s}\right) + 3L\sqrt{T\log(1/\delta)} + 12L\bar{T}^{\frac{4 - \gamma}{6}}\sqrt{n^2 + n\log(1/\delta)} + 6L\sqrt{n\bar{T}\log(1/\delta)}, 
\end{eqnarray*}
with probability at least $1 - 4\delta$.  Plugging the choices of $\eta_t = \frac{1}{Lt^{(2+\gamma)/3}}$ and $\mu_t = \frac{n}{t^{(1-\gamma)/3}}$ and $\bar{T} = \Theta(T)$ yields that 
\begin{equation*}
\sum_{t=1}^T f_t(S^t) - \left(\sum_{t=1}^T \tfrac{1}{\alpha}\bar{f}_t(S) - \beta\ushort{f}_t(S)\right) = O\left(nT^{\frac{2+\gamma}{3}} + \sqrt{n\log(1/\delta)}T^{\frac{4-\gamma}{6}}\right),  
\end{equation*}
with probability at least $1 - 4\delta$. Letting $S = S_\star^T = \argmin_{S \subseteq [n]} \sum_{t=1}^T f_t(S)$ and changing $\delta$ to $\frac{\delta}{4}$ yields that $R_{\alpha, \beta}(T) = O(nT^{\frac{2+\gamma}{3}} + \sqrt{n\log(1/\delta)}T^{\frac{4-\gamma}{6}})$ with probability at least $1 - \delta$ as desired.

%% file: Arxiv_submission.bbl
\begin{thebibliography}{83}
\providecommand{\natexlab}[1]{#1}
\providecommand{\url}[1]{\texttt{#1}}
\expandafter\ifx\csname urlstyle\endcsname\relax
  \providecommand{\doi}[1]{doi: #1}\else
  \providecommand{\doi}{doi: \begingroup \urlstyle{rm}\Url}\fi

\bibitem[Anari et~al.(2019)Anari, Haghtalab, Naor, Pokutta, Singh, and
  Torrico]{Anari-2019-Structured}
N.~Anari, N.~Haghtalab, S.~Naor, S.~Pokutta, M.~Singh, and A.~Torrico.
\newblock Structured robust submodular maximization: Offline and online
  algorithms.
\newblock In \emph{AISTATS}, pages 3128--3137. PMLR, 2019.

\bibitem[Arora et~al.(2012)Arora, Hazan, and Kale]{Arora-2012-Multiplicative}
S.~Arora, E.~Hazan, and S.~Kale.
\newblock The multiplicative weights update method: A meta-algorithm and
  applications.
\newblock \emph{Theory of Computing}, 8\penalty0 (1):\penalty0 121--164, 2012.

\bibitem[Auer(2002)]{Auer-2002-Using}
P.~Auer.
\newblock Using confidence bounds for exploitation-exploration trade-offs.
\newblock \emph{Journal of Machine Learning Research}, 3\penalty0
  (Nov):\penalty0 397--422, 2002.

\bibitem[Bach(2010)]{Bach-2010-Structured}
F.~Bach.
\newblock Structured sparsity-inducing norms through submodular functions.
\newblock In \emph{NeurIPS}, pages 118--126, 2010.

\bibitem[Bai et~al.(2016)Bai, Iyer, Wei, and Bilmes]{Bai-2016-Algorithms}
W.~Bai, R.~Iyer, K.~Wei, and J.~Bilmes.
\newblock Algorithms for optimizing the ratio of submodular functions.
\newblock In \emph{ICML}, pages 2751--2759. PMLR, 2016.

\bibitem[Bian et~al.(2017)Bian, Buhmann, Krause, and
  Tschiatschek]{Bian-2017-Guarantees}
A.~A. Bian, J.~M. Buhmann, A.~Krause, and S.~Tschiatschek.
\newblock Guarantees for greedy maximization of non-submodular functions with
  applications.
\newblock In \emph{ICML}, pages 498--507. PMLR, 2017.

\bibitem[Bistritz et~al.(2019)Bistritz, Zhou, Chen, Bambos, and
  Blanchet]{Bistritz-2019-Online}
I.~Bistritz, Z.~Zhou, X.~Chen, N.~Bambos, and J.~Blanchet.
\newblock Online {EXP3} learning in adversarial bandits with delayed feedback.
\newblock In \emph{NeurIPS}, page 11349–11358, 2019.

\bibitem[Blum and Mansour(2007)]{Blum-2007-External}
A.~Blum and Y.~Mansour.
\newblock From external to internal regret.
\newblock \emph{Journal of Machine Learning Research}, 8\penalty0 (6), 2007.

\bibitem[Bogunovic et~al.(2016)Bogunovic, Scarlett, Krause, and
  Cevher]{Bogunovic-2016-Truncated}
I.~Bogunovic, J.~Scarlett, A.~Krause, and V.~Cevher.
\newblock Truncated variance reduction: A unified approach to {B}ayesian
  optimization and level-set estimation.
\newblock In \emph{NeurIPS}, pages 1507--1515, 2016.

\bibitem[Bogunovic et~al.(2018)Bogunovic, Zhao, and
  Cevher]{Bogunovic-2018-Robust}
I.~Bogunovic, J.~Zhao, and V.~Cevher.
\newblock Robust maximization of non-submodular objectives.
\newblock In \emph{AISTATS}, pages 890--899. PMLR, 2018.

\bibitem[Boykov and Kolmogorov(2004)]{Boykov-2004-Experimental}
Y.~Boykov and V.~Kolmogorov.
\newblock An experimental comparison of min-cut/max-flow algorithms for energy
  minimization in vision.
\newblock \emph{IEEE Transactions on Pattern Analysis and Machine
  Intelligence}, 26\penalty0 (9):\penalty0 1124--1137, 2004.

\bibitem[Boykov et~al.(2001)Boykov, Veksler, and Zabih]{Boykov-2001-Fast}
Y.~Boykov, O.~Veksler, and R.~Zabih.
\newblock Fast approximate energy minimization via graph cuts.
\newblock \emph{IEEE Transactions on Pattern Analysis and Machine
  Intelligence}, 23\penalty0 (11):\penalty0 1222--1239, 2001.

\bibitem[Buchbinder et~al.(2014)Buchbinder, Feldman, and
  Schwartz]{Buchbinder-2014-Online}
N.~Buchbinder, M.~Feldman, and R.~Schwartz.
\newblock Online submodular maximization with preemption.
\newblock In \emph{SODA}, pages 1202--1216. SIAM, 2014.

\bibitem[Cardoso and Cummings(2019)]{Cardoso-2019-Differentially}
A.~R. Cardoso and R.~Cummings.
\newblock Differentially private online submodular minimization.
\newblock In \emph{AISTATS}, pages 1650--1658. PMLR, 2019.

\bibitem[Cesa-Bianchi and Lugosi(2006)]{Cesa-2006-Prediction}
N.~Cesa-Bianchi and G.~Lugosi.
\newblock \emph{Prediction, Learning, and Games}.
\newblock Cambridge University Press, 2006.

\bibitem[Cesa-Bianchi and Lugosi(2012)]{Cesa-2012-Combinatorial}
N.~Cesa-Bianchi and G.~Lugosi.
\newblock Combinatorial bandits.
\newblock \emph{Journal of Computer and System Sciences}, 78\penalty0
  (5):\penalty0 1404--1422, 2012.

\bibitem[Chapelle(2014)]{Chapelle-2014-Modeling}
O.~Chapelle.
\newblock Modeling delayed feedback in display advertising.
\newblock In \emph{KDD}, pages 1097--1105, 2014.

\bibitem[Chen et~al.(2018{\natexlab{a}})Chen, Feldman, and
  Karbasi]{Chen-2018-Weakly}
L.~Chen, M.~Feldman, and A.~Karbasi.
\newblock Weakly submodular maximization beyond cardinality constraints: Does
  randomization help greedy?
\newblock In \emph{ICML}, pages 804--813. PMLR, 2018{\natexlab{a}}.

\bibitem[Chen et~al.(2018{\natexlab{b}})Chen, Harshaw, Hassani, and
  Karbasi]{Chen-2018-Projection}
L.~Chen, C.~Harshaw, H.~Hassani, and A.~Karbasi.
\newblock Projection-free online optimization with stochastic gradient: From
  convexity to submodularity.
\newblock In \emph{ICML}, pages 814--823. PMLR, 2018{\natexlab{b}}.

\bibitem[Chen et~al.(2018{\natexlab{c}})Chen, Hassani, and
  Karbasi]{Chen-2018-Online}
L.~Chen, H.~Hassani, and A.~Karbasi.
\newblock Online continuous submodular maximization.
\newblock In \emph{AISTATS}, pages 1896--1905. PMLR, 2018{\natexlab{c}}.

\bibitem[Conforti and Cornu{\'e}jols(1984)]{Conforti-1984-Submodular}
M.~Conforti and G.~Cornu{\'e}jols.
\newblock Submodular set functions, matroids and the greedy algorithm: tight
  worst-case bounds and some generalizations of the {R}ado-{E}dmonds theorem.
\newblock \emph{Discrete Applied Mathematics}, 7\penalty0 (3):\penalty0
  251--274, 1984.

\bibitem[Dani et~al.(2008)Dani, Hayes, and Kakade]{Dani-2008-Stochastic}
V.~Dani, T.~P. Hayes, and S.~M. Kakade.
\newblock Stochastic linear optimization under bandit feedback.
\newblock In \emph{COLT}, pages 355--366. Omnipress, 2008.

\bibitem[Das and Kempe(2011)]{Das-2011-Submodular}
A.~Das and D.~Kempe.
\newblock Submodular meets spectral: Greedy algorithms for subset selection,
  sparse approximation and dictionary selection.
\newblock In \emph{ICML}, pages 1057--1064, 2011.

\bibitem[Dughmi(2009)]{Dughmi-2009-Submodular}
S.~Dughmi.
\newblock Submodular functions: Extensions, distributions, and algorithms. a
  survey.
\newblock \emph{ArXiv Preprint: 0912.0322}, 2009.

\bibitem[Edmonds(2003)]{Edmonds-2003-Submodular}
J.~Edmonds.
\newblock Submodular functions, matroids, and certain polyhedra.
\newblock In \emph{Combinatorial Optimization—Eureka, You Shrink!}, pages
  11--26. Springer, 2003.

\bibitem[El~Halabi and Cevher(2015)]{El-2015-Totally}
M.~El~Halabi and V.~Cevher.
\newblock A totally unimodular view of structured sparsity.
\newblock In \emph{AISTATS}, pages 223--231. PMLR, 2015.

\bibitem[El~Halabi and Jegelka(2020)]{El-2020-Optimal}
M.~El~Halabi and S.~Jegelka.
\newblock Optimal approximation for unconstrained non-submodular minimization.
\newblock In \emph{ICML}, pages 3961--3972. PMLR, 2020.

\bibitem[El~Halabi et~al.(2018)El~Halabi, Bach, and
  Cevher]{El-2018-Combinatorial}
M.~El~Halabi, F.~Bach, and V.~Cevher.
\newblock Combinatorial penalties: Which structures are preserved by convex
  relaxations?
\newblock In \emph{AISTATS}, pages 1551--1560. PMLR, 2018.

\bibitem[Elenberg et~al.(2018)Elenberg, Khanna, Dimakis, and
  Negahban]{Elenberg-2018-Restricted}
E.~R. Elenberg, R.~Khanna, A.~G. Dimakis, and S.~Negahban.
\newblock Restricted strong convexity implies weak submodularity.
\newblock \emph{The Annals of Statistics}, 46\penalty0 (6B):\penalty0
  3539--3568, 2018.

\bibitem[Flaxman et~al.(2005)Flaxman, Kalai, and McMahan]{Flaxman-2005-Online}
A.~D. Flaxman, A.~T. Kalai, and H.~B. McMahan.
\newblock Online convex optimization in the bandit setting: gradient descent
  without a gradient.
\newblock In \emph{SODA}, pages 385--394, 2005.

\bibitem[Freedman(1975)]{Freedman-1975-Tail}
D.~A. Freedman.
\newblock On tail probabilities for martingales.
\newblock \emph{The Annals of Probability}, pages 100--118, 1975.

\bibitem[Fujishige(2005)]{Fujishige-2005-Submodular}
S.~Fujishige.
\newblock \emph{Submodular Functions and Optimization}.
\newblock Elsevier, 2005.

\bibitem[Gonz{\'a}lez et~al.(2016)Gonz{\'a}lez, Dai, Hennig, and
  Lawrence]{Gonzalez-2016-Batch}
J.~Gonz{\'a}lez, Z.~Dai, P.~Hennig, and N.~Lawrence.
\newblock Batch {B}ayesian optimization via local penalization.
\newblock In \emph{AISTATS}, pages 648--657. PMLR, 2016.

\bibitem[Gr{\"o}tschel et~al.(2012)Gr{\"o}tschel, Lov{\'a}sz, and
  Schrijver]{Grotschel-2012-Geometric}
M.~Gr{\"o}tschel, L.~Lov{\'a}sz, and A.~Schrijver.
\newblock \emph{Geometric algorithms and combinatorial optimization}, volume~2.
\newblock Springer Science \& Business Media, 2012.

\bibitem[Gyorgy and Joulani(2021)]{Gyorgy-2021-Adapting}
A.~Gyorgy and P.~Joulani.
\newblock Adapting to delays and data in adversarial multi-armed bandits.
\newblock In \emph{ICML}, pages 3988--3997. PMLR, 2021.

\bibitem[Harshaw et~al.(2019)Harshaw, Feldman, Ward, and
  Karbasi]{Harshaw-2019-Submodular}
C.~Harshaw, M.~Feldman, J.~Ward, and A.~Karbasi.
\newblock Submodular maximization beyond non-negativity: Guarantees, fast
  algorithms, and applications.
\newblock In \emph{ICML}, pages 2634--2643. PMLR, 2019.

\bibitem[Harvey et~al.(2020)Harvey, Liaw, and Soma]{Harvey-2020-Improved}
N.~Harvey, C.~Liaw, and T.~Soma.
\newblock Improved algorithms for online submodular maximization via
  first-order regret bounds.
\newblock In \emph{NeurIPS}, pages 123--133. Curran Associates, Inc., 2020.

\bibitem[Hassidim and Singer(2018)]{Hassidim-2018-Optimization}
A.~Hassidim and Y.~Singer.
\newblock Optimization for approximate submodularity.
\newblock In \emph{NeurIPS}, pages 394--405, 2018.

\bibitem[Hazan(2016)]{Hazan-2016-Introduction}
E.~Hazan.
\newblock Introduction to online convex optimization.
\newblock \emph{Foundations and Trends in Optimization}, 2\penalty0
  (3-4):\penalty0 157--325, 2016.

\bibitem[Hazan and Kale(2012)]{Hazan-2012-Online}
E.~Hazan and S.~Kale.
\newblock Online submodular minimization.
\newblock \emph{Journal of Machine Learning Research}, 13\penalty0 (10), 2012.

\bibitem[Hazan et~al.(2007)Hazan, Agarwal, and Kale]{Hazan-2007-Logarithmic}
E.~Hazan, A.~Agarwal, and S.~Kale.
\newblock Logarithmic regret algorithms for online convex optimization.
\newblock \emph{Machine Learning}, 69\penalty0 (2-3):\penalty0 169--192, 2007.

\bibitem[H{\'e}liou et~al.(2020)H{\'e}liou, Mertikopoulos, and
  Zhou]{Heliou-2020-Gradient}
A.~H{\'e}liou, P.~Mertikopoulos, and Z.~Zhou.
\newblock Gradient-free online learning in continuous games with delayed
  rewards.
\newblock In \emph{ICML}, pages 4172--4181. PMLR, 2020.

\bibitem[Hoeffding(1963)]{Hoeffding-1963-Probability}
W.~Hoeffding.
\newblock Probability inequalities for sums of bounded random variables.
\newblock \emph{Journal of the American Statistical Association}, 58\penalty0
  (301):\penalty0 13--30, 1963.

\bibitem[Horel and Singer(2016)]{Horel-2016-Maximization}
T.~Horel and Y.~Singer.
\newblock Maximization of approximately submodular functions.
\newblock In \emph{NeurIPS}, pages 3045--3053, 2016.

\bibitem[Iwata(2003)]{Iwata-2003-Faster}
S.~Iwata.
\newblock A faster scaling algorithm for minimizing submodular functions.
\newblock \emph{SIAM Journal on Computing}, 32\penalty0 (4):\penalty0 833--840,
  2003.

\bibitem[Iwata and Orlin(2009)]{Iwata-2009-Simple}
S.~Iwata and J.~B. Orlin.
\newblock A simple combinatorial algorithm for submodular function
  minimization.
\newblock In \emph{SODA}, pages 1230--1237, 2009.

\bibitem[Iwata et~al.(2001)Iwata, Fleischer, and
  Fujishige]{Iwata-2001-Combinatorial}
S.~Iwata, L.~Fleischer, and S.~Fujishige.
\newblock A combinatorial strongly polynomial algorithm for minimizing
  submodular functions.
\newblock \emph{Journal of the ACM}, 48\penalty0 (4):\penalty0 761--777, 2001.

\bibitem[Iyer and Bilmes(2012)]{Iyer-2012-Algorithms}
R.~Iyer and J.~Bilmes.
\newblock Algorithms for approximate minimization of the difference between
  submodular functions, with applications.
\newblock In \emph{UAI}, pages 407--417, 2012.

\bibitem[Iyer et~al.(2013)Iyer, Jegelka, and Bilmes]{Iyer-2013-Curvature}
R.~K. Iyer, S.~Jegelka, and J.~A. Bilmes.
\newblock Curvature and optimal algorithms for learning and minimizing
  submodular functions.
\newblock In \emph{NeurIPS}, pages 2742--2750, 2013.

\bibitem[Jegelka and Bilmes(2011)]{Jegelka-2011-Online}
S.~Jegelka and J.~Bilmes.
\newblock Online submodular minimization for combinatorial structures.
\newblock In \emph{ICML}, pages 345--352, 2011.

\bibitem[Joulani et~al.(2013)Joulani, Gyorgy, and
  Szepesv{\'a}ri]{Joulani-2013-Online}
P.~Joulani, A.~Gyorgy, and C.~Szepesv{\'a}ri.
\newblock Online learning under delayed feedback.
\newblock In \emph{ICML}, pages 1453--1461. PMLR, 2013.

\bibitem[Joulani et~al.(2016)Joulani, Gyorgy, and
  Szepesv{\'a}ri]{Joulani-2016-Delay}
P.~Joulani, A.~Gyorgy, and C.~Szepesv{\'a}ri.
\newblock Delay-tolerant online convex optimization: Unified analysis and
  adaptive-gradient algorithms.
\newblock In \emph{AAAI}, pages 1744--1750, 2016.

\bibitem[Kalai and Vempala(2005)]{Kalai-2005-Efficient}
A.~Kalai and S.~Vempala.
\newblock Efficient algorithms for online decision problems.
\newblock \emph{Journal of Computer and System Sciences}, 71\penalty0
  (3):\penalty0 291--307, 2005.

\bibitem[Kawahara et~al.(2015)Kawahara, Iyer, and
  Bilmes]{Kawahara-2015-Approximate}
Y.~Kawahara, R.~Iyer, and J.~Bilmes.
\newblock On approximate non-submodular minimization via tree-structured
  supermodularity.
\newblock In \emph{AISTATS}, pages 444--452. PMLR, 2015.

\bibitem[Kuhnle et~al.(2018)Kuhnle, Smith, Crawford, and
  Thai]{Kuhnle-2018-Fast}
A.~Kuhnle, J.~D. Smith, V.~Crawford, and M.~Thai.
\newblock Fast maximization of non-submodular, monotonic functions on the
  integer lattice.
\newblock In \emph{ICML}, pages 2786--2795. PMLR, 2018.

\bibitem[Lattimore and Szepesv{\'a}ri(2020)]{Lattimore-2020-Bandit}
T.~Lattimore and C.~Szepesv{\'a}ri.
\newblock \emph{Bandit Algorithms}.
\newblock Cambridge University Press, 2020.

\bibitem[Lee et~al.(2015)Lee, Sidford, and Wong]{Lee-2015-Faster}
Y.~T. Lee, A.~Sidford, and S.~C-W. Wong.
\newblock A faster cutting plane method and its implications for combinatorial
  and convex optimization.
\newblock In \emph{FOCS}, pages 1049--1065. IEEE, 2015.

\bibitem[Lehmann et~al.(2006)Lehmann, Lehmann, and
  Nisan]{Lehmann-2006-Combinatorial}
B.~Lehmann, D.~Lehmann, and N.~Nisan.
\newblock Combinatorial auctions with decreasing marginal utilities.
\newblock \emph{Games and Economic Behavior}, 55\penalty0 (2):\penalty0
  270--296, 2006.

\bibitem[Lov{\'a}sz(1983)]{Lovasz-1983-Submodular}
L.~Lov{\'a}sz.
\newblock Submodular functions and convexity.
\newblock In \emph{Mathematical Programming The State of The Art}, pages
  235--257. Springer, 1983.

\bibitem[Matsuoka et~al.(2021)Matsuoka, Ito, and
  Ohsaka]{Matsuoka-2021-Tracking}
T.~Matsuoka, S.~Ito, and N.~Ohsaka.
\newblock Tracking regret bounds for online submodular optimization.
\newblock In \emph{AISTATS}, pages 3421--3429. PMLR, 2021.

\bibitem[McCormick(2005)]{Mccormick-2005-Submodular}
S.~T. McCormick.
\newblock Submodular function minimization.
\newblock \emph{Handbooks in Operations Research and Management Science},
  12:\penalty0 321--391, 2005.

\bibitem[Narasimhan et~al.(2005)Narasimhan, Jojic, and
  Bilmes]{Narasimhan-2005-Q}
M.~Narasimhan, N.~Jojic, and J.~Bilmes.
\newblock Q-{C}lustering.
\newblock In \emph{NeurIPS}, pages 979--986, 2005.

\bibitem[Orlin(2009)]{Orlin-2009-Faster}
J.~B. Orlin.
\newblock A faster strongly polynomial time algorithm for submodular function
  minimization.
\newblock \emph{Mathematical Programming}, 118\penalty0 (2):\penalty0 237--251,
  2009.

\bibitem[Pike-Burke et~al.(2018)Pike-Burke, Agrawal, Szepesvari, and
  Grunewalder]{Pike-2018-Bandits}
C.~Pike-Burke, S.~Agrawal, C.~Szepesvari, and S.~Grunewalder.
\newblock Bandits with delayed, aggregated anonymous feedback.
\newblock In \emph{ICML}, pages 4105--4113. PMLR, 2018.

\bibitem[Quanrud and Khashabi(2015)]{Quanrud-2015-Online}
K.~Quanrud and D.~Khashabi.
\newblock Online learning with adversarial delays.
\newblock In \emph{NeurIPS}, pages 1270--1278, 2015.

\bibitem[Rapaport et~al.(2008)Rapaport, Barillot, and
  Vert]{Rapaport-2008-Classification}
F.~Rapaport, E.~Barillot, and J-P. Vert.
\newblock Classification of array{CGH} data using fused {SVM}.
\newblock \emph{Bioinformatics}, 24\penalty0 (13):\penalty0 i375--i382, 2008.

\bibitem[Roughgarden and Wang(2018)]{Roughgarden-2018-Optimal}
T.~Roughgarden and J.~R. Wang.
\newblock An optimal learning algorithm for online unconstrained submodular
  maximization.
\newblock In \emph{COLT}, pages 1307--1325. PMLR, 2018.

\bibitem[Sakaue(2019)]{Sakaue-2019-Greedy}
S.~Sakaue.
\newblock Greedy and {IHT} algorithms for non-convex optimization with monotone
  costs of non-zeros.
\newblock In \emph{AISTATS}, pages 206--215. PMLR, 2019.

\bibitem[Schrijver(2000)]{Schrijver-2000-Combinatorial}
A.~Schrijver.
\newblock A combinatorial algorithm minimizing submodular functions in strongly
  polynomial time.
\newblock \emph{Journal of Combinatorial Theory, Series B}, 80\penalty0
  (2):\penalty0 346--355, 2000.

\bibitem[Shalev-Shwartz(2011)]{Shalev-2011-Online}
S.~Shalev-Shwartz.
\newblock Online learning and online convex optimization.
\newblock \emph{Foundations and Trends in Machine Learning}, 4\penalty0
  (2):\penalty0 107--194, 2011.

\bibitem[Shalev-Shwartz and Singer(2006)]{Shalev-2006-Convex}
S.~Shalev-Shwartz and Y.~Singer.
\newblock Convex repeated games and {F}enchel duality.
\newblock In \emph{NIPS}, pages 1265--1272, 2006.

\bibitem[Streeter and Golovin(2008)]{Streeter-2008-Online}
M.~Streeter and D.~Golovin.
\newblock An online algorithm for maximizing submodular functions.
\newblock In \emph{NeurIPS}, pages 1577--1584, 2008.

\bibitem[Sviridenko et~al.(2017)Sviridenko, Vondr{\'a}k, and
  Ward]{Sviridenko-2017-Optimal}
M.~Sviridenko, J.~Vondr{\'a}k, and J.~Ward.
\newblock Optimal approximation for submodular and supermodular optimization
  with bounded curvature.
\newblock \emph{Mathematics of Operations Research}, 42\penalty0 (4):\penalty0
  1197--1218, 2017.

\bibitem[Svitkina and Fleischer(2011)]{Svitkina-2011-Submodular}
Z.~Svitkina and L.~Fleischer.
\newblock Submodular approximation: Sampling-based algorithms and lower bounds.
\newblock \emph{SIAM Journal on Computing}, 40\penalty0 (6):\penalty0
  1715--1737, 2011.

\bibitem[Thang and Srivastav(2021)]{Thang-2021-Online}
N.~K. Thang and A.~Srivastav.
\newblock Online non-monotone {DR}-submodular maximization.
\newblock In \emph{AAAI}, pages 9868--9876, 2021.

\bibitem[Thune et~al.(2019)Thune, Cesa-Bianchi, and
  Seldin]{Thune-2019-Nonstochastic}
T.~S. Thune, N.~Cesa-Bianchi, and Y.~Seldin.
\newblock Nonstochastic multiarmed bandits with unrestricted delays.
\newblock In \emph{NeurIPS}, pages 6538--6547, 2019.

\bibitem[Trevisan(2014)]{Trevisan-2014-Inapproximability}
L.~Trevisan.
\newblock Inapproximability of combinatorial optimization problems.
\newblock \emph{Paradigms of Combinatorial Optimization: Problems and New
  Approaches}, pages 381--434, 2014.

\bibitem[Vernade et~al.(2017)Vernade, Capp{\'{e}}, and
  Perchet]{Vernade-2017-Stochastic}
C.~Vernade, O.~Capp{\'{e}}, and V.~Perchet.
\newblock Stochastic bandit models for delayed conversions.
\newblock In \emph{UAI}. AUAI Press, 2017.

\bibitem[Vondr{\'a}k(2007)]{Vondrak-2007-Submodularity}
J.~Vondr{\'a}k.
\newblock \emph{Submodularity in Combinatorial Optimization}.
\newblock PhD thesis, Charles University, 2007.

\bibitem[Zhou et~al.(2017)Zhou, Mertikopoulos, Bambos, Glynn, and
  Tomlin]{Zhou-2017-Countering}
Z.~Zhou, P.~Mertikopoulos, N.~Bambos, P.~Glynn, and C.~Tomlin.
\newblock Countering feedback delays in multi-agent learning.
\newblock In \emph{NeurIPS}, pages 6172--6182, 2017.

\bibitem[Zhou et~al.(2019)Zhou, Xu, and Blanchet]{Zhou-2019-Learning}
Z.~Zhou, R.~Xu, and J.~Blanchet.
\newblock Learning in generalized linear contextual bandits with stochastic
  delays.
\newblock In \emph{NeurIPS}, pages 5197--5208, 2019.

\bibitem[Zimmert and Seldin(2020)]{Zimmert-2020-Optimal}
J.~Zimmert and Y.~Seldin.
\newblock An optimal algorithm for adversarial bandits with arbitrary delays.
\newblock In \emph{AISTATS}, pages 3285--3294. PMLR, 2020.

\bibitem[Zinkevich(2003)]{Zinkevich-2003-Online}
M.~Zinkevich.
\newblock Online convex programming and generalized infinitesimal gradient
  ascent.
\newblock In \emph{ICML}, pages 928--936, 2003.

\end{thebibliography}
